\documentclass[twoside,11pt,colorlinks,linktoc=all]{article}

\usepackage[nohyperref,abbrvbib]{jmlr2e}

\usepackage[utf8]{inputenc}

\usepackage[dvipsnames,table,]{xcolor}  
\usepackage[colorlinks,linktoc=all]{hyperref}  
\usepackage[all]{hypcap}
\hypersetup{citecolor=Black}
\hypersetup{linkcolor=Black}
\hypersetup{urlcolor=Black}

\usepackage[font=small,labelfont=bf,tableposition=top]{caption}
\usepackage{subcaption}

\usepackage{textcomp}
\usepackage{gensymb}
\usepackage{amsmath}
\usepackage{mathtools}
\usepackage{enumitem}
\usepackage{multirow}
\usepackage{dirtytalk}
\usepackage{verbatim}
\usepackage{url}
\usepackage{booktabs}
\usepackage{afterpage}
\usepackage{chngcntr}
\usepackage{xargs} 
\usepackage[ruled,vlined]{algorithm2e}
\usepackage{ragged2e}

\usepackage[capitalize,nameinlink]{cleveref}
\crefname{section}{\S}{\S\S}
\Crefname{section}{\S}{\S\S}

\graphicspath{ {images/} }

\newcommand{\mat}[1]{\boldsymbol{#1}}
\newcommand{\norm}[1]{\left\lVert\mat{#1}\right\rVert}
\newcommand{\dotprod}[2]{\mat{#1}^{\top} \mat{#2}}
\newcommand{\dotprodi}[3]{\mat{#1}_{#3}^{\top} \mat{#2}}

\newcommand{\mscpi}[3]{\left(\mat{#1}_{#3}^{\top} \mat{#2}\right)}

\newcommand{\bigO}{\mathcal{O}}
\newcommand{\iu}{\mathrm{i}\mkern1mu}
\DeclarePairedDelimiter{\diagfences}{(}{)}
\newcommand{\diag}{\operatorname{diag}\diagfences}

\DeclareMathOperator*{\argmin}{arg\,min}

\usepackage{lastpage}
\jmlrheading{25}{2024}{1-\pageref{LastPage}}{2/22; Revised
7/24}{8/24}{22-0118}{Jonas Wacker, Motonobu Kanagawa and Maurizio Filippone}
\ShortHeadings{Improved Random Features for Dot Product Kernels}{Wacker, Kanagawa and Filippone}

\firstpageno{1}

\begin{document}

\title{
Improved Random Features for Dot Product Kernels}

\author{%
    \name Jonas Wacker \email jonas.wacker@gmail.com \\[.5ex]
    \name Motonobu Kanagawa \email motonobu.kanagawa@eurecom.fr \\[.5ex]
    \addr Data Science Department, EURECOM, France\\[.5ex]
    \name Maurizio Filippone \email maurizio.filippone@kaust.edu.sa \\[.5ex]
    \addr Statistics Program, KAUST, Saudi Arabia
}

\editor{Jean-Philippe Vert}

\maketitle

\begin{abstract}%
 
Dot product kernels, such as polynomial and exponential (softmax) kernels, are among the most widely used kernels in machine learning, as they enable modeling the interactions between input features, which is crucial in applications like computer vision, natural language processing, and recommender systems. We make several novel contributions for  improving the efficiency of random feature approximations for dot product kernels, to make these kernels more useful in large scale learning. First, we present a generalization of existing random feature approximations for polynomial kernels, such as Rademacher and Gaussian sketches and TensorSRHT, using complex-valued random features. We show empirically that the use of complex features can significantly reduce the variances of these approximations. Second, we provide a theoretical analysis for understanding the factors affecting the efficiency of various random feature approximations, by deriving closed-form expressions for their variances. These variance formulas elucidate conditions under which certain approximations (e.g., TensorSRHT) achieve lower variances than others (e.g., Rademacher sketches), and conditions under which the use of complex features leads to lower variances than real features. Third, by using these variance formulas, which can be evaluated in practice, we develop a data-driven optimization approach to improve random feature approximations for general dot product kernels, which is also applicable to the Gaussian kernel. We describe the improvements brought by these contributions with extensive experiments on a variety of tasks and datasets.
 
\end{abstract}


\tableofcontents

\section{Introduction}

Statistical learning methods based on {\em positive definite kernels}, namely Gaussian processes \citep{Rasmussen2006} and kernel methods \citep{Schoelkopf2001}, are among the most theoretically principled approaches in machine learning with competitive empirical performance. Due to their strong theoretical guarantees, these methods should be of primary choice in applications where the learning machine should behave in an anticipated manner, e.g., when high-stake decision-making is involved or when safety is required. 
However, a well-known drawback of these methods is their high computational costs, as naive implementations usually require the computational complexity of $\mathcal{O}(N^3)$ or at least $\mathcal{O}(N^2)$, where $N$ is the training data size. 
This unfavorable scalability is an obstacle for these methods to handle a large amount of data. Moreover, it is problematic from a sustainability viewpoint, since these methods may perform essentially redundant computations and thus waste available computational resources.

The scalability issue has been a focus of research since the earliest literature \citep[Chapter 7]{wahba1990spline}, and many approximation methods for reducing the computational costs have been developed (e.g., \citealt{williams2001using,Rahimi2007,titsias2009variational,hensman2017variational}). 
One of the most successful approximations are those based on {\em random features}, initiated by \citet{Rahimi2007}. This approach constructs a random feature map $\Phi$ that transforms an input point $\mat{x}$ to a finite dimensional feature vector $\Phi(\mat{x}) \in \mathbb{R}^D$, so that the inner product of two feature maps $\Phi(\mat{x})^\top \Phi(\mat{y})$ approximates the kernel value $k(\mat{x},\mat{y})$ of the two input points $\mat{x}, \mat{y} \in \mathbb{R}^d$. The resulting computational complexity is dominated by the dimensionality $D$ of the random features, and thus the computational costs can be drastically reduced if $D$ is much smaller than the training data size $N$. 
\citet{Rahimi2007} proposed {\em random Fourier features} for {\em shift-invariant kernels} on the Euclidean space $\mathbb{R}^d$. These are kernels that depend only on the difference of two input points, i.e., of the form $k(\mat{x}, \mat{y}) = k( \mat{x} - \mat{y} )$, such as the Gaussian and Mat\'ern kernels. For a recent overview of random Fourier features and their extensions, see \citet{Liu2020a}.

Another important class of kernels are {\em dot product kernels}, which can be written as a function of the dot product (or the inner product) between two input points, i.e., kernels of the form $k(\mat{x}, \mat{y}) = k( \mat{x}^\top \mat{y} )$. 
Representative examples include {\em polynomial kernels},  $k(\mat{x}, \mat{y}) = ( \mat{x}^\top \mat{y} + \nu  )^p$ with $\nu \geq  0$ and $p \in \mathbb{N}$,   and {\em exponential kernels} (or {\em softmax kernels}), $k(\mat{x}, \mat{y}) = \exp(\dotprod{x}{y} / \sigma^2)$ with $\sigma > 0$ .  
These kernels can model the interactions between input features\footnote{As explained in Section \ref{sec:dot-product-kernels}, any dot product kernel can be written as a weighted sum of polynomial kernels. Since each polynomial kernel models multiplicative interactions between input features (See Section \ref{sec:polynomial-sketches}), the resulting dot product kernel also implicitly models such interactions.} \citep[e.g.,][]{Agrawal19a},  and thus are useful in applications such as genomic data analysis \citep{Aschard2016,weissbrod2016multikernel}, recommender systems \citep{Rendle2010,blondel2016polynomial}, computer vision \citep{lin2015bilinear,Gao2016, Fukui2016}, and natural language processing \citep{yamada2003statistical,chang2010training,Vaswani2017}.
Recent notable applications of dot product kernels include 
 {\em bilinear pooling} in computer vision  \citep{lin2015bilinear}, which essentially uses a polynomial kernel, and the {\em dot product attention mechanism}  in the Transformer architecture \citep{Vaswani2017}, which uses an exponential kernel.

As for kernel methods in general, approximations are necessary to make use of dot product kernels in large scale learning.
However, since dot product kernels are {\em not} shift-invariant, one cannot apply random Fourier features for their approximations, except for some specific cases \citep[c.f.][]{Pennington2015, Choromanski21a}.
Therefore, alternative random feature approximations have been proposed for dot product kernels in the literature, with a particular focus on polynomial kernels \citep{Kar2012, Pham2013, Hamid2014,Avron2014,Ahle2020,Song21c}.
These approximations are based on {\em sketching}  \citep{Woodruff2014}, which is a randomized linear projection of input feature vectors into a low dimensional space.
Random feature approximations play a key role in the above applications of dot product kernels  \cite[e.g.,][]{Gao2016, Fukui2016,Choromanski21a}.
 \\

This paper contributes to the above line of research, by suggesting various approaches to improving the efficiency of random feature approximations for dot product kernels. 
Our overarching goal is to make these kernels more useful in large scale learning, thereby widening their applicability.
Specifically, we make the following contributions: (From now on, we refer to sketching-based random feature approximations for polynomial kernels as {\em polynomial sketches} for brevity.)

\paragraph{Complex-valued features.}
We propose a generalization of polynomial sketches using  {\em complex-valued} random features.
This generalization is applicable to all polynomial sketches\footnote{There exists a recent line of research on improving the efficiency of polynomial sketches using a hierarchical feature construction  \citep[e.g.,][]{Ahle2020, Song21c}.  One can also use the proposed complex polynomial sketches as  base sketches in such a hierarchical construction. Therefore, our contributions are complementary to this line of research.}, including those using i.i.d.~Gaussian or Rademacher features \citep{Kar2012,Hamid2014} and structured sketches such as  TensorSRHT \citep{Hamid2014,Ahle2020}.
Our  approach is an extension of complex-valued random features  for the {\em linear} kernel discussed in \citet{Choromanski2017}  to {\em polynomial} kernels.  
We empirically show that the generalized polynomial sketches using complex features are statistically more efficient than those using real features (in terms of the resulting variances) in particular for higher degree polynomial kernels, and that the former leads to better performance in downstream learning tasks.
To corroborate these empirical findings, we provide a theoretical  analysis of the variances of these sketches, as explained below.

\paragraph{Variance formulas.}
We derive {\em closed-form} formulas for the variances of  the Gaussian and Rademacher polynomial sketches \citep{Kar2012,Hamid2014} and  TensorSRHT  \citep{Hamid2014, Ahle2020} as well as for their complex generalizations. 
These variance formulas provide new insights into the factors affecting the efficiency of these polynomial sketches, complementing existing theoretical results \citep[c.f.][]{Kar2012, Hamid2014, Ahle2020}.
Specifically, they elucidate conditions under which certain approximations (e.g., TensorSRHT) achieve lower variances than others (e.g, Rademacher sketches), and conditions under which the use of complex  features leads to lower variances than real features. 
Importantly, these variance formulas can be evaluated in practice, and thus can be optimized; this is how we develop a novel optimization approach to random feature construction,  explained next.

\paragraph{Optimized Maclaurin approximation for general dot product  kernels.}
Using the derived variance formulas,  we develop a data-driven optimization approach to random feature approximations for general dot product kernels, which is also applicable to the Gaussian kernel. 
Inspired from the randomized Maclaurin approximation of \citet{Kar2012}, we use a finite-degree  Maclaurin approximation of the kernel, given as a weighted sum of polynomial kernels of different degrees. 
Our approach optimizes the cardinalities of random features for approximating the polynomial kernels of different degrees  (given a total number of random features), so as to minimize the mean square error of the approximate kernel with respect to the data distribution -- we utilize the variance formulas to define this optimization objective. 
This optimized Maclaurin approach is compatible with exiting polynomial sketches as well as their complex generalizations, and enhances the efficiency of these sketches to achieve state-of-the-art performance, as we show in our experiments.

\paragraph{Extensive empirical comparison.} 
We conduct extensive experiments to study the effectiveness of the suggested approaches. 
Our investigations include the approximations of polynomial and Gaussian kernels, and cover various random feature approximations.
We study not only the quality of kernel approximation, but also the performance in downstream  learning tasks of Gaussian process regression and classification. 
We generally observe that the proposed approaches lead to significant reduction of kernel approximation errors, and also to state-of-the-art performance in the downstream tasks on most datasets. 

\paragraph{Software package.}
We provide a GitHub repository\footnote{Our code is available at: https://github.com/joneswack/dp-rfs} with modern implementations for all the methods studied in this work supporting GPU acceleration and automatic differentiation in PyTorch \citep{paszke2019pytorch}. Since version 1.8, PyTorch natively supports numerous linear algebra operations on complex numbers\footnote{The PyTorch 1.8 release notes are available at: https://github.com/pytorch/pytorch/releases/tag/v1.8.0}. The same is true for NumPy \citep{harris2020array} and TensorFlow \citep{tensorflow2015-whitepaper}. Therefore, it is straightforward to implement the complex-valued polynomial sketches proposed in this work. \\

This paper is organized as follows. 
Section~\ref{sec:preliminaries} presents preliminaries. 
In Section~\ref{sec:polynomial-sketches}, we review polynomial sketches using i.i.d.~random features and introduce their complex generalizations. We also provide a theoretical analysis and derive variance formulas. 
In Section~\ref{sec:structured-projections}, we study structured polynomial sketches and their complex generalizations, also deriving their variance formulas. 
In Section~\ref{sec:approx-dot-prod-kernels}, we study the approximation of general dot product kernels, and present the optimized Maclaurin approach. 
In Section~\ref{sec:experiments}, we report the results of extensive experiments. 
The appendix contains many supplementary materials, including proofs for theoretical results, additional experiments, and an explanation of Gaussian process regression and classification using complex random features.

\section{Preliminaries}
\label{sec:preliminaries}

This section serves as preliminaries for describing our main contributions. We first introduce basic notation and definitions in Section \ref{sec:notation}. 
We then define positive definite kernels in Section \ref{sec:kernel-def}.

\subsection{Notation} \label{sec:notation}
Let $\mathbb{N}$ and $\mathbb{R}$ denote the sets of natural and real numbers, respectively, and let $\mathbb{R}^d$ denote the real vector space of dimension $d \in \mathbb{N}$. 
Let $\mathbb{C}$ be the set of complex numbers, and $\overline{c}$ for $c \in \mathbb{C}$ be the complex conjugate of $c$.
Let $\iu := \sqrt{-1}$ be the imaginary unit.
 
We use $\mathcal{X}$ to denote a set of input points, and we generally assume $\mathcal{X} \subseteq \mathbb{R}^d$. We write the vector-valued inputs by bold-faced letters, e.g., $\mat{x} \in \mathbb{R}^d$.
For $\mat{x} := (x_1, \dots, x_d)^\top \in \mathbb{R}^d$, let $\| \mat{x} \| := \| \mat{x} \|_2 :=\sqrt{ \sum_{i=1}^d x_i^2 }$ be the 2-norm, and $ \| \mat{x} \|_1 := \sum_{i=1}^d |x_i|$ be the 1-norm. We may interchangeably use $\| \mat{x} \|$ and $\| \mat{x} \|_2$ depending on the context.

For any two vectors $\mat{a} \in \mathbb{R}^{d_1}$ and $\mat{b} \in \mathbb{R}^{d_2}$, $\mat{a} \otimes \mat{b} := \mathrm{vec}(\mat{a} \mat{b}^{\top}) \in \mathbb{R}^{d_1 \cdot d_2}$ denotes the vectorized outer product between $\mat{a}$ and $\mat{b}$.

We denote by $\mathbb{E} [\cdot]$ the expected value and by $\mathbb{V} [\cdot]$ the variance of a random variable. For complex-valued vectors $\mat{z} = \mat{x} + \iu \mat{y} \in \mathbb{C}^{d}$, with $\mat{x}, \mat{y} \in \mathbb{R}^d$, we define $\mathcal{R} \{ \mat{z} \} := \mat{x}$ and $\mathcal{I} \{ \mat{z} \} := \mat{y}$ to be their real and imaginary parts, respectively. 

We further define $\lfloor \cdot \rfloor$ and $\lceil \cdot \rceil$ to be the floor and ceil operators that round a floating point number down/up to the next integer, whereas ${\rm mod}(a, b)$ with $a, b \in \mathbb{N}$ is the arithmetic modulus that gives the rest after dividing $a$ by $b$.

\subsection{Positive Definite Kernels} \label{sec:kernel-def}

Let $\mathcal{X}$ be a nonempty set. A symmetric function $k: \mathcal{X} \times \mathcal{X} \rightarrow \mathbb{R}$ is called {\em positive definite kernel}, if for every $m \in \mathbb{N}$, $x_1, \dots, x_m \in \mathcal{X}$ and $c_1, \dots, c_m \in \mathbb{R}$
\begin{equation*}
    \sum_{i=1}^m \sum_{j=1}^m c_i c_j k(x_i, x_j) \geq 0.
\end{equation*}
We may simply call such $k$ {\em kernel}.
Popular examples include {\em Gaussian kernels} and {\em polynomial kernels}, among many others. 

For any kernel $k$, there exists a corresponding {\em reproducing kernel Hilbert space (RKHS)} consisting of functions on $\mathcal{X}$. In kernel methods \citep{Schoelkopf2001}, the RKHS provides implicit feature representations for points in $\mathcal{X}$, where each feature vector can be potentially infinite dimensional.
A learning method is defined conceptually on such feature representations in the RKHS, but the resulting concrete algorithm can be formulated as a finite dimensional optimization problem defined through the evaluations of the kernel $k$ evaluated at given data points $x_1, \dots, x_N$:
\begin{equation} \label{eq:kernel-values}
    k(x_i, x_j) , \quad i,j = 1,\dots, N.
\end{equation}
This reduction to a finite dimensional optimization problem is the core idea of kernel methods, enabled by the so-called kernel trick. 
However, the exact solution to the optimization problem often requires the computational complexity of $\mathcal{O}(N^3)$ or at least $\mathcal{O}(N^2)$ with $N$ being the data size, which poses a computational challenge to kernel methods. 

Similarly, any positive definite kernel $k$ can be used to define a Gaussian process whose covariance function is $k$ \citep{Rasmussen2006}. 
In Bayesian nonparametric learning, a Gaussian process is used to define a prior distribution over functions on $\mathcal{X}$, and the resulting posterior distribution is given in terms of the values of $k$ evaluated on the data points.
Like kernel methods, Gaussian processes also face a computational challenge, as naively computing the posterior requires the complexity of $\mathcal{O}(N^3)$ or at least $\mathcal{O}(N^2)$.

Various techniques have been proposed to speed up Gaussian processes and kernel methods by approximately computing the solution of interest. 
Approximations based on {\em random features} are one of the most successful approximation approaches, and these are the main topic of this paper.

\section{Polynomial Sketches}
\label{sec:polynomial-sketches}

We study here random feature approximations of {\em polynomial kernels}, defined as
\begin{equation} \label{eq:poly-kernels}
    k(\mat{x}, \mat{y}) = ( \mat{x}^\top \mat{y} + \nu  )^p,
\end{equation}
where $\nu \geq 0$ and $p \in \mathbb{N}$.
We call such random features {\em polynomial sketches}. 
Since polynomial kernels are not shift-invariant, widely known random Fourier features \citep{Rahimi2007} cannot be applied directly.
Polynomial sketches are a fundamentally different approach, and can be understood as implicit randomized projections of the explicit high dimensional feature maps of polynomial kernels. 

For simplicity, we focus on {\em homogeneous} polynomial kernels of the form 
\begin{equation} \label{eq:homo-poly-kernel}
        k(\mat{x}, \mat{y}) = ( \mat{x}^\top \mat{y}   )^p,
\end{equation}
i.e., $\nu = 0$ in \eqref{eq:poly-kernels}. The inhomogeneous case $\nu > 0$ can be reduced to the homogeneous case, by appending $\sqrt{\nu}$ to the input vectors, i.e., by setting $\tilde{\mat{x}} :=  [\mat{x}^\top, \sqrt{\nu}]^\top \in \mathbb{R}^{d+1}$ and $\tilde{\mat{y}} :=  [\mat{y}^\top, \sqrt{\nu}]^\top \in \mathbb{R}^{d+1}$, we have
$$
( \mat{x}^\top \mat{y} + \nu  )^p = (\tilde{\mat{x}}^\top \tilde{\mat{y}})^p 
$$
In this way, polynomial sketches for the homogeneous case can also be applied to the inhomogeneous case. 

We first review existing polynomial sketches with i.i.d.~real-valued features in Section \ref{sec:real-poly-sketch}. 
In Section \ref{sec:complex-projections}, we propose polynomial sketches with {\em complex-valued features}. These complex-valued sketches are an extension of the complex-valued sketches for the {\em linear} kernel discussed in \citet{Choromanski2017} to {\em polynomial} kernels.
We derive the variance formulas of these complex sketches in Section \ref{sec:var-comp-poly} and present a probabilistic error bound in
Section \ref{sec:prob-error-comp-poly}.

\subsection{Real-valued Polynomial Sketches }
\label{sec:real-poly-sketch}

We first study polynomial sketches proposed by \citet{Kar2012}, which are also discussed in \citet{Hamid2014}.  
We do not cover here  TensorSketch of \citet{Pham2013}, as it is conceptually different from the other polynomial sketches discussed in this paper.\footnote{However, we will include TensorSketch in our empirical evaluation in Section \ref{sec:experiments}.}

Let $D \in \mathbb{N}$ be the number of random features, and $p \in \mathbb{N}$ be the degree of the polynomial kernel \eqref{eq:homo-poly-kernel}.
Suppose we generate $p \times D$ i.i.d.~random vectors
\begin{equation} \label{eq:random-vectors-real}
\mat{w}_{i,\ell} \in \mathbb{R}^d \quad \text{satisfying}\quad \mathbb{E} [\mat{w}_{i,\ell} \mat{w}_{i,\ell}^{\top}] = \mat{I}_d, \quad i \in \{1, \dots, p\},\quad \ell \in \{1, \dots, D\}, 
\end{equation}
where $\mat{I}_d \in \mathbb{R}^{d \times d}$ denotes the identity matrix.  

Then we define a random feature map as
\begin{align}
    \label{eqn:polynomial-estimator}
    \Phi_{\mathcal{R}}(\mat{x})
    :=  \frac{1}{\sqrt{D}} \left[ (\prod_{i=1}^p \dotprodi{w}{x}{i,1}), \dots, (\prod_{i=1}^p \dotprodi{w}{x}{i,D}) \right]^{\top} \in \mathbb{R}^D.
\end{align}
The resulting approximation of the polynomial kernel \eqref{eq:homo-poly-kernel} is given by
\begin{equation} \label{eq:approx-kernel-real-D}
    \hat{k}_\mathcal{R} (\mat{x}, \mat{y}) :=  \Phi_{\mathcal{R}}(\mat{x})^{\top} \Phi_{\mathcal{R}}(\mat{y}),
\end{equation}
which is unbiased, as the expectation with respect to the random vectors \eqref{eq:random-vectors-real} gives
\begin{align*}
    \mathbb{E}\left[\Phi_{\mathcal{R}}(\mat{x})^{\top} \Phi_{\mathcal{R}}(\mat{y})\right]
    = \frac{1}{D} \sum_{\ell=1}^D \prod_{i=1}^p \mat{x}^{\top} \mathbb{E} [\mat{w}_{i,\ell} \mat{w}_{i,\ell}^{\top}] \mat{y}
    = (\dotprod{x}{y})^p.
\end{align*}
\citet{Kar2012} suggest to define random vectors in \eqref{eq:random-vectors-real} using the Rademacher distribution: each element of $\mat{w}_{i,\ell}$ is independently drawn from $\{-1, 1\}$ with equal probability.  
We study later how the distribution of the random vectors affects the quality of kernel approximation.

\paragraph{Implicit sketching of high-dimensional features.} 
The random feature map \eqref{eqn:polynomial-estimator} can be interpreted as a linear sketch (projection) of an explicit high-dimensional feature vector for the polynomial kernel. 
To describe this, consider the case $D = 1$ and let $\mat{w}_i = (w_{i,1}, \dots, w_{i,d})^\top \in \mathbb{R}^d $, $i=1,\dots,p$, be i.i.d.~random vectors satisfying $\mathbb{E}[ \mat{w}_i \mat{w}_i^\top ] = \mat{I}_d$. 
Then, the random feature map $\Phi_\mathcal{R}(\mat{x})$ (which is one dimensional in this case) for $\mat{x} := (x_1, \dots, x_d)^\top$ is given by
\begin{align}
 \Phi_\mathcal{R}(\mat{x}) =   \prod_{i=1}^p \dotprodi{w}{x}{i}
    = \prod_{i=1}^p \sum_{j=1}^d w_{i,j} x_j
    = \sum_{j_1=1, \dots, j_p=1}^d w_{1,j_1} x_{j_1} \cdots w_{p,j_p} x_{j_p}
    = \mat{w}^{(p) \top} \mat{x}^{(p)},
    \label{eqn:random-tensor-expansion}
\end{align} 
where $\mat{x}^{(p)} \in \mathbb{R}^{d^p}$ and $\mat{w}^{(p)} \in \mathbb{R}^{d^p}$ are defined as (recall the notation in Section \ref{sec:notation}) 
$$
\mat{x}^{(p)} := \mat{x} \underbrace{\otimes \cdots \otimes}_{p \, \text{times}} \mat{x} \in \mathbb{R}^{d^p}, \quad \mat{w}^{(p)} := \mat{w} \underbrace{\otimes \cdots \otimes}_{p \, \text{times}} \mat{w} \in \mathbb{R}^{d^p}, 
$$
Therefore, for $D = 1$, the approximate kernel is given as 
\begin{equation} \label{eq:approx-kernel-1d-kar12}
\hat{k}_\mathcal{R} (\mat{x}, \mat{y}) :=  \Phi_{\mathcal{R}}(\mat{x}) \cdot \Phi_{\mathcal{R}}(\mat{y}) = \mat{w}^{(p) \top} \mat{x}^{(p)} \cdot \mat{w}^{(p) \top} \mat{y}^{(p)}  
\end{equation}

On the other hand, the polynomial kernel can be written as \citep[Proposition 2.1]{Schoelkopf2001}:
\begin{equation*}
\quad (\mat{x}^{\top} \mat{y})^p = (\mat{x}^{(p)})^{\top} \mat{y}^{(p)}, 
\end{equation*} 
where $\mat{y}^{(p)} = \mat{y} \otimes \cdots \otimes \mat{y} \in \mathbb{R}^{d^p}$. 
Thus, $\mat{x}^{(p)}$ and $\mat{y}^{(p)}$ are the exact feature maps of the input vectors $\mat{x}$ and $\mat{y}$, respectively. 
The comparison of this expression with \eqref{eq:approx-kernel-1d-kar12} implies that the random feature map $ \Phi_\mathcal{R}(\mat{x})  = \mat{w}^{(p) \top} \mat{x}^{(p)} \in \mathbb{R}$ in \eqref{eqn:random-tensor-expansion}  is a projection of the exact feature map $\mat{x}^{(p)} \in \mathbb{R}^{d^p}$ onto $\mathbb{R}$.

Similarly, if $D > 1$, the random feature map $ \Phi_\mathcal{R}(\mat{x}) \in \mathbb{R}^D$ in  \eqref{eqn:polynomial-estimator} can be interpreted as a projection of the exact feature map $\mat{x}^{(p)}$ onto $\mathbb{R}^D$. 
A remarkable point of this random feature map is that it can be obtained without constructing the exact feature vector $\mat{x}^{(p)}$, the latter being infeasible if $d$ or $p$ is large.\footnote{The exact feature expansion $\mat{x}^{(p)}$ leads to $d^p$ dimensional vectors. By grouping up equal terms, we can reduce the dimensionality to $\binom{d+p-1}{p}$, which still leads to unrealistic dimensional feature vectors as soon as $d$ is large. For example, working with MNIST images of size 28x28 ($d=784$) leads to 307,720 features for $p=2$ and to 80,622,640 features for $p=3$. This justifies the need for randomized approximations of the polynomial kernel.} 
Indeed, the computational complexity of constructing the random feature map $\Phi_\mathcal{R}(\mat{x})$ is $\bigO(p d D)$, while the exact feature map $\mat{x}^{(p)}$ requires $\bigO(d^p)$.

\subsection{Complex-valued Polynomial Sketches}

\label{sec:complex-projections}

We now introduce complex-valued polynomial sketches, one of our novel contributions.
We do this by extending the analysis of  \citet{Choromanski2017} for linear sketches to polynomial sketches.\footnote{More specifically,  \citet{Choromanski2017} analyze the variance of the {\em real part} of the approximate complex-valued kernel in \cref{eq:approx-kernel-D-complex} for $p=1$.  In contrast, we study \cref{eq:approx-kernel-D-complex}  with generic $p \in \mathbb{N}$, and analyze the variance of \cref{eq:approx-kernel-D-complex} itself, including both the real and imaginary parts. 
}

As before, without loss of generality, we focus on approximating the homogeneous polynomial kernel $k(\mat{x}, \mat{y}) = (\mat{x}^\top \mat{y})^p$ of degree $p \in \mathbb{N}$.
Let $D \in \mathbb{N}$.
Suppose we generate $p \times D$ {\em complex-valued} random vectors 
satisfying 
\begin{align}
&\mat{z}_{i,j} \in \mathbb{C}^d \quad \text{satisfying} \quad \mathbb{E}[ \mat{z}_{i,j}  \overline{\mat{z}_{i,j}}^\top ] = \mat{I}_d, \quad i \in \{1, \dots, p\},\quad j \in \{1, \dots, D\}  \label{eq:complex-weights-properties}
\end{align}
We then define a {\em complex-valued random feature map} as
\begin{align}
    \label{eqn:comp-rad-polynomial-sketch}
    \Phi_\mathcal{C}(\mat{x})
    :=  \frac{1}{\sqrt{D}} \left[ (\prod_{i=1}^p \dotprodi{z}{x}{i,1}), \dots, (\prod_{i=1}^p \dotprodi{z}{x}{i,D}) \right]^{\top} \in \mathbb{C}^D, \quad \mat{x} \in \mathbb{R}^d,
\end{align}
and the resulting approximate kernel as 
\begin{align}
    \hat{k}_\mathcal{C}( \mat{x}, \mat{y} ) &:= \Phi_{\mathcal{C}}(\mat{x})^\top \overline{\Phi_{\mathcal{C}}(\mat{y})} = \frac{1}{D} \sum_{j=1}^D \prod_{i=1}^p (\mat{z}_{i,j}^\top \mat{x}) \overline{(\mat{z}_{i,j}^\top \mat{y})}, \quad \mat{x}, \mat{y} \in \mathbb{R}^d.
    \label{eq:approx-kernel-D-complex}
\end{align}

\cref{eq:approx-kernel-D-complex} is a generalization of the approximate kernel \eqref{eq:approx-kernel-real-D} with real-valued features, as \cref{eq:approx-kernel-real-D} can be recovered by defining the complex random vectors $\mat{z}_{i,k}$ in \cref{eq:complex-weights-properties} as real random vectors $\mat{w}_{i,k}$ in \cref{eq:random-vectors-real}; in this case the requirement $\mathbb{E}[ \mat{z}_{i,j}  \overline{\mat{z}_{i,j}}^\top ] = \mathbb{E}[ \mat{w}_{i,j}  \mat{w}_{i,j}^\top ] = \mat{I}_d$ is satisfied.

For example, complex-valued random vectors $\mat{z}_{i,j}$ satisfying \cref{eq:complex-weights-properties} can be generated as follows.

\begin{example} \label{example:two-random-reals-complex}
Suppose we generate $2 \times p \times D$ independent real-valued random vectors  
\begin{align} \label{eq:two-real-for-complex-properties}
  & \mat{v}_{i,j},\ \mat{w}_{i,j} \in \mathbb{R}^d \quad \text{satisfying} \quad  \mathbb{E}[\mat{v}_{i,j}] = \mathbb{E}[\mat{w}_{i,j}] = \mat{0}, \quad \mathbb{E}[\mat{v}_{i,j} \mat{v}_{i,j}^{\top}] = \mathbb{E}[\mat{w}_{i,j} \mat{w}_{i,j}^{\top}] = \mat{I}_d
\end{align}
for $i \in \{1, \dots, p\}$, $j \in \{1, \dots, D\}$.
Then one can define complex-valued random vectors \eqref{eq:complex-weights-properties} as 
\begin{equation} \label{eq:comp-vec-from-two-reals}
    \mat{z}_{i,j} := \sqrt{\frac{1}{2}} (\mat{v}_{i,j} + \iu \mat{w}_{i,j}) \in \mathbb{C}^d, \quad i \in \{1, \dots, p\}, \quad j \in \{1, \dots, D\}.
\end{equation}
\end{example}

The following two examples are specific cases of \cref{example:two-random-reals-complex} and are complex versions of the real-valued Rademacher and Gaussian sketches discussed previously.

\begin{example}[Complex Rademacher Sketch]
\label{example:complex-rademacher-sketch}

In \cref{example:two-random-reals-complex},  suppose that elements of random vectors $\mat{v}_{i,j}$ and $\mat{w}_{i,j}$ are independently sampled from the Rademacher distribution, i.e., sampled uniformly from $\{1, -1\}$. Then the resulting random vectors $\mat{v}_{i,j}$, $\mat{w}_{i,j}$ satisfy the conditions in \cref{eq:two-real-for-complex-properties} and thus the complex random vectors in \cref{eq:comp-vec-from-two-reals} satisfy the condition \cref{eq:complex-weights-properties}. 

\end{example}

\begin{example}[Complex Gaussian Sketch]
\label{example:complex-Gauss-sketch}
In \cref{example:two-random-reals-complex}, suppose that elements of random vectors $\mat{v}_{i,j}$ and $\mat{w}_{i,j}$ are independently sampled from the standard Gaussian distribution, $\mathcal{N}(0,1)$. Then the resulting random vectors $\mat{v}_{i,j}$, $\mat{w}_{i,j}$ satisfy the conditions in \cref{eq:two-real-for-complex-properties} and thus the complex random vectors in \cref{eq:comp-vec-from-two-reals} satisfy the condition \cref{eq:complex-weights-properties}. 
\end{example}

\begin{example} \label{example:complex-rademacher-uniform}
Suppose the elements of each random vector $\mat{z}_{i,j} \in \mathbb{C}^d$ are
independently sampled from the uniform distribution on $\{ 1, -1, \iu, - \iu \}$.
Then the requirement in \cref{eq:complex-weights-properties} is satisfied.
\end{example}

\cref{example:complex-rademacher-uniform} is essentially identical to the complex Rademacher sketch in \cref{example:complex-rademacher-sketch}, in that each element of $\mat{z}_{i,j}$ in \cref{example:complex-rademacher-uniform} can be obtained by multiplying $e^{\iu \pi /4}$ to an element of $\mat{z}_{i,j}$ in  \cref{example:complex-rademacher-sketch}, and vice versa. The multiplication by $e^{\iu \pi /4}$ is equivalent to rotating an element counter-clockwise by 45 degrees. 
See \cref{fig:complex-rotation} for an illustration.
One can see that this multiplication by $e^{\iu \pi /4}$ does not change the resulting approximate kernel \eqref{eq:approx-kernel-D-complex}. 
In this sense, the constructions of \cref{example:complex-rademacher-sketch} and  \cref{example:complex-rademacher-uniform} are equivalent.
 However,  the sketch in \cref{example:complex-rademacher-uniform} gives a computational advantage over \cref{example:complex-rademacher-sketch}: Since every element of each random vector $\mat{z}_{i,j}$ is either real {\em or} imaginary, the inner products $\mat{z}_{i,j}^{\top} \mat{x}$ in \cref{eqn:comp-rad-polynomial-sketch} can be computed at the same cost as for real polynomial sketches. %

\begin{figure}[t]
\centering
\includegraphics[width=0.8\textwidth]{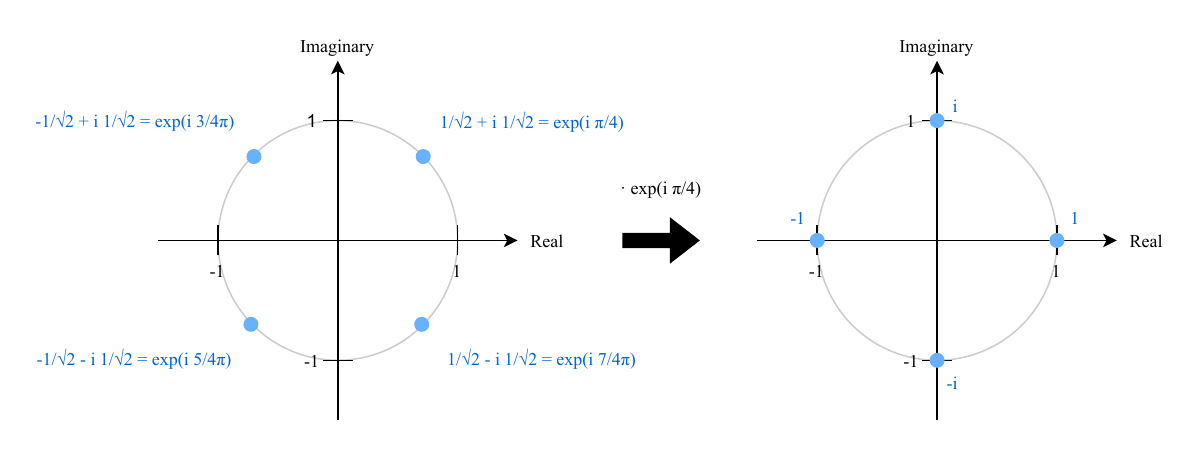}
\caption{Multiplying each element of a random vector $\mat{z}_{i,j}$ in \cref{example:complex-rademacher-sketch} by $\exp(\iu \frac{\pi}{4})$ corresponds to a counter-clockwise rotation of that element by 45 degrees on the complex plane. The support of the resulting elements is $\{1, -1, \iu, -\iu \}$ and the construction of \cref{example:complex-rademacher-uniform} is obtained.}
\label{fig:complex-rotation}
\end{figure}

We show in the following proposition that the approximate kernel \eqref{eq:approx-kernel-D-complex} is an unbiased estimator of the polynomial kernel $( \mat{x}^\top \mat{y} )^p$.  

\begin{proposition}
\label{prop:unbiased-complex-approximate-kernel}
Let $\mat{x}, \mat{y} \in \mathbb{R}^d$ be arbitrary, and  $\hat{k}_\mathcal{C}( \mat{x}, \mat{y} ) $ be the approximate kernel in \eqref{eq:approx-kernel-D-complex}.
Then we have
$$
  \mathbb{E}[ \hat{k}_\mathcal{C}( \mat{x}, \mat{y} ) ] = ( \mat{x}^\top \mat{y} )^p
$$
\end{proposition}

\begin{proof}
Since \cref{eq:approx-kernel-D-complex} is the empirical average of $D$ terms, it is sufficient to show the unbiasedness of each term. 
To this end, we consider here the case $D = 1$ and drop the index $j$. 
We have
\begin{align*}
    & \mathbb{E} \left[ \prod_{i=1}^p   \mscpi{z}{x}{i} \overline{\mscpi{z}{y}{i}} \right] 
    = \prod_{i=1}^p \mathbb{E} \left[ \mscpi{z}{x}{i} \overline{\mscpi{z}{y}{i}} \right] =
    \prod_{i=1}^p  \mat{x}^\top \mathbb{E} \left[ \mat{z}_i \overline{\mat{z}_i}^\top \right] \mat{y}
     = (\mat{x}^\top \mat{y})^p.
\end{align*}
where we used  \cref{eq:complex-weights-properties} in the last identity.
\end{proof}

\subsection{Variance of Complex-valued Polynomial Sketches} \label{sec:var-comp-poly}
We now study the variance of the approximate kernel \eqref{eq:approx-kernel-D-complex} with the complex-valued random feature map \eqref{eqn:comp-rad-polynomial-sketch}. We consider the case $D = 1$ and drop the index $j$:
\begin{equation} \label{eq:approx-kernel-complex-1}
    \hat{k}_\mathcal{C}(\mat{x}, \mat{y}) = \prod_{i=1}^p  \mscpi{z}{x}{i} \overline{\mscpi{z}{y}{i}}.
\end{equation}
The variance of the case $D > 1$ can be obtained by dividing the variance of \cref{eq:approx-kernel-complex-1} by $D$, since the approximate kernel \eqref{eq:approx-kernel-D-complex} is the average of $D$ i.i.d.~copies of \cref{eq:approx-kernel-complex-1}. We denote by $z_{i,k}$ the $k$-th element of $\mat{z}_i$. 

Note that the variance of a complex random variable $Z \in \mathbb{C}$ is defined by
$$
\mathbb{V}[Z] := \mathbb{E}[ | Z - \mathbb{E}[Z] |^2 ] = \mathbb{E}[ ( Z - \mathbb{E}[Z] ) \overline{( Z - \mathbb{E}[Z] )} ] = \mathbb{E}[|Z|^2] - | \mathbb{E}[Z] |^2
$$
Theorem \ref{thm:expression-var-complex-aprox-kernel} below characterizes the variance in terms of the input vectors $\mat{x}, \mat{y} \in \mathbb{R}^d$ and the distribution of the complex weight vectors \eqref{eq:complex-weights-properties}.  
The proof is given in \cref{sec:appendix-complex-variance}.

\begin{theorem}  
\label{thm:expression-var-complex-aprox-kernel}

Let $\mat{x} := (x_1, \dots, x_d)^\top \in \mathbb{R}^d$ and $\mat{y} := (y_1, \dots, y_d)^\top \in \mathbb{R}^d$ be any input vectors. 
Let $\mat{z}_1,\dots, \mat{z}_p \in \mathbb{C}^d$ be i.i.d.~random vectors satisfying \eqref{eq:complex-weights-properties}, such that elements $z_{i1}, \dots, z_{id}$ of each vector $\mat{z}_i = (z_{i1}, \dots, z_{id})^\top$ are themselves i.i.d.
Let $\mat{z} = (z_1, \dots, z_d)^\top \in \mathbb{C}^d$ be a random vector  independently and identically distributed as $\mat{z}_1,\dots, \mat{z}_p$, and write $z_k = a_k + \iu b_k$ %
with $a_k, b_k \in \mathbb{R}$.
Suppose 
\begin{equation} \label{eq:coro-complex-var-condition}
    \mathbb{E}[a_k b_k] = 0, \quad \mathbb{E}[a_k^2] = q, \quad \mathbb{E}[b_k^2] = 1-q \quad \text{where}\quad  0 \leq q \leq 1.
\end{equation}
Then, for the approximate kernel \eqref{eq:approx-kernel-complex-1}, we have
\begin{align}  
\mathbb{V}[ \hat{k}_\mathcal{C}(\mat{x}, \mat{y}) ] &=  \bigg( \sum_{k=1}^d \mathbb{E}[|z_k|^4] x_k^2 y_k^2 +  \| \mat{x} \|^2 \| \mat{y} \|^2 -  \sum_{k=1}^d x_k^2 y_k^2 \nonumber \\
& \quad +  \left( (2q - 1)^2 + 1 \right) \big( (\dotprod{x}{y})^2 -  \sum_{k=1}^d x_k^2 y_k^2  \big)   \bigg)^p  - (\mat{x}^\top \mat{y})^{2p} \label{eq:complex-poly-var-coro-1} 
\end{align}

\end{theorem}

\cref{thm:expression-var-complex-aprox-kernel} applies to a spectrum of complex polynomial sketches in terms of $q$, where the case $q = 1$ is the case of real-valued polynomial sketches in \cref{eq:approx-kernel-1d-kar12}. 
Indeed, to our knowledge, this result is also new for real-valued polynomial sketches. This variance formula is not only of theoretical interest, but also offers a way of estimating the variance from data. It will be used later to define the objective function of the proposed optimized Maclaurin approach.  

\paragraph{Condition in \cref{eq:coro-complex-var-condition}.}
The key condition is \cref{eq:coro-complex-var-condition},\footnote{\cref{eq:coro-complex-var-condition} implies that $z_k$ is a {\em proper} complex random variable \citep{neeser1993proper}.} where the constant $q$ is the average length of the real part $a_k$ of each random element $z_k = a_k + \iu b_k$.
Note that $\mathbb{E}[b_k^2] = 1 - q$ follows from $\mathbb{E}[a_k^2] = q$ since $1 = \mathbb{E}[|z_k|^2] = a_k^2 + b_k^2$. 
\cref{eq:coro-complex-var-condition} is satisfied for Examples \ref{example:complex-rademacher-sketch}, \ref{example:complex-Gauss-sketch} and \ref{example:complex-rademacher-uniform} with $q = 1/2$ and for the real-valued Rademacher and Gaussian sketches with $q = 1$. 
If $z_k$ is sampled uniformly from $\{ \iu, - \iu \}$, which is eligible as it satisfies \cref{eq:complex-weights-properties}, then $q = 0$. 
If $z_k$ is sampled uniformly from $\{ 1, - 1\}$ with probability $q$ and from $\{ \iu, - \iu \}$ with probability $1 - q$, then \cref{eq:coro-complex-var-condition} is satisfied with this $q$.

\paragraph{Lower Bound.}
The variance in \cref{eq:complex-poly-var-coro-1} can be lower-bounded by using Jensen's inequality $\mathbb{E}[ |z_k|^4 ] \geq ( \mathbb{E}[ |z_k|^2 ] )^2 = 1$:
\begin{align}
\mathbb{V}[ \hat{k}_\mathcal{C}(\mat{x}, \mat{y}) ]     & \geq \left(  \| \mat{x} \|^2 \| \mat{y} \|^2  +  \left( (2q - 1)^2 + 1 \right) \big( (\dotprod{x}{y})^2 -  \sum_{k=1}^d x_k^2 y_k^2  \big)    \right)^p - (\mat{x}^\top \mat{y})^{2p}. \label{eq:complex-poly-var-coro-2}
\end{align}
\cref{eq:complex-poly-var-coro-2} is the smallest possible variance attainable by complex polynomial sketches satisfying the conditions in \cref{thm:expression-var-complex-aprox-kernel}. 
For $q = 1/2$, this lower bound is attained by the complex Rademacher sketch (\cref{example:complex-rademacher-sketch}) and its equivalent construction (\cref{example:complex-rademacher-uniform}), for which we have $\mathbb{E}[ |z_k|^4] = 1$.
On the other hand, for the complex Gaussian sketch (\cref{example:complex-Gauss-sketch}) we have $\mathbb{E}[ | z_k |^4 ] = \mathbb{E} [(a_k^2 + b_k^2)^2] = 2$.

\paragraph{Concrete Examples.}
Below, we summarize the variance formula for the real and complex Rademacher sketches, and the real and complex Gaussian sketches:
\begin{align}
 \text{\bf (Real~Radem.)}\quad \mathbb{V} \left[ \hat{k}_\mathcal{R}(\mat{x}, \mat{y}) \right] 
& = \left(\norm{x}^2 \norm{y}^2 + 2 \left[ ( \mat{x}^\top \mat{y} )^2 - \sum_{k=1}^d x_k^2 y_k^2 \right] \right)^p - (\mat{x}^\top \mat{y})^{2p}, \label{eqn:rademacher-variance}    \\
\text{\bf (Comp.~Radem.)} \quad  \mathbb{V}[ \hat{k}_\mathcal{C}(\mat{x}, \mat{y}) ]
& = \left(\norm{x}^2 \norm{y}^2 + ( \mat{x}^\top \mat{y} )^2 - \sum_{k=1}^d x_k^2 y_k^2 \right)^p - (\mat{x}^\top \mat{y})^{2p}
\label{eqn:complex-rademacher-variance} \\
 \text{\bf (Real~Gauss.)}\quad \mathbb{V} \left[ \hat{k}_\mathcal{R}(\mat{x}, \mat{y}) \right] 
& = \left(\| \mat{x} \|^2 \| \mat{y} \|^2 + 2 ( \mat{x}^\top \mat{y} )^2 \right)^p - (\mat{x}^\top \mat{y})^{2p}. \label{eqn:normal-polynomial-est-variance} \\
\text{\bf (Comp.~Gauss.)} \quad  \mathbb{V}[ \hat{k}_\mathcal{C}(\mat{x}, \mat{y}) ]
& = \left(\| \mat{x} \|^2 \| \mat{y} \|^2 + ( \mat{x}^\top \mat{y} )^2 \right)^p - (\mat{x}^\top \mat{y})^{2p}
\label{eqn:complex-normal-variance}
\end{align}

\paragraph{Comparing the Real and Complex Polynomial Sketches.}
Let us now compare the variances of real ($q = 1$) and complex ($q \not= 1$) polynomial sketches. 
First, it is easy to see that the variance of the complex Gaussian polynomial sketch (\cref{eqn:complex-normal-variance}) is upper-bounded by the variance of the real Gaussian sketch (\cref{eqn:normal-polynomial-est-variance}).
For the lower bound in \cref{eq:complex-poly-var-coro-2}, which is attained by the Rademacher sketches, a more detailed analysis is needed.
To this end, consider the term that depends on $q$:
\begin{equation*}
 \left( (2q - 1)^2 + 1 \right) \big( (\dotprod{x}{y})^2 -  \sum_{k=1}^d x_k^2 y_k^2  \big)  
\end{equation*}
The variance in \cref{eq:complex-poly-var-coro-1} is a monotonically increasing function of this term. 
Suppose 
\begin{equation} \label{eq:cond-for-complex-to-be-better}
   (\dotprod{x}{y})^2 -  \sum_{k=1}^d x_k^2 y_k^2  = \sum_{i=1}^d \sum_{\substack{j=1 \\ j \neq i}}^d   x_i x_j y_i y_j \geq  0
\end{equation}
Then $q = 1/2$ (e.g., complex sketches in Examples \ref{example:complex-rademacher-sketch}, \ref{example:complex-Gauss-sketch} and \ref{example:complex-rademacher-uniform}) makes the term the smallest, while $q = 1$ and $q=0$ (purely real and imaginary polynomial sketches) makes it the largest. In other words, for input vectors $\mat{x}$ and $\mat{y}$ satisfying \cref{eq:cond-for-complex-to-be-better}, complex-valued sketches with $q = 1/2$ result in a lower variance than the real-valued counterparts with $q =1$.
On the other hand, if \cref{eq:cond-for-complex-to-be-better} does not hold, real-valued sketches result in a lower variance than the complex-valued counterparts. 

Therefore, whether complex-valued Rademacher sketches ($q=1/2$) yield a lower variance than real-valued Rademacher sketches ($q=1$) depends on whether \cref{eq:cond-for-complex-to-be-better} holds. 
For example, \cref{eq:cond-for-complex-to-be-better} holds true if input vectors $\mat{x} = (x_1, \dots, x_d)^\top$ and $\mat{y} = (y_1, \dots, y_d)^{\top}$ are {\em nonnegative}: $x_1, ..., x_d \geq 0$ and $y_1, \dots, y_d \geq 0$. 
Nonnegative input vectors are ubiquitous in real-world applications, e.g., where each input feature represents the amount of a certain quantity, where input vectors are given by bag-of-words representations,  one-hot encoding (categorical data), or min-max feature scaling,  and where they are outputs of a ReLU neural network\footnote{c.f. DeepFried Convnets \citep{Yang2015} and fine-grained image recognition \citep{Gao2016}.}.  
For such applications with nonnegative input vectors, complex-valued polynomial sketches always yield a smaller variance than the real-valued counterparts.

\subsection{Probabilistic Error Bounds for Rademacher Sketches}

\label{sec:prob-error-comp-poly}

We present here probabilistic error bounds for the approximate kernel in \cref{eq:approx-kernel-D-complex} in terms of the number $D$ of random features, using the variance formula obtained in the previous subsection and focusing on Rademacher sketches. The proof of the following result is given in \cref{sec:appendix-rademacher-bound}.

\begin{theorem}
    \label{trm:real-bernstein-bound}
     Let $\mat{x}, \mat{y} \in \mathbb{R}^d$ be arbitrary input vectors.
    For $0 \leq q \leq 1$, consider a polynomial sketch in \cref{thm:expression-var-complex-aprox-kernel} such that $\mathbb{E}[|z_k|^4] = 1$ and thus it attains the variance in \cref{eq:complex-poly-var-coro-2}.
    Define a constant $\sigma^2 \geq 0$ by
    $$
    \sigma^2 := \frac{ 1}{\|\mat{x}\|^{2p} \|\mat{y}\|^{2p}} \left[ \left(  \| \mat{x} \|^2 \| \mat{y} \|^2  +  \left( (2q - 1)^2 + 1 \right) \big( (\dotprod{x}{y})^2 -  \sum_{k=1}^d x_k^2 y_k^2  \big)    \right)^p - (\mat{x}^\top \mat{y})^{2p} \right]
    $$
    Let $\epsilon, \delta > 0$ be arbitrary, and $D \in \mathbb{N}$ be such that 
    \begin{equation} \label{eq:real-rademacher-scaling-D}
    D \geq 2 \left( \frac{2}{3 \epsilon} + \frac{ \sigma^2}{\epsilon^2} \right) \log \left( \frac{2}{\delta} \right).
    \end{equation}
     Then, for  the approximate kernel $\hat{k}_{\mathcal{C}}(\mat{x}, \mat{y})$  in \cref{eq:approx-kernel-D-complex}, we have
    \begin{equation*}
        \mathrm{Pr} \left[\left| \hat{k}_\mathcal{C} (\mat{x}, \mat{y}) - (\mat{x}^\top \mat{y})^p \right| \leq \epsilon \norm{x}_1^p \norm{y}_1^p \right] \geq 1 - \delta .
    \end{equation*}
\end{theorem}

\cref{eq:real-rademacher-scaling-D} shows that the required number $D$ of random features to achieve the relative accuracy of $\varepsilon$ (where the ``relative'' is with respect to $ \norm{x}_1^p \norm{y}_1^p$) with probability at least $1-\delta$.  
For small $\epsilon$,  the second term $\sigma^2 / \epsilon^2$ dominates the first term $2/(3\epsilon)$. This second term depends on $\sigma^2$, which is a scaled version of the variance in \cref{eq:complex-poly-var-coro-2} of the approximate kernel for $D = 1$. 
Thus, if the variance in \cref{eq:complex-poly-var-coro-2} is smaller (resp.~larger), one needs a smaller (resp.~larger) number of random features to achieve the relative accuracy of $\epsilon$.
 
Let us now compare the real-valued ($q = 1$) and complex-valued ($q = 1/2$) Rademacher sketches. 
As discussed earlier, the complex Rademacher sketch has a smaller variance than the real Rademacher sketch when the inequality in \cref{eq:cond-for-complex-to-be-better} holds for the two input vectors $\mat{x}, \mat{y}$. In particular, this inequality always holds when the input vectors are nonnegative. 
Therefore, in this case, the complex Rademacher sketch requires a smaller number of random features than the real Rademacher sketch to achieve a given accuracy.

When the inequality in \cref{eq:cond-for-complex-to-be-better} holds, the advantage of the complex Rademacher sketch becomes more significant for larger $p$.
We illustrate this behavior in \cref{fig:comparison-comp-real-over-p}, where vectors $\mat{x}$ and $\mat{y}$ are sampled randomly in the positive quadrant by first drawing $\tilde{\mat{x}}$ and $\tilde{\mat{y}}$ from $\mathcal{N}(\mat{0}, \mat{I}_d)$ and then by computing $\mat{x}=|\tilde{\mat{x}}| / \| \tilde{\mat{x}} \|, \mat{y}=|\tilde{\mat{y}}| / \| \tilde{\mat{y}} \|$.
(Here $|\tilde{\mat{x}}|$ denotes the vector whose elements are the absolute values of the elements of $\tilde{\mat{x}}$.)
In the figure we report the mean squared error of the approximate kernel for $d=8$ and $d=32$ for Gaussian and Rademacher sketches. 
To facilitate quantitative analysis of the improvement offered by complex random features, we also report the ratio between the mean squared error obtained by complex and real random features.
As expected from the theoretical analysis, the results show a decreasing error ratio for increasing values of $p$, with an improvement of Rademacher over Gaussian, which is slightly larger for lower $d$.

In \cref{fig:comparison-comp-real-over-p} we also included the tree construction in \citet[Alg.~1]{Ahle2020} with (complex) Rademacher sketches as nodes.
This approach has an improved scaling w.r.t. $p$ compared to our method; however, for commonly used degrees $p=2,3$ our method outperforms it. 
Another interesting observation from the figure is that \citet{Ahle2020} does not benefit from complex sketches in the same way as our approach, even though this yields an error ratio below one.
Based on these results, we recommend our approach over the tree method by \citet{Ahle2020} for low degrees $p=2,3$ as it is competitive and easier to implement. 
In \cref{sec:experiments} we provide more evidence to support these results.

\begin{figure}[t] 
\centering
\includegraphics[width=1.0\textwidth]{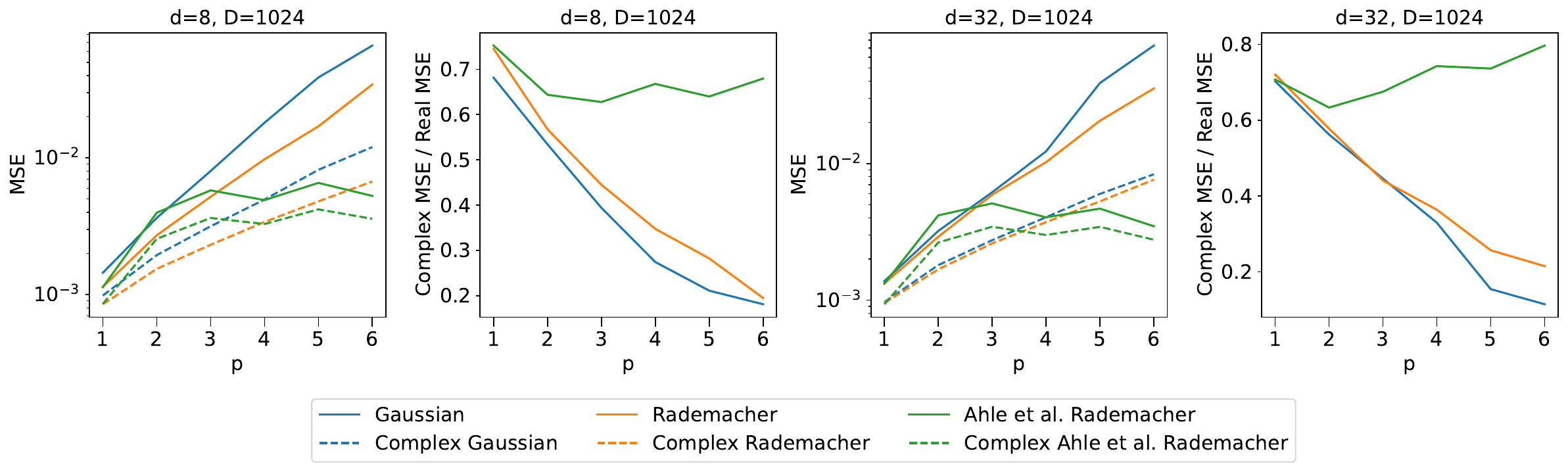}
\caption{  This plot shows the mean squared error $\mathbb{E}[ | \hat{k}(\mat{x}, \mat{y}) - (\mat{x}^\top \mat{y})^p |^2 ]$ for different values of the degree $p$ as well as the mean squared error ratio of the complex sketches over their real analogues. We sample 1,000 independent vector pairs $(\mat{x}, \mat{y}) \in \mathbb{R}^d \times \mathbb{R}^d$ with $\mat{x}=|\tilde{\mat{x}}| / \| \tilde{\mat{x}} \|, \mat{y}=|\tilde{\mat{y}}| / \| \tilde{\mat{y}} \|$ and $\tilde{\mat{x}}, \tilde{\mat{y}} \sim \mathcal{N}(\mat{0}, \mat{I}_d)$. The mean is then taken over 1,000 independent constructions of the approximate kernel $\hat{k}(\mat{x}, \mat{y})$, and over every input pair $(\mat{x}, \mat{y})$.}
\label{fig:comparison-comp-real-over-p}
\end{figure}

\section{Structured Polynomial Sketches }
\label{sec:structured-projections}

We study here {\em structured} polynomial sketches and their extensions with complex features. In Section \ref{sec:polynomial-sketches}, we studied polynomial sketches in \cref{eqn:polynomial-estimator} (or \cref{eqn:comp-rad-polynomial-sketch} for complex extensions), where the $p \times D$ random vectors $\mat{w}_{i, \ell} \in \mathbb{R}^d$ ($i=1,\dots,p$, $\ell = 1,\dots,D$) are generated in an i.i.d.~manner. 
By putting a structural constraint on these vectors, one can construct more efficient random features with a lower variance.
Moreover, such a structural constraint leads to a computational advantage, as the imposed structure may be used for implementing an efficient algorithm for fast matrix multiplication.

We consider structured polynomial sketches known as {\em TensorSRHT (Tensor Subsampled Randomized Hadamard Transform)}.  
 \citet{Tropp2011} studied TensorSRHT for $p=1$ (linear case) and \citet{Hamid2014,Ahle2020} extended it\footnote{ The sketches proposed by \citet{Hamid2014} and \citet{Ahle2020} are similar but not exactly the same. \citet{Hamid2014} uses $p \times B$ independent linear SRHT sketches \citep[see][]{Tropp2011}, where $B := \left\lceil \frac{D}{d} \right\rceil$ is the number of SRHT blocks per degree. The elements of these sketches are then shuffled across degrees and blocks, and the blocks are multiplied elementwise over $p$. \citet{Ahle2020} compute only $p$ independent sketches and subsample from their tensor product instead.} to arbitrary polynomial degrees $p$. 
 \citet{Ahle2020} introduced the name TensorSRHT, which refers to the fact that it implicitly sketches from the $d^p$-dimensional space of tensorized inputs $\mat{x}^{(p)}$, as shown in \cref{eqn:random-tensor-expansion}.

In Section \ref{sec:real-tensorSRHT}, we introduce TensorSRHT with real features, and present its extension using complex features in Section \ref{sec:complex-TensorSRHT}.
We then make a comparison between the real and complex versions of TensorSRHT in Section \ref{sec:compare-tensorSRHT}.

 Note that the TensorSRHT algorithm presented here is a slight modification\footnote{Instead of permuting elements across degrees and blocks, we only permute the elements inside each block as it is done for SRHT \citep[see][]{Tropp2011}. Our sketch and the one proposed by \citet{Ahle2020} are equivalent when $D \leq d$ and different when $D > d$.} of the one proposed by \citet{Hamid2014}. 
We show that our modification still yields unbiased approximations of polynomial kernels. 
It further allows us to derive the variance of the sketch in closed form, which has not been done in any of the previous works. We show, for the first time, that TensorSRHT is more efficient than Rademacher sketches for odd $p$. %

\subsection{Real TensorSRHT} \label{sec:real-tensorSRHT}

TensorSRHT imposes an orthogonality constraint on the vectors $\mat{w}_{i, 1}, \dots, \mat{w}_{i, D}$ through predefined structured matrices, specifically {\em unnormalized Hadamard matrices}.
Let $n := 2^m$ with $m \in \mathbb{N}$, and define $\mat{H}_n \in \{ 1, - 1\}^{n \times n}$ to be the unnormalized Hadamard matrix, which is recursively defined as 
\begin{align} \label{eq:hadamard}
    \mat{H}_{2n} :=
    \begin{bmatrix}
        \mat{H}_n & \mat{H}_n \\
        \mat{H}_n & -\mat{H}_n
    \end{bmatrix},
    \quad
    \text{with}
    \quad
    \mat{H}_2 :=
    \begin{bmatrix}
        1 & 1\\
        1 & -1
    \end{bmatrix}.
\end{align}
From now on, we always use $\mat{H}_d \in \{ 1, -1 \}^{d \times d}$ with $d$ being the dimensionality of input vectors, assuming $d = 2^m$ for some $m \in \mathbb{N}$. If $d \not= 2^m$ for any $m \in \mathbb{N}$, we pad $0$ to input vectors until their dimensionality becomes $2^m$ for some $m \in \mathbb{N}$.
For $i=1,\dots,d$, let $\mat{h}_i \in \{1, -1 \}^d$ be the $i$-th column 
of $\mat{H}_d$, i.e.,
$$
\mat{H}_d = ( \mat{h}_1, \dots, \mat{h}_d )  \in \{1 , -1 \}^{d \times d}.
$$
Note that we have $
\mat{H}_d \mat{H}_d^{\top} = \mat{H}_d^{\top} \mat{H}_d = d \mat{I}_d$,
which implies that distinct columns (and rows) of $\mat{H}_d$ are orthogonal to each other, i.e., $\mat{h}_{i}^\top \mat{h}_j = 0$ for $i \not= j$.

\paragraph{\underline{Case $D = d$.}} For ease of explanation, suppose here that  the number $D$ of random features is equal to the dimensionality $d$ of input vectors: $D = d$. 
For $i = 1, \dots, p$, define $\mat{w}_{i} \in \mathbb{R}^d$ as a random vector whose elements are i.i.d.~Rademacher random variables:
$$
\mat{w}_i := (w_{i,1}, \dots, w_{i,d})^\top \in \mathbb{R}^d, \quad w_{i,j}  \stackrel{i.i.d.}{\sim} {\rm unif}(\{1, -1 \})\quad (j = 1, \dots, d)
$$
Consider a random permutation of the indices $\pi: \{1,\dots, d\} \to \{1,\dots,d\}$, and let
$$\pi(1), \dots, \pi(d)$$ be the permuted indices.
For $i = 1, \dots, p$ and $\ell = 1, \dots, D$, we then define a random vector $\mat{s}_{i, \ell} \in \mathbb{R}^d$ as the Hadamard product (i.e., element-wise product) of the Rademacher vector $\mat{w}_i$ and the permuted column $\mat{h}_{\pi(\ell)}$ of the Hadamard matrix: 
\begin{equation} \label{eq:weight-vector-SRHT}
    \mat{s}_{i, \ell} := \mat{w}_{i} \circ \mat{h}_{\pi(\ell)} = ( w_{i,1} h_{\pi(\ell), 1}, \dots, w_{i,d} h_{\pi(\ell), d} )^\top \in \mathbb{R}^d,  
\end{equation}
where $h_{\pi(\ell), j}$ denotes the $j$-th element of $\mat{h}_{\pi(\ell)}$.

Because of the orthogonality of the columns $\mat{h}_1, \dots, \mat{h}_d$ of the Hadamard matrix $\mat{H}_d$, the random weight vectors $\mat{s}_{i, 1}, \dots, \mat{s}_{i, d}$ are orthogonal to each other almost surely: for $\ell \not= m$, we have
\begin{align*}
 \mat{s}_{i, \ell}^\top \mat{s}_{i, m} =     (\mat{w}_{i} \circ \mat{h}_{\pi(\ell)} )^\top (\mat{w}_{i} \circ \mat{h}_{\pi(m)} ) = \sum_{j = 1}^d w_{i, j}^2 \mat{h}_{\pi(\ell), j} \mat{h}_{\pi(m), j} = \mat{h}_{\pi(\ell)}^\top \mat{h}_{\pi(m)} = 0.
\end{align*}
Note also that, given the permutation $\pi(1), \dots, \pi(d)$, the elements of each random vector $\mat{s}_{i, \ell}$ in \eqref{eq:weight-vector-SRHT} are i.i.d.~Rademacher variables.

Finally, we define a random feature map $\Phi_\mathcal{R}(\mat{x}): \mathbb{R}^d \to \mathbb{R}^D$ for the case $D = d$ as
\begin{align} \label{eq:tensorSRHT}
    \Phi_\mathcal{R}(\mat{x})
    :=  \frac{1}{\sqrt{D}} \left[ (\prod_{i=1}^p \dotprodi{s}{x}{i,1}), \dots, (\prod_{i=1}^p \dotprodi{s}{x}{i,d}) \right]^{\top} \in \mathbb{R}^d,
\end{align}
 which defines an approximate kernel as
 $$\hat{k}_\mathcal{R}(\mat{x}, \mat{y}) := \Phi_\mathcal{R}(\mat{x})^\top  \Phi_\mathcal{R}(\mat{y}) = \frac{1}{D}\sum_{\ell = 1}^D \Phi(\mat{x})_{\mathcal{R}, \ell} \Phi(\mat{y})_{\mathcal{R}, \ell}$$ 
 where $\Phi(\cdot)_{\mathcal{R}, \ell}$ denotes the $\ell$-th element of $\Phi_{\mathcal{R}}(\cdot)$. 
 
 The orthogonality of the weight vectors in \cref{eq:weight-vector-SRHT} 
 leads to {\em negative covariances} between the terms $\Phi(\mat{x})_{\mathcal{R}, \ell} \Phi(\mat{y})_{\mathcal{R}, \ell}$ and $\Phi(\mat{x})_{\mathcal{R}, m} \Phi(\mat{y})_{\mathcal{R}, m}$ with distinct indices $\ell \not= m$ in the approximate kernel. 
These negative covariances decrease the overall variance of the approximate kernel, as we will show later in \cref{lemma:variance-formula-struct} and \cref{sec:analysis-variances-tensorSRHT} 

\paragraph{\underline{Case $D \not= d$.}} 
We now explain the case $D \not =d$. If $D< d$, we first compute the feature map in \cref{eq:tensorSRHT} and keep the first $D$ components of it. 
If $D > d$, we independently generate the feature map in \cref{eq:tensorSRHT}  $B := \left\lceil \frac{D}{d} \right\rceil$ times and concatenate the resulting $B$ vectors to obtain a $Bd$-dimensional feature map, and then discard the redundant last $Bd - D$ components of it to obtain a $D$-dimensional feature map. %
In this way, we can obtain a $D$-dimensional feature map for arbitrary $D \in \mathbb{N}$, which we write as
\begin{align}
    \label{eqn:real-struct-polynomial-sketch}
    \Phi_\mathcal{R}(\mat{x})
    :=  \frac{1}{\sqrt{D}} \left[ (\prod_{i=1}^p \dotprodi{s}{x}{i,1}), \dots, (\prod_{i=1}^p \dotprodi{s}{x}{i,D}) \right]^{\top} \in \mathbb{R}^D.
\end{align}
The entire procedure for constructing the structured polynomial sketch in \cref{eqn:real-struct-polynomial-sketch} is outlined in \cref{alg:tensor-srht-algorithm}, where we also cover the complex-valued case discussed later. 

In \cref{alg:tensor-srht-algorithm}, we use an equivalent matrix formulation, since it enables the Fast Walsh-Hadamard transform by employing the associativity, and thus the feature map can be computed much faster.
To explain this more precisely, let $\mat{D}_i := {\rm diag}(w_{i1}, \dots, w_{id})\in \mathbb{R}^{d \times d}$ be a diagonal matrix whose diagonal entries $w_{i1}, \dots, w_{id} \in \{1, -1\}$ are i.i.d.~Rademacher random variables, and $\mat{P}_\pi := ( \mat{e}_{\pi(1)}, \dots, \mat{e}_{\pi(d)} ) \in \mathbb{R}^{d \times d}$ be a permutation matrix, where $\mat{e}_{\pi(\ell)} \in \mathbb{R}^d$ is a vector whose $\pi(\ell)$-th element is $1$ and other elements are $0$.
We can then compute 
$$(\mat{s}_{i,1}^\top \mat{x}, \dots, \mat{s}_{i,d}^\top \mat{x}) = \mat{x}^\top (\mat{s}_{i,1}, \dots, \mat{s}_{i,d}) =  \mat{x}^\top ( \mat{D}_i \mat{H}_d \mat{P}_\pi ) = \left( (\mat{x}^\top \mat{D}_i) \mat{H}_d \right) \mat{P}_\pi
$$ 
by 1) first computing $\mat{x}^\top \mat{D}_i$, 2) then multiplying the Hadamard matrix $\mat{H}_d$ using the Fast Walsh-Hadamard transform, and 3) lastly multiplying the permutation matrix $\mat{P}_\pi$, which is more efficient than first precomputing $ \mat{D}_i \mat{H}_d \mat{P}_\pi$ and then multiplying $\mat{x}^\top$. 
In this way, thanks to the Fast Walsh-Hadamard transform, $(\mat{s}_{i, 1}^\top \mat{x}, \dots, \mat{s}_{i, d}^\top \mat{x})$ can be computed in $\bigO(d \log{d})$ instead of $\bigO(d^2)$ \citep{Fino1976}.
The total computational complexity is therefore $\bigO(p D \log{d})$ and the memory requirement is $\bigO(p D)$, and this is a computational advantage of TensorSRHT.

\begin{algorithm}[t]
\caption{Real and Complex TensorSRHT}
\label{alg:tensor-srht-algorithm}
    \SetAlgoLined
    \KwResult{A feature map $\Phi_{\mathcal{R}/\mathcal{C}}(\mat{x})$}
    Pad $\mat{x}$ with zeros so that $d$ becomes a power of $2$ \;
    Let $B = \left\lceil \frac{D}{d} \right\rceil$ be the number of stacked projection blocks \;

    \ForAll{$b \in \{1, \dots, B\}$}{ 
        \ForAll{$i \in \{1, \dots, p\}$}{
    \underline{\bf Real case} Generate a random vector ${\bf w}_i = (w_{i,1}, \dots, w_{i,d})^\top \in  \mathbb{R}^d$ as $w_{i,1}, \dots, w_{i,d} \stackrel{i.i.d}{\sim} {\rm unif}( \{ 1, - 1\} )$, and define a diagonal matrix $\mat{D}_i := {\rm diag}(\mat{w}_i) \in \mathbb{R}^{d \times d}$;
    
\underline{\bf Complex case} Generate a random vector ${\bf z}_i = (z_{i,1}, \dots, z_{i,d})^\top \in  \mathbb{C}^d$ as $z_{i,1}, \dots, z_{i,d} \stackrel{i.i.d}{\sim} {\rm unif}( \{ 1, - 1, \iu, - \iu \} )$, and define a diagonal matrix $\mat{D}_i := {\rm diag}(\mat{w}_i) \in \mathbb{C}^{d \times d}$ ;

       Randomly permute the indices $1, \dots, d$ to $\pi(1), \dots, \pi(d)$ \;
       
       Let $\mat{P}_\pi := ( \mat{e}_{\pi(1)}, \dots, \mat{e}_{\pi(d)} ) \in \mathbb{R}^{d \times d}$, where $\mat{e}_{\pi(\ell)} \in \mathbb{R}^d$ is a vector whose $\pi(\ell)$-th element is $1$ and other elements are $0$\ ($\ell = 1,\dots,d$) \;
       
       Let $(\mat{s}_{i, 1}, \dots, \mat{s}_{i, d}) := \mat{D}_i \mat{H}_d \mat{P}_\pi$ \; 
        }
        Compute $\Phi_b(\mat{x}) := \sqrt{1/D} [ (\prod_{i=1}^p \dotprodi{s}{x}{i,1}), \dots, (\prod_{i=1}^p \dotprodi{s}{x}{i,d}) ]^{\top}$ \;
    }
  
    Concatenate the elements of $\Phi_1(\mat{x}), \dots, \Phi_B(\mat{x})$ to yield a single projection vector $\Phi_{\mathcal{R}/\mathcal{C}}(\mat{x})$ and keep the first $D$ entries \;
\end{algorithm}

The feature map in \cref{eqn:real-struct-polynomial-sketch} induces an approximate kernel $\hat{k}_\mathcal{R}( \mat{x}, \mat{y} ) = \Phi_{\mathcal{R}} (\mat{x})^{\top} \Phi_{\mathcal{R}} (\mat{y})$. 
The following proposition summarizes that this approximate kernel is unbiased with respect to the target polynomial kernel $k(\mat{x}, \mat{y}) = (\mat{x}^\top \mat{y})^p$.
As mentioned earlier,  TensorSRHT discussed here is slightly different from the existing versions. 
Therefore this result is novel in its own right. 
The result follows from Proposition \ref{prop:unbiased-struct-approximate-kernel-complex} in the next subsection, so we omit the proof.

\begin{proposition}
\label{prop:unbiased-struct-approximate-kernel}
Let $\mat{x}, \mat{y} \in \mathbb{R}^d$ be arbitrary, and  $\hat{k}_\mathcal{R}( \mat{x}, \mat{y} ) = \Phi_{\mathcal{R}} (\mat{x})^{\top} \Phi_{\mathcal{R}} (\mat{y})$ be the approximate kernel with $\Phi_{\mathcal{R}}(\mat{x}),  \Phi_{\mathcal{R}}(\mat{y}) \in \mathbb{R}^D$ given by the random feature map in \cref{eqn:real-struct-polynomial-sketch}.
Then we have
$ \mathbb{E}[ \hat{k}_\mathcal{R}( \mat{x}, \mat{y} ) ] = ( \mat{x}^\top \mat{y} )^p$.
\end{proposition}

We next study the variance of the approximate kernel given by TensorSRHT, which is the mean squared error of the approximate kernel since it is unbiased as shown above.  
The following theorem provides a closed form expression for the variance, whose proof is given for the more general complex case in \cref{sec:proof-var-comp-TensorSRHT}. 
It is a novel result and extends \citet[Theorem 3.3]{Choromanski2017} to the setting $p > 1$ and $D > d$.

\begin{theorem}[Variance of Real TensorSRHT] \label{lemma:variance-formula-struct} 
Let $\mat{x}, \mat{y} \in \mathbb{R}^d$ be arbitrary, and  $\hat{k}_\mathcal{R}( \mat{x}, \mat{y} ) = \Phi_{\mathcal{R}} (\mat{x})^{\top} \Phi_{\mathcal{R}} (\mat{y})$ be the approximate kernel with $\Phi(\mat{x}),  \Phi(\mat{y}) \in \mathbb{R}^D$ given by the random feature map in \cref{eqn:real-struct-polynomial-sketch}. 
Then we have
\begin{align}
  \mathbb{V} \left[ \hat{k}_\mathcal{R}( \mat{x}, \mat{y} ) \right]  
     = \underbrace{\frac{V_{\rm Rad}^{(p)}}{D} }_{(A)}  - \underbrace{ \frac{c(D,d)}{D^2}  \left[ (\mat{x}^\top \mat{y})^{2p}  -  \left( ( \mat{x}^\top \mat{y} )^2 - \frac{V_{\rm Rad}^{(1)}}{d-1}   \right)^p \right] }_{(B)}
    \label{eqn:tensor-srht-variance},
\end{align}
where  $V_{\rm Rad}^{(p)} \geq 0$ and $ V_{\rm Rad}^{(1)} \geq 0$ are the variances of the real Rademacher sketch with a single feature in \cref{eqn:rademacher-variance} with generic $p \in \mathbb{N}$ and $p = 1$, respectively,  and $c(D,d) \in \mathbb{N}$ is defined by
\begin{equation} \label{eq:non-zero-pairs}
   c(D,d) := \lfloor D/d \rfloor d(d-1) + {\rm mod}(D, d)( {\rm mod}(D, d) - 1).
\end{equation}

\end{theorem}

\begin{remark} \label{remark:var-real-tensorSRHT}
The constant $c(D,d)$ in \cref{eq:non-zero-pairs} is the number of  pairs of indices $\ell, \ell' = 1, \dots, D$ with $\ell \not= \ell'$  for which the covariance of the weight vectors $\mat{s}_{i, \ell}$ and $\mat{s}_{i, \ell'}$ in \cref{eqn:real-struct-polynomial-sketch} is non-zero (see the proof in Appendix \ref{sec:proof-var-comp-TensorSRHT} for details). 
If $D = Bd$ for some $B \in \mathbb{N}$, this constant simplifies to $c(D,d) = B d (d-1)$, and the variance in \cref{eqn:tensor-srht-variance} becomes 
$$
  \mathbb{V} \left[ \hat{k}_\mathcal{R}( \mat{x}, \mat{y} ) \right] =    \frac{1}{D} V_{\rm Rad}^{(p)}  - \frac{d-1}{D}  \left[ (\mat{x}^\top \mat{y})^{2p}  -  \left( ( \mat{x}^\top \mat{y} )^2 - \frac{V_{\rm Rad}^{(1)}}{d-1}  \right)^p \right].
$$
An interesting subcase is $p=1$,  for which the variance becomes zero. Thus, setting $D \in \{kd \mid k\in\mathbb{N}\}$ for $p=1$ is equivalent to using the linear kernel with the original inputs.
\end{remark}

\cref{lemma:variance-formula-struct} enables understanding the condition under which TensorSRHT has a smaller variance than the unstructured Rademacher sketch in  \cref{eqn:polynomial-estimator}.
Note that the term $(A)$ in \cref{eqn:tensor-srht-variance}  is the variance of the approximate kernel with the Rademacher sketch with $D$ features. 
On the other hand, the term $(B)$  in \cref{eqn:tensor-srht-variance} can be interpreted as the effect of the structured sketch. 
The term $(B)$ always becomes non-negative when $p$ is {\rm odd}, and thus the overall variance of TensorSRHT becomes smaller than the Rademacher sketch, as  summarized in the following corollary.
Thus, when $p$ is odd, TensorSRHT should be preferred over the Rademacher sketch.

\begin{corollary} \label{thrm:tensor-srht-odd-p}
Let $p \in \mathbb{N}$ be odd.
Then,  for all input vectors $\mat{x}, \mat{y} \in \mathbb{R}^d$, the variance of the approximate kernel with TensorSRHT in \cref{eqn:tensor-srht-variance} is smaller or equal to the variance of the approximate kernel with the Rademacher sketch: %
$$
\frac{V_{\rm Rad}^{(p)}}{D} - \frac{c(D,d)}{D^2}  \left[ (\mat{x}^\top \mat{y})^{2p}  -  \left( ( \mat{x}^\top \mat{y} )^2 - \frac{V_{\rm Rad}^{(1)}}{d-1}   \right)^p \right] \leq \frac{V_{\rm Rad}^{(p)}}{D}  
$$
\end{corollary}
\begin{proof}
Since, $V_{\rm Rad}^{(1)} \geq 0$, we have  $( \mat{x}^\top \mat{y} )^2 - \frac{1}{d-1} V_{\rm Rad}^{(1)} \leq ( \mat{x}^\top \mat{y} )^2$. For odd $p$ this leads to $\left( ( \mat{x}^\top \mat{y} )^2 - \frac{1}{d-1} V_{\rm Rad}^{(1)} \right)^p \leq (\mat{x}^\top \mat{y})^{2p}$. The assertion immediately follows.
\end{proof}

If $p$ is even, on the other hand, the variance of TensorSRHT can be larger than the Rademacher sketch for certain input vectors
 $\mat{x}, \mat{y} \in \mathbb{R}^d$.
For instance, if $\mat{x}$ and $\mat{y}$ are orthogonal, i.e., $\mat{x}^\top \mat{y} = 0$, then the variance of TensorSRHT in \cref{eqn:tensor-srht-variance} is 
$$
\text{\cref{eqn:tensor-srht-variance}} = \frac{V_{\rm Rad}^{(p)}}{D} + \frac{c(D,d)}{D^2} \left( \frac{ V_{\rm Rad}^{(1)} }{d-1} \right)^p \geq \frac{V_{\rm Rad}^{(p)}}{D}   .
$$
Therefore, for even $p$, we do not have a theoretical guarantee for the advantage of TensorSRHT over the Rademacher sketch in terms of their variances. 
In practice, however,  TensorSRHT has often  a smaller variance than the Rademacher sketch also for even $p$, as demonstrated in our experiments described later. 
Moreover, TensorSRHT has a computational advantage over the Rademacher sketch, thanks to the fast Walsh-Hadamard transform.

\begin{remark}
One can straightforwardly derive a probabilistic error bound for TensorSRHT by using \cref{lemma:variance-formula-struct} and Chebyshev's inequality. However, deriving an exponential tail bound for TensorSRHT is more involved, because different features in the feature map $\Phi_\mathcal{R}(\mat{x})$ in \cref{eq:tensorSRHT} are dependent for TensorSRHT and thus applying Bersnstein's inequality is not straightforward.
One can find an exponential tail bound for TensorSRHT in \citet[Lemma 33 in the longer version]{Ahle2020}, while they analyze a slightly different version of TensorSRHT from ours and their bound is a uniform upper bound that holds for all input vectors simultaneously. 

\end{remark}

Our variance formula in \cref{eqn:tensor-srht-variance}, which is a novel contribution to the literature, provides a precise characterization of how the variance of the approximate kernel depends on the input vectors $\mat{x}, \mat{y} \in \mathbb{R}^d$, and shows when TensorSRHT is more advantageous than the Rademacher sketch.
Moreover, as the variance formula can be computed in practice, it can be used for designing an objective function for a certain optimization problem, as we do in Section \ref{sec:improving-maclaurin} for designing a data-driven approach to feature construction.

\subsection{Complex-valued TensorSRHT}

\label{sec:complex-TensorSRHT}

We present here a generalization of TensorSRHT by allowing for complex features. To this end, let $z \in \mathbb{C}$ be a random variable such that (i) $| z | = 1$ almost surely, (ii) $\mathbb{E}[z] = 0$ and (iii) $z$ is symmetric, i.e., the distributions of $z$ and $-z$ are the same.
Define then $\mat{z}_1, \dots, \mat{z}_p \in \mathbb{C}^d$ as i.i.d.~complex random vectors such that elements of each random vector $\mat{z}_i$ are i.i.d.~realizations of $z$: %
\begin{equation} \label{eq:complex-vectors}
    \mat{z}_i = (z_{i,1}, \dots, z_{i, d})^\top \in \mathbb{C}^d, \quad z_{i,j} \stackrel{i.i.d.}{\sim}  P_z \quad (j=1,\dots,d),
\end{equation}
where $P_z$ denotes the probability distribution of $z$.

Let $\pi : \{1,\dots,d\} \to \{1,\dots,d\}$ be a random permutation of indices $1,\dots,d$. For $i=1,\dots,p$ and $\ell = 1,\dots,D$, we then define a random vector $\mat{s}_{i,\ell} \in \mathbb{C}^d$ as the Hadamard product of the random vector $\mat{z}_i$ in \eqref{eq:complex-vectors} and the permuted column $\mat{h}_{\pi(\ell)}$ of the Hadamard matrix $\mat{H}_d$:
\begin{equation} \label{eq:complex-weight-vector-SRHT}
    \mat{s}_{i, \ell} := \mat{z}_{i} \circ \mat{h}_{\pi(\ell)} = ( z_{i,1} h_{\pi(\ell), 1}, \dots, z_{i,d} h_{\pi(\ell), d} )^\top \in \mathbb{C}^d,  
\end{equation}

With these weight vectors $\mat{s}_{i, \ell}$, we define a random feature map exactly in the same way as the feature map in  \cref{eqn:real-struct-polynomial-sketch} for the real TensorSRHT in Section \ref{sec:real-tensorSRHT}. We define the resulting feature map $\Phi_{\mathcal{C}}: \mathbb{R}^d \to \mathbb{C}^D$  by
\begin{align}
    \label{eqn:complex-struct-polynomial-sketch}
    \Phi_\mathcal{C}(\mat{x})
    :=  \frac{1}{\sqrt{D}} \left[ (\prod_{i=1}^p \dotprodi{s}{x}{i,1}), \dots, (\prod_{i=1}^p \dotprodi{s}{x}{i,D}) \right]^{\top} \in \mathbb{C}^D.
\end{align}
We call this feature construction {\em complex TensorSRHT}.

Admissible examples of the distribution $P_z$ in \cref{eq:complex-vectors} include: (1) the uniform distribution on $\{1, -1\}$; (2)  the uniform distribution on $\{1, -1, \iu, - \iu\}$; and (3) the uniform distribution on the unit circle in $\mathbb{C}^d$.
Example (1) is where $z$ is a real Rademacher random variable, and in this case, the complex TensorSRHT coincides with the real TensorSRHT. Thus, the complex TensorSRHT is a strict generalization of the real TensorSRHT. 

We first show that the complex TensorSRHT provides an unbiased approximation of the polynomial kernel $k(\mat{x}, \mat{y}) = (\dotprod{x}{y})^p$.
Since the real TensorSRHT is a special case, its unbiasedness follows from this result. 

\begin{proposition}
\label{prop:unbiased-struct-approximate-kernel-complex}
Let $\mat{x}, \mat{y} \in \mathbb{R}^d$ be arbitrary, and  $\hat{k}_\mathcal{C}( \mat{x}, \mat{y} ) = \Phi_{\mathcal{C}} (\mat{x})^{\top} \overline{\Phi_{\mathcal{C}} (\mat{y})}$ be the approximate kernel with $\Phi_\mathcal{C}(\mat{x}),  \Phi_\mathcal{C}(\mat{y}) \in \mathbb{C}^D$ given by the random feature map in \cref{eqn:complex-struct-polynomial-sketch}.
Then we have
$ \mathbb{E}[ \hat{k}_\mathcal{C}( \mat{x}, \mat{y} ) ] = ( \mat{x}^\top \mat{y} )^p$.
\end{proposition}

\begin{proof}
We first show $\mathbb{E} [\mat{s}_{i, \ell} \overline{\mat{s}_{i, \ell}}^{\top}] = \mat{I}_d$ for all $i=1, \dots, p$ and $\ell = 1, \dots, D$.
This follows from the fact that, for all $t, u = 1,\dots,d$, we have
$$
\mathbb{E} [ ( \mat{s}_{i, \ell} \overline{\mat{s}_{i, \ell}}^{\top} )_{tu}] = \mathbb{E}[ z_{i, t} h_{\pi(\ell),t} \overline{z_{i, u}} h_{\pi(\ell),u}  ] = 
\begin{cases}
\mathbb{E}[ |z_{i, t}|^2] \mathbb{E}[ h_{\pi(\ell),t}^2  ] = 1 \quad (\text{if }\ t = u),\\
\mathbb{E}[ z_{i, t}] \mathbb{E}[\overline{z_{i, u}}] \mathbb{E}[  h_{\pi(\ell),t} h_{\pi(\ell),u}  ] = 0 \quad (\text{if }\ t \not= u),
\end{cases}.
$$
Using this, we have
\begin{align*} \mathbb{E}\left[\Phi_{\mathcal{C}}(\mat{x})^{\top} \overline{\Phi_{\mathcal{C}}(\mat{y})}\right]
    = \mathbb{E} \left[\frac{1}{D} \sum_{\ell=1}^D \prod_{i=1}^p \mat{x}^{\top} \mat{s}_{i,\ell} \overline{\mat{s}_{i,\ell}}^{\top} \mat{y} \right] = \frac{1}{D} \sum_{\ell=1}^D \prod_{i=1}^p \mat{x}^{\top} \mathbb{E} [\mat{s}_{i,\ell} \overline{\mat{s}_{i,\ell}}^{\top}] \mat{y}
    = (\dotprod{x}{y})^p
\end{align*}
\end{proof}

We now study the variance of the approximate kernel given by the complex TensorSRHT in \cref{eqn:complex-struct-polynomial-sketch}. 
To this end, we use the same notation as \cref{thm:expression-var-complex-aprox-kernel} to write the real and imaginary parts of the random variable $z$ as $z = a + \iu b$ with real-valued random variables $a, b \in \mathbb{R}$.
The proof of the following theorem is provided in \cref{sec:proof-var-comp-TensorSRHT}.
\begin{theorem}[Variance of Complex TensorSRHT] \label{thm:variance-complexTensorSRHT}
Let $\mat{x}, \mat{y} \in \mathbb{R}^d$ be arbitrary, and  $\hat{k}_\mathcal{C}( \mat{x}, \mat{y} ) = \Phi_{\mathcal{C}} (\mat{x})^{\top} \overline{\Phi_{\mathcal{C}} (\mat{y})}$ be the approximate kernel with $\Phi_\mathcal{C}(\mat{x}),  \Phi_\mathcal{C}(\mat{y}) \in \mathbb{C}^D$ given by the complex random feature map in  \cref{eqn:complex-struct-polynomial-sketch}. 
For the random variable $z$ defining \cref{eq:complex-vectors}, 
write $z = a + \iu b$ with $a, b \in \mathbb{R}$, and suppose that 
\begin{equation*}  
    \mathbb{E}[a b] = 0, \quad \mathbb{E}[a^2] = q, \quad \mathbb{E}[b^2] = 1-q \quad \text{where}\quad  0 \leq q \leq 1.
\end{equation*}
Then we have
\begin{equation}  \label{eqn:complex-tensor-srht-variance}
 \mathbb{V}[\hat{k}_\mathcal{C}(\mat{x}, \mat{y})] = \underbrace{\frac{V_q^{(p)}}{D}}_{(A)}     - \underbrace{\frac{c(D,d)}{D^2} \left[ (\mat{x}^\top \mat{y})^{2p}  -  \left( ( \mat{x}^\top \mat{y} )^2 - \frac{V_q^{(1)}}{d-1}   \right)^p \right] }_{(B)},
\end{equation}
where $V_q^{(p)} \geq 0$ and $V_q^{(1)} \geq 0$  are \cref{eq:complex-poly-var-coro-2}  with the considered value of $p$ and $p = 1$, respectively, and  $c(D,d) \in \mathbb{N}$ is defined in \eqref{eq:non-zero-pairs}.

\end{theorem}

Regarding  \cref{thm:variance-complexTensorSRHT}, we make the following observations. 
\begin{itemize}
\item The case $q = 1$ recovers \cref{lemma:variance-formula-struct} on the real TensorSRHT, where $z \in \{ 1, - 1\}$ is a Rademacher random variable.  
The case $q = 1/2$ is the complex TensorSRHT with, for instance, $P_z$ being  the uniform distribution on $\{1,-1, \iu, -\iu\}$ or on the unit circle in $\mathbb{C}$.
Other values of $q \in [0,1]$ can also be considered, but we do not discuss them further. 
\item The first term $(A)$ in \cref{eqn:complex-tensor-srht-variance}  is the variance of the unstructured polynomial sketch in \cref{eqn:comp-rad-polynomial-sketch} with $D$ features, since $V_q^{(p)}$ is its variance with a single feature ($D = 1$) in \cref{eq:complex-poly-var-coro-2}.
The second term $(B)$ in \cref{eqn:complex-tensor-srht-variance} is the effect of using the structured sketch.  The quantity $V_q^{(1)}$ is the variance of the unstructured sketch in \cref{eqn:comp-rad-polynomial-sketch}  with a single feature in  \cref{eq:complex-poly-var-coro-2} with $p = 1$.
\item As for the real case, the variance \eqref{eqn:complex-tensor-srht-variance} becomes zero when $p=1$ and $D \in \{kd \mid k \in \mathbb{N}\}$.
\end{itemize}

As we discussed for the real TensorSRHT in Corollary \ref{thrm:tensor-srht-odd-p},    \cref{thm:variance-complexTensorSRHT} enables understanding a condition under which the complex TensorSRHT is advantageous over the corresponding unstructured complex sketch in \cref{eqn:comp-rad-polynomial-sketch}.
As for the real case, the condition is that the degree $p$ of the polynomial kernel is {\em odd}, as stated in the following.

\begin{corollary} \label{coro:complex-tensor-srht-odd-p}
Let $p \in \mathbb{N}$ be odd.
Then,  for all input vectors $\mat{x}, \mat{y} \in \mathbb{R}^d$, the variance of the approximate kernel with the complex TensorSRHT in \cref{eqn:complex-tensor-srht-variance} is smaller or equal to the variance of the approximate kernel with the corresponding unstructured polynomial sketch: 
$$
\frac{V_q^{(p)}}{D}   - \frac{c(D,d)}{D^2}  \left[ (\mat{x}^\top \mat{y})^{2p}  -  \left( ( \mat{x}^\top \mat{y} )^2 - \frac{V_q^{(1)}}{d-1}   \right)^p \right] \leq \frac{V_q^{(p)}}{D}  
$$
\end{corollary}
\begin{proof}
Since, $V_q^{(1)} \geq 0$, we have  $( \mat{x}^\top \mat{y} )^2 - \frac{1}{d-1} V_q^{(1)} \leq ( \mat{x}^\top \mat{y} )^2$. For odd $p$ this leads to $\left( ( \mat{x}^\top \mat{y} )^2 - \frac{1}{d-1} V_q^{(1)} \right)^p \leq (\mat{x}^\top \mat{y})^{2p}$. The assertion immediately follows.
\end{proof}

As discussed for the real case, 
if $p$ is even, the variance of the complex TensorSRHT can be larger than the corresponding unstructured sketch for certain input vectors $\mat{x}, \mat{y} \in \mathbb{R}^d$ (e.g., when $\mat{x}^\top \mat{y} = 0$). 
Empirically, however, the complex TensorSRHT often has a smaller variance also for even $p$, as we demonstrate later.

\subsection{Comparing the Real and Complex TensorSRHT}

\label{sec:compare-tensorSRHT}

Let us now compare the real and complex TensorSRHT. 
To make the discussion clearer, suppose that the number of random features satisfies $D = B d$ for some $B \in \mathbb{N}$, as in Remark \ref{remark:var-real-tensorSRHT}.
Then the variance formula in
\cref{eqn:complex-tensor-srht-variance} simplifies to 
\begin{equation}  \label{eq:var-formula-comp-TensorSRHT-simple} 
 \mathbb{V}[\hat{k}_\mathcal{C}(\mat{x}, \mat{y})] = \underbrace{\frac{V_q^{(p)}}{D}}_{(A)}     - \underbrace{\frac{d-1}{D} \left[ (\mat{x}^\top \mat{y})^{2p}  -  \left( ( \mat{x}^\top \mat{y} )^2 - \frac{V_q^{(1)}}{d-1}   \right)^p \right] }_{(B)}. 
\end{equation}
Recall that setting $q = 1$ and $q = 1/2$  recover the variances of real and complex TensorSRHT, respectively. 
Thus, let us compare these two cases. 
We make the following observations:
\begin{itemize}
\item As discussed in Section \ref{sec:var-comp-poly}, it holds that $V_{1/2}^{(p)} \leq V_{1}^{(p)}$ and $V_{1/2}^{(1)} \leq V_{1}^{(1)}$  given that the input vectors $\mat{x} = (x_1,\dots,x_d), \mat{y} = (y_1,\dots,y_d)$ satisfy the inequality in \cref{eq:cond-for-complex-to-be-better}, i.e., 
$\sum_{i \not= j}  x_i x_j y_i y_j \geq  0$, which is satisfied when $\mat{x}$ and $\mat{y}$ are non-negative vectors.  
\item Thus, if \cref{eq:cond-for-complex-to-be-better} is satisfied, the first term $(A)$ becomes smaller for $q = 1/2$ (complex case) than $q  = 1$ (real case).  On the other hand, if $p$ is odd,  the second term $(B)$ becomes smaller for $q = 1/2$ than $q = 1$; thus, the variance reduction (i.e., $-(B)$) is smaller for $q = 1/2$ than $q = 1$.
\end{itemize}

The above observations suggest that, even when \cref{eq:cond-for-complex-to-be-better} is satisfied, whether the complex TensorSRHT $(q = 1/2)$ has a smaller variance than the real TensorSRHT ($q = 1$) depends on the balance between the two terms $(A)$ and $(B)$ and on the properties of the input vectors $\mat{x}, \mat{y} \in \mathbb{R}^d$.   
We have not been able to provide a theoretical characterization of exact situations where the complex TensorSRHT has a smaller variance than the real TensorSRHT.

To complement the lack of a theoretical characterization, we performed experiments to compare the variances of real and complex TensorSRHT, whose results are shown in
\cref{fig:var-real-complex-TensorSRHT}.
We evaluated the variance formula in \cref{eq:var-formula-comp-TensorSRHT-simple} for $q = 1$ (real) and  $q = 1/2$ (complex), for  1000 pairs of input vectors $\mat{x}, \mat{y}$ randomly sampled from a given dataset (EEG, CIFAR 10 ResNet34 features, MNIST and Gisette). 
For each pair $\mat{x}, \mat{y}$, we computed the ratio of \cref{eq:var-formula-comp-TensorSRHT-simple} with $q=1/2$ divided by  \cref{eq:var-formula-comp-TensorSRHT-simple} with $q = 1$, and \cref{fig:var-real-complex-TensorSRHT} shows the empirical cumulative distribution function of this ratio for the 4 datasets.  
In these datasets, the input vectors are nonnegative.

 \cref{fig:var-real-complex-TensorSRHT}  shows that, for  100\%, 100\%, 97.8\%, and 100\% of the cases of the 4 datasets, respectively, the variance of the complex TensorSRHT is smaller than that of the real TensorSRHT. 
 Moreover, the ratio of the variances tends to be even smaller for a larger value of $p$.
 These results suggest that the complex TensorSRHT is effective in reducing the variance of the real TensorSRHT, and the variance reduction is more significant for a larger value $p$ of the polynomial degree. 
We leave a theoretical analysis for explaining this improvement of the complex TensorSRHT for future work.

\begin{figure}[t]
\centering
\includegraphics[width=1.0\textwidth]{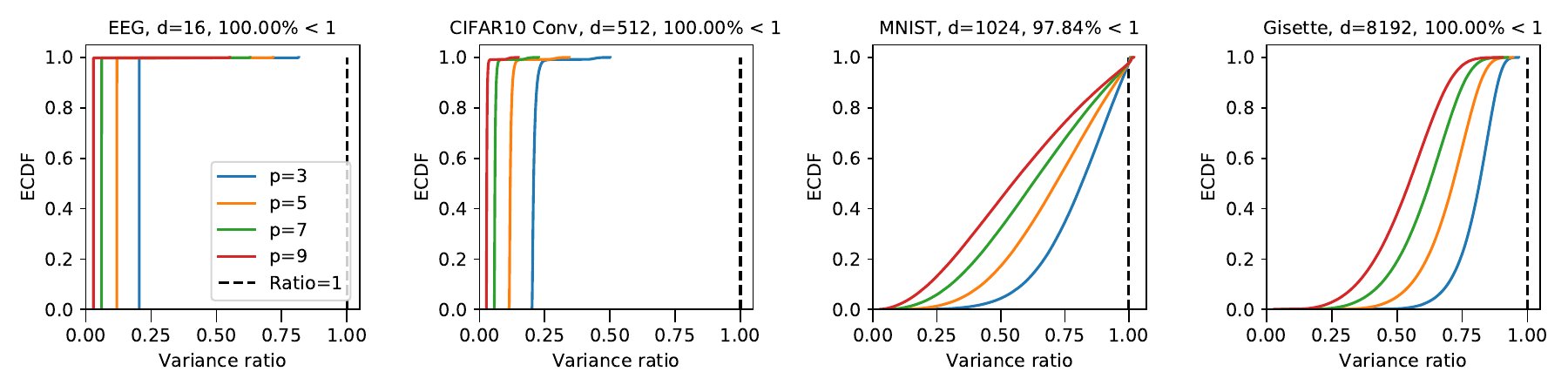}
\caption{Empirical cumulative distribution of pairwise ratios Var(Compl.\ TensorSRHT) / Var(Real TensorSRHT) on a subsample (1000 samples) of four different datasets (EEG, CIFAR10 ResNet34 features, MNIST, Gisette) with unit-normalized data where $D=d$. The datasets are not zero-centered and therefore entirely positive. }
\label{fig:var-real-complex-TensorSRHT}
\end{figure}

\section{Approximating Dot Product Kernels}

\label{sec:approx-dot-prod-kernels}
We discuss here how polynomial sketches described so far can be used for approximating more general {\em dot product kernels}, i.e., kernels whose values depend only on the inner product of input vectors. 

In Sections \ref{sec:dot-product-kernels} and \ref{sec:random-maclaurin}, we first review a key result on the  Maclaurin expansion of dot product kernels and the resulting random sketch approach by \citet{Kar2012}, and show how the polynomial sketches described so far can be used.
In Section \ref{sec:improving-maclaurin}, we then introduce a data-driven optimization approach to improving the random sketches based on the Maclaurin expansion.
In Section \ref{sec:approx-Gauss-kernel}, we describe how to apply this approach for approximating the  Gaussian kernel.
In Section \ref{sec:illustration-objective-optim}, we provide a numerical illustration of the optimization objective.

\subsection{Maclaurin Expansion of Dot Product Kernels}
\label{sec:dot-product-kernels}

Let $\mathcal{X} \subset \mathbb{R}^d$ be a subset, and let $k: \mathcal{X} \times \mathcal{X}: \to \mathbb{R}$ be a positive definite kernel on $\mathcal{X}$. 
The kernel $k$ is called {\em dot product kernel}, if there exists a function $f: \mathbb{R} \to \mathbb{R}$ such that 
\begin{equation} \label{eq:dot-prod-kernel}
k(\mat{x}, \mat{y}) = f(\dotprod{x}{y}) \quad \text{for all } \mat{x}, \mat{y} \in \mathcal{X}.
\end{equation}
Examples of dot product kernels include polynomial kernels $k(\mat{x}, \mat{y}) = ( \mat{x}^\top \mat{y} + \nu )^p$ with $\nu \geq 0$ and $p \in \mathbb{N}$, which have been our focus in this paper, and exponential kernels $k(\mat{x}, \mat{y}) = \exp(\mat{x}^\top \mat{y} / l^2)$ with $l > 0$.
Other examples of dot product kernels can be found in, e.g., \citet{Smola2001}.

We focus on dot product kernels for which the function $f$ in \cref{eq:dot-prod-kernel} is an analytic function whose Maclaurin expansion has non-negative coefficients: $f(x)=\sum_{n=0}^{\infty} a_n x^n$ and $a_n \geq 0$ for $n \in \{ 0 \} \cup \mathbb{N}$.
In other words, we consider dot product kernels that can be expanded as 
\begin{equation} \label{eq:dot-prod-kernel-expand}
k(\mat{x}, \mat{y}) = \sum_{n = 0}^\infty a_n (\mat{x}^\top \mat{y})^n \quad   \text{for all } \mat{x}, \mat{y} \in \mathcal{X},
\end{equation}
with $a_n \geq 0$ for all $n \in \{ 0 \} \cup \mathbb{N}$.

Many dot product kernels can be expanded as \cref{eq:dot-prod-kernel-expand}.
In fact, \citet[Theorem 1]{Kar2012} show that, if $\mathcal{X}$ is the unit ball of $\mathbb{R}^d$, the function $k$ of the form of \cref{eq:dot-prod-kernel} is positive definite on $\mathcal{X}$ if and only if it can be written as \cref{eq:dot-prod-kernel-expand}.

We show here a few concrete examples.
The polynomial kernel $ k(\mat{x}, \mat{y}) =   ( \mat{x}^\top \mat{y} + \nu  )^p$ 
with $p \in \mathbb{N}$ and $\nu \geq 0$ can be expanded as
\begin{equation} \label{eq:poly-ker-expansion}
    ( \mat{x}^\top \mat{y} + \nu  )^p
    = \sum_{n=0}^p \binom{p}{n} \nu^{p-n} (\dotprod{x}{y})^n,
\end{equation}
and thus $a_n = \binom{p}{n} \nu^{p-n} \geq 0$ for $n \in \{0, \dots, p\}$ and $a_n=0$ for $n>p$ in  \cref{eq:dot-prod-kernel-expand}.
The exponential kernel $k(\mat{x}, \mat{y}) = \exp( \mat{x}^\top \mat{y} / l^2)$ can be expanded as
\begin{equation}
    \exp\left( \frac{\mat{x}^\top \mat{y} }{l^2} \right) 
    = \sum_{n=0}^\infty \frac{1}{n! l^{2n}} (\dotprod{x}{y})^n
    \label{eqn:exponential-dot-prod-kernel}
\end{equation}
and thus $a_n = 1/ (n! l^{2n})$ for $n \in \mathbb{N}$ in  \cref{eq:dot-prod-kernel-expand}.

\paragraph{Gaussian kernel as a weighted dot product kernel}
The  Gaussian kernel  defined as $k(\mat{x}, \mat{y}) = \exp(-\|\mat{x}-\mat{y}\|^2 / (2 l^2))$ with $l > 0$ can be written as a {\em weighted} exponential kernel:
\begin{align}
\exp\left(- \frac{ \|\mat{x}-\mat{y}\|^2} {2 l^2} \right) 
& =  \exp\left(- \frac{\|\mat{x}\|^2} {2 l^2} \right) \exp\left(-\frac{\|\mat{y}\|^2}{2 l^2} \right) \exp\left( \frac{\mat{x}^{\top} \mat{y}} {l^2} \right) \nonumber  \\
& =  \exp\left(- \frac{\|\mat{x}\|^2} {2 l^2} \right) \exp\left(-\frac{\|\mat{y}\|^2}{2 l^2} \right) 
    \sum_{n=0}^\infty \frac{1}{n! l^{2n}} (\dotprod{x}{y} )^n, \label{eq:Gaussian-kernel-expansion}
\end{align}
where the second identity uses the Maclaurin expansion of the exponential kernel in \cref{eqn:exponential-dot-prod-kernel}. For approximating the Gaussian kernel, \citet{Cotter2011} proposed a finite dimensional feature map based on a truncation of this expansion.

\subsection{Random Sketch based on the Maclaurin Expansion}
\label{sec:random-maclaurin}

We describe here the approach of \citet{Kar2012} on the unbiased 
approximation of dot product kernels based on the Maclaurin expansion in \cref{eq:dot-prod-kernel-expand}.
We discuss this approach to provide a basis and motivation for our new approach for approximating dot product kernels.

First, we define a probability measure $\mu$ on $\{ 0 \} \cup \mathbb{N}$. \citet{Kar2012} propose to define $\mu$ as 
\begin{equation} \label{eq:sampling-dist-rand-Mac}
\mu(n) \propto c^{-(n+1)}, \quad n \in \{0\} \cup \mathbb{N}, 
\end{equation}
for a constant $c > 1$ (e.g., $c = 2$). 
Using this probability measure and the Rademacher sketch, \citet{Kar2012} propose a doubly stochastic approximation of the dot product kernel in \cref{eq:dot-prod-kernel-expand}. %
This approach first generates an i.i.d.~sample of size  $D \in \mathbb{N}$ from this probability measure $\mu$
\begin{equation} \label{eq:iid-numbers-maclaurin}
n_1, \dots, n_D \stackrel{i.i.d.}{\sim} \mu
\end{equation}
and defines $D_n$ for $n \in \{0\} \cup \mathbb{N}$ as the number of times $n$ appears in $n_1, \dots, n_D$; thus $\sum_{n = 0}^\infty D_n = D$.

Then, for each $n \in \{0\} \cup \mathbb{N}$ with $D_n > 0$, construct a random feature map $\Phi_n: \mathcal{X} \to \mathbb{R}^{D_n}$ with $D_n$ features of the form in \cref{eqn:polynomial-estimator} that provide an unbiased approximation of the polynomial kernel $k_n(\mat{x}, \mat{y}) := (\mat{x}^\top \mat{y})^n$ of degree $n$:
\begin{equation} \label{eq:macularin-poly-approx}
\mathbb{E}[  \Phi_n(\mat{x})^\top \Phi_n(\mat{y}) ]  = (\mat{x}^\top \mat{y})^n.
\end{equation}
The original formulation of \citet{Kar2012} uses the Rademacher sketch as $\Phi_n$, but one can use other sketches in Sections \ref{sec:polynomial-sketches} and \ref{sec:structured-projections}, such as the Gaussian sketch and TensorSRHT. 

Finally, defining a random variable $n^* \sim \mu$, the dot product kernel in \cref{eq:dot-prod-kernel-expand} is rewritten and approximated as 
\begin{align}
 k(\mat{x}, \mat{y}) &= \sum_{n = 0}^\infty a_n (\mat{x}^\top \mat{y})^n  =  \sum_{n=0}^\infty  \frac{a_n}{\mu(n)} \mu(n) (\mat{x}^\top \mat{y})^n  =\mathbb{E}_{n^* \sim \mu} \left[ \frac{a_{n^*}}{\mu(n^*)}  (\mat{x}^\top \mat{y})^{n^*} \right] \nonumber \\
&   \approx \frac{1}{D} \sum_{n \in \{ n_1, \dots, n_D \}} D_n  \frac{a_n}{\mu(n)}  (\mat{x}^\top \mat{y})^{n} =  \frac{1}{D} \sum_{n: D_n > 0} D_n \frac{a_n}{\mu(n)} (\mat{x}^\top \mat{y})^n \nonumber \\
& \approx \frac{1}{D} \sum_{n: D_n > 0} D_n \frac{a_n}{\mu(n)}  \Phi_n(\mat{x})^\top \Phi_n(\mat{y}),  \label{eq:Maculaurin-approx}
\end{align}
where the first approximation is the Monte Carlo approximation of the expectation $\mathbb{E}_{n^* \sim \mu}$ using the i.i.d.~sample in \cref{eq:iid-numbers-maclaurin} and the second approximation is using the random feature map in  \cref{eq:macularin-poly-approx}.
The approximation in \cref{eq:Maculaurin-approx} is unbiased, since the two approximations are statistically independent and both are unbiased.

The first approximation for \cref{eq:Maculaurin-approx} can be interpreted as first selecting polynomial degrees $n \in \{0\} \cup \mathbb{N}$ and assigning the number of features $D_n$ to each selected degree, given a budget constraint $D = \sum_{n: D_n >0} D_n$. 
While performing these assignments by random sampling as in \cref{eq:iid-numbers-maclaurin} makes the approximation in \cref{eq:Maculaurin-approx} unbiased, the resulting variance of \cref{eq:Maculaurin-approx} can be large. In the next subsection, we introduce a data-driven optimization approach to this feature assignment problem, to achieve a good balance between the bias and variance.

\subsection{Optimization for a Truncated Maclaurin Approximation  }
\label{sec:improving-maclaurin}

We develop here an optimization algorithm for selecting the polynomial degrees $n$ and assigning the number of random features to each selected polynomial degree in the Maclaurin sketch in \cref{eq:Maculaurin-approx} . 
The objective function is an estimate of the expected bias and variance of the resulting approximate kernel, and we define it using the variance formulas derived in Sections \ref{sec:polynomial-sketches} and \ref{sec:structured-projections}.

We consider a biased approximation obtained by truncating the Maclaurin expansion in \cref{eq:dot-prod-kernel-expand} up to the $p$-th degree polynomials, where $p$ is to be determined by optimization.
Let $D_{\rm total} \in \mathbb{N}$ be the total number of random features, which is specified by a user. 
For each $n = 1,\dots,p$, let $D_n \in \{ 0 \} \cup \mathbb{N}$ be the number of random features for approximating the $n$-th term $(\mat{x}^\top \mat{y})^n$ of the Maclaurin expansion in \cref{eq:dot-prod-kernel-expand}, such that $\sum_{n=1}^p D_n = D_{\rm total}$.
The numbers $D_n$ are to be determined by optimization.
Let $\Phi_n : \mathbb{R}^d \rightarrow \mathbb{C}^{D_n}$ be a (possibly complex) random feature map defined in Sections \ref{sec:polynomial-sketches} and \ref{sec:structured-projections} such that $\mathbb{E}[\Phi_n(\mat{x})^{\top} \overline{\Phi_n(\mat{y})}] = (\dotprod{x}{y})^n$ for all $\mat{x}, \mat{y} \subset \mathcal{X} \subset \mathbb{R}^d$.  
Note that $\Phi_n$ can be a real-valued feature map, but we use the notation for the complex case since it subsumes the real case.

We then define an approximation to the dot product kernel in  \cref{eq:dot-prod-kernel-expand} as
\begin{equation}
    \label{eqn:approximate-kernel}
    \hat{k}(\mat{x}, \mat{y}) := a_0 + \sum_{n=1}^p a_n \Phi_n(\mat{x})^{\top} \overline{\Phi_n(\mat{y})}, \quad \mat{x}, \mat{y} \in \mathcal{X}
\end{equation}
This approximation is biased, since it ignores the polynomial terms whose degrees are higher than $p$ in the expansion of \cref{eq:dot-prod-kernel-expand}.  
One can reduce this bias by increasing $p$, but this may lead to a higher variance.   
Therefore, there is a bias-variance trade-off in the choice of $p$. 
We describe below how to choose $p$ and the number of features $D_n$ of each random feature map $\Phi_n(\mat{x}), \Phi_n( \mat{y} ) \in \mathbb{C}^{D_n}$ for $n = 1,\dots, p$.

\subsubsection{Optimization Objective}

For a given learning task, we are usually provided data points generated from an unknown probability distribution $P(\mat{x})$ on the input domain $\mathcal{X} \subset \mathbb{R}^d$. The approximate kernel $\hat{k}(\mat{x}, \mat{y})$ in \cref{eqn:approximate-kernel} should be an accurate approximation of the target kernel $k(\mat{x}, \mat{y})$ for input vectors $\mat{x}, \mat{y}$ drawn from this unknown data distribution $P(\mat{x})$.
Therefore, we consider the following {\em integrated mean squared error} as our objective function: 
\begin{align}
&   \int \int \mathbb{E} \left[ \left( k(\mat{x}, \mat{y}) - \hat{k}(\mat{x}, \mat{y}) \right)^2  \right] dP(\mat{x})dP(\mat{y}) \label{eqn:initial-objective} \\
 & =   \int \int \underbrace{ \mathbb{V}[\hat{k}(\mat{x}, \mat{y})] }_{\rm variance} dP(\mat{x})dP(\mat{y})   + \int \int \underbrace{ \left( k(\mat{x}, \mat{y}) - \mathbb{E} \left[ \hat{k}(\mat{x}, \mat{y}) \right] \right)^2  }_{{\rm bias}^2} dP(\mat{x})dP(\mat{y})  \label{eqn:bias-variance-decomposition}
\end{align}
where the expectation $\mathbb{E}[\cdot]$ and variance $\mathbb{V}[\cdot]$ are taken with respect to the random feature maps in the approximate kernel in \cref{eqn:approximate-kernel}, and the identity follows from the standard bias-variance decomposition. 

We study the variance and bias terms in \cref{eqn:bias-variance-decomposition}.
Let $\delta[D_n > 0]$ be an indicator such that $\delta[D_n > 0] = 1$  if $D_n > 0$ and  $\delta[D_n > 0] = 0$ otherwise.
Using this indicator, and since the $p$ random feature maps $\Phi_1, \dots, \Phi_p$ in   \cref{eqn:approximate-kernel} are statistically independent,  the variance term in \cref{eqn:bias-variance-decomposition}  can be written as 
\begin{align}
    \mathbb{V} \left[\hat{k} (\mat{x}, \mat{y}) \right]
    = \sum_{n=1}^{p} \delta[D_n > 0] \ a_n^2  \mathbb{V} \left[ \Phi_n(\mat{x})^{\top} \overline{\Phi_n(\mat{y})} \right].
    \label{eqn:maclaurin-variance}
\end{align}
Each individual term $\mathbb{V} [\Phi_n(\mat{x})^{\top} \overline{ \Phi_n(\mat{y})}]$ in \cref{eqn:maclaurin-variance} is the variance of the approximate kernel $\hat{k}_n(\mat{x}, \mat{y}) := \Phi_n(\mat{x})^{\top} \overline{ \Phi_n(\mat{y})}$ for approximating the polynomial kernel $k_n(\mat{x}, \mat{y}) := (\mat{x}^\top \mat{y})^n$ of degree $n = 1, \dots, p$. 
Therefore,  one can explicitly compute $\mathbb{V} [\Phi_n(\mat{x})^{\top} \overline{ \Phi_n(\mat{y})}]$ for any given $\mat{x}, \mat{y} \in \mathbb{R}^d$ using the variance formulas derived  in Sections \ref{sec:polynomial-sketches} and \ref{sec:structured-projections}. For the convenience of the reader, we summarize the variance formulas for specific cases in \cref{tbl:variances}. 
Regarding the bias term in \cref{eqn:bias-variance-decomposition},  the expectation of the  approximate kernel  \eqref{eqn:approximate-kernel}  is given by
\begin{align}
    \mathbb{E} \left[\hat{k} (\mat{x}, \mat{y}) \right]
    = \sum_{n=0}^{p} \delta[D_n > 0] \ a_n (\mat{x}^\top \mat{y})^n,
    \label{eqn:maclaurin-expectation}
\end{align}
since $\mathbb{E}  \left[ \Phi_n(\mat{x})^{\top} \overline{ \Phi_n(\mat{y}) } \right] = (\mat{x}^\top \mat{y})^n$ for  $n = 1,\dots, p$ with $D_n > 0$.

\begin{table}[t]
\begin{center}
\resizebox{0.9 \textwidth}{!}{
\begin{tabular}{l | l}
\toprule
Sketch & Variance \\
\midrule
Real Gaussian & \multirow{2}{*}{$D^{-1} \Big[ \Big( \| \mat{x} \|^2 \| \mat{y} \|^2 + 2 ( \mat{x}^\top \mat{y} )^2 \Big)^n - (\mat{x}^\top \mat{y})^{2n} \Big]$} \\
  & \\
  Complex Gaussian & \multirow{2}{*}{$ D^{-1}\Big[ \Big( \| \mat{x} \|^2 \| \mat{y} \|^2 + ( \mat{x}^\top \mat{y} )^2 \Big)^n - (\mat{x}^\top \mat{y})^{2n} \Big]$} \\
  & \\
  \midrule
  Real Rademacher & \multirow{2}{*}{$D^{-1}\Big[ \Big(  \|\mat{x}\|^2 \|\mat{y}\|^2 + 2 \Big( ( \mat{x}^\top \mat{y} )^2 - \sum_{k=1}^d x_k^2 y_k^2 \Big) \Big)^n - (\mat{x}^\top \mat{y})^{2n} \Big]$} \\
  & \\
Complex  Rademacher & \multirow{2}{*}{$D^{-1} \Big[ \Big(  \|\mat{x}\|^2 \|\mat{y}\|^2 + ( \mat{x}^\top \mat{y} )^2 - \sum_{k=1}^d x_k^2 y_k^2 \Big)^n - (\mat{x}^\top \mat{y})^{2n} \Big]$} \\
  & \\
 \midrule
Real TensorSRHT &  $\text{Real Rademacher Variance }  $ \\
  & $- \frac{c(D,d)}{D^2} \Big[ (\mat{x}^\top \mat{y})^{2n} - \Big( ( \mat{x}^\top \mat{y} )^2 - \frac{1}{d-1} \Big( \|\mat{x}\|^2 \|\mat{y}\|^2 + ( \mat{x}^\top \mat{y} )^2 - 2 \sum_{k=1}^d x_k^2 y_k^2 \Big) \Big)^n \Big]$ \\
 Complex TensorSRHT &
 $\text{Complex Rademacher Variance } $ \\
  & $ - \frac{c(D,d)}{D^2} \Big[ (\mat{x}^\top \mat{y})^{2n} - \Big( ( \mat{x}^\top \mat{y} )^2 - \frac{1}{d-1} \Big( \|\mat{x}\|^2 \|\mat{y}\|^2 - \sum_{k=1}^d x_k^2 y_k^2 \Big) \Big)^n \Big]$ \\
 \midrule
Conv. Sur. TensorSRHT &
$\left\{\begin{array}{ll} D^{-1} \Big(V_q^{(n)} + (d-1) {\rm Cov}_q^{(n)}\Big) & \mbox{if } {\rm Cov}_q^{(n)} > 0 \text{ or } D > d, \\ D^{-1} \Big( V_q^{(n)} - {\rm Cov}_q^{(n)} \Big) + {\rm Cov}_q^{(n)} & \mbox{otherwise.} \end{array} \right.$ \\
(Real case: $q=1$) & $V_q^{(n)} = \Big(\norm{x}^2 \norm{y}^2  + ((2q - 1)^2 + 1) ( (\dotprod{x}{y})^2 -  \textstyle\sum_{k=1}^d x_k^2 y_k^2)\Big)^n - (\mat{x}^\top \mat{y})^{2n}$ \\
(Complex case: $q=1/2$) & ${\rm Cov}_q^{(n)} = \Big(( \mat{x}^\top \mat{y} )^2 - \frac{V_q^{(1)}}{d-1} \Big)^n - (\mat{x}^\top \mat{y})^{2n}$ \\
 \bottomrule
\end{tabular}
}
\end{center}
\caption{
Closed-form expressions for the variance $\mathbb{V} \big[ \Phi_n(\mat{x})^{\top} \overline{ \Phi_n(\mat{y})}\big]$ for different random feature maps $\Phi_n : \mathbb{R}^d \to \mathbb{C}^D$ to approximate  polynomial kernel of order $n \in \mathbb{N}$. 
Here, $D \in \mathbb{N}$ is the number of random features
and $c(D, d) := \lfloor D/d \rfloor d(d-1) + (D \mod d)(D \mod d - 1)$. See Sections \ref{sec:polynomial-sketches} and \ref{sec:structured-projections} for details and more generic results. 
We also show convex surrogate functions in  \cref{eqn:var-convexified-2} and  \cref{eqn:var-convexified-positive-cov}  for the variance of TensorSRHT derived in \cref{sec:incremental-tensorsrht}. 
}
\label{tbl:variances}
\end{table}

Note that the integrals in \cref{eqn:bias-variance-decomposition} with respect to $P$ are not available in practice, as $P$ is the unknown data distribution.
We instead assume that an i.i.d.~sample $\mat{x}_1, \dots, \mat{x}_m$ of size $m \in \mathbb{N}$ from $P$ is available.  
This sample may be a subsample of a larger dataset from $P$.
For example, in a supervised learning problem, $\mat{x}_1, \dots, \mat{x}_m$  may be a random subsample of training input points. 

Using the i.i.d.~sample $\mat{x}_1, \dots, \mat{x}_m$, the objective function in \cref{eqn:bias-variance-decomposition} can then be unbiasedly approximated in a U-statistics form as
\begin{align}
& \frac{1}{m(m-1)} \sum_{i \not= j} 
    \mathbb{V} [\hat{k}(\mat{x}_i, \mat{x}_j)] + \frac{1}{m(m-1)}  \sum_{i \not= j} 
     \left( k(\mat{x}_i, \mat{x}_j) - \mathbb{E}[\hat{k}(\mat{x}_i, \mat{x}_j )] \right)^2   \nonumber \\
     & =  \frac{1}{m(m-1)}  \sum_{n=1}^{p} \delta[D_n > 0]~ a_n^2 \sum_{i \not= j}   \mathbb{V} \left[ \Phi_n(\mat{x}_i)^{\top} \overline{\Phi_n(\mat{x}_j)} \right]  \label{eq:maclaurin-var-estimate} \\
     & \quad + \frac{1}{m(m-1)} \sum_{i \not= j}    \left( k(\mat{x}_i, \mat{x}_j) -  \sum_{n=0}^{p} \delta[D_n > 0] \ a_n (\mat{x}_i^\top \mat{x}_j)^n \right)^2,   \label{eq:maclaurin-bias-estimate}  \\
     & =: g(p, (D_n)_{n = 1}^p )  \label{eq:maclaurin-empirical-objective}
\end{align}
where we used \cref{eqn:maclaurin-variance} and \cref{eqn:maclaurin-expectation}.

Finally, we formulate our optimization problem. 
To make the problem tractable, we search for the degree $p$ of the approximate kernel in \cref{eqn:approximate-kernel} from the range $\{p^*_{\rm min}, p^*_{\rm min} + 1, \dots, p^*_{\rm max}  \}$, where $p^*_{\rm min}, p^*_{\rm max} \in \mathbb{N}$ with $p^*_{\rm min} < p^*_{\rm max}$ are lower and upper bounds of $p$ selected by the user. 
We then define our optimization problem as follows:
\begin{align} \label{eq:full-optimization-problem}
& \min_{p ,  (D_n)_{n=1}^p }  g(p, (D_n)_{n = 1}^p )   \quad  \text{subject to }\quad  p  \in \{p^*_{\rm min}, p^*_{\rm min} + 1, \dots, p^*_{\rm max} \}, \\  
& D_n \in \{0, \dots D_{\rm total}\},  \quad \sum_{n=1}^p D_n = D_{\rm total},  \quad D_n \geq 1 \text{ if and only if } a_n > 0 \quad (n=1,\dots,p).  \nonumber
\end{align}
where $g(p, (D_n)_{n=1}^p)$ is defined in \cref{eq:maclaurin-empirical-objective}. %

To present our approach to solving \cref{eq:full-optimization-problem}, we will first define a simplified optimization problem and describe an algorithm for solving it.
We will then use this simplified problem and its solver to develop a solver for the full problem in \cref{eq:full-optimization-problem}.

\subsubsection{Solving a Simplified Problem}

We consider a simplified problem of \cref{eq:full-optimization-problem} in which the polynomial degree $p \in \mathbb{N}$ is fixed and given, and the number of random features $D_n$  is positive, $D_n \geq 1$, for every polynomial degree $n = 1,\dots,p$ with $a_n > 0$.
Note that the bias term of the objective function $g(p, (D_n)_{n=1}^n)$, i.e.~
\cref{eq:maclaurin-bias-estimate}, only depends on $(D_n)_{n=1}^n$ through the indicator function $\delta[D_n > 0]$. 
Therefore, under the constraint that $D_n \geq 1$ for all $n = 1,\dots, p$ with $a_n > 0$, \cref{eq:maclaurin-bias-estimate} becomes constant with respect to $(D_n)_{n=1}^p$.

Thus, the optimization problem
 \cref{eq:full-optimization-problem} under the additional constraint of $p$ being fixed and $D_n \geq 1$ for all $n=1,\dots,p$ with $a_n > 0$ is equivalent to the following optimization problem:
\begin{align}     \label{eqn:optimization-objective}
& \min_{ (D_n)_{n=1}^p} \frac{1}{m(m-1)}    \sum_{n=1}^{p} \ a_n^2  \sum_{i \not= j}   \mathbb{V} \left[ \Phi_n(\mat{x}_i)^{\top} \overline{\Phi_n(\mat{x}_j)} \right] \quad \text{subject to}   \\ 
 &   D_n \subset \{0, \dots D_{\rm total}\},\quad    \sum_{n = 1}^p D_n = D_{\rm total}, \quad D_n \geq 1 \text{ if and only if } a_n > 0 \quad (n = 1,\dots,p).\nonumber
\end{align}
This is a discrete optimization problem with one equality constraint, and is an instance of the so-called  \textit{Resource Allocation Problem} \citep{FloudasChristodoulosA.PardalosPanosM.2009}.

We discuss properties of the objective function in \cref{eqn:optimization-objective} and describe a solver.
To this end, we first consider the case where $\Phi_n: \mathbb{R}^d \to \mathbb{C}^{D_n}$ is one of the unstructured polynomial sketches in Section \ref{sec:polynomial-sketches}; we will later explain its extension to structured sketches from Section \ref{sec:structured-projections}. 
In this case, we have $ \mathbb{V} \left[ \Phi_n(\mat{x})^{\top} \overline{\Phi_n(\mat{y})} \right] =  C_{\mat{x},\mat{y}}^{(n)} /D_n$ for a constant $C_{\mat{x},\mat{y}}^{(n)}$ depending on $\mat{x}, \mat{y} \in \mathbb{R}^d$ and the polynomial degree $n \in \mathbb{N}$ but not on $D_n$, as summarized in \cref{tbl:variances}. 
Therefore,  
\begin{equation} \label{eq:simple-objective-individual-term}
a_n^2  \sum_{i \not= j}   \mathbb{V} \left[ \Phi_n(\mat{x}_i)^{\top} \overline{ \Phi_n(\mat{x}_j) } \right] = \frac{a_n^2}{D_n} \sum_{i \not= j} C_{\mat{x}_i, \mat{x}_j}^{(n)}
\end{equation}
is convex and monotonically decreasing with respect to $D_n$.
From this property, one can use 
the \textit{Incremental Algorithm} \citep[p. 384]{FloudasChristodoulosA.PardalosPanosM.2009} to directly solve the optimization problem \eqref{eqn:optimization-objective}.

\cref{alg:incremental-algorithm} describes the Incremental Algorithm for solving the simplified problem in \cref{eqn:optimization-objective}. 
At every iteration, the algorithm finds $n \in \{1, \dots, p\}$ such that adding one more feature to the feature map $\Phi_n$ (i.e., $D_n = D_n + 1$) decreases the objective function most, and sets $D_n = D_n + 1$.
Note again that a closed form expression for $ \mathbb{V} \left[ \Phi_n(\mat{x}_i)^{\top} \overline{\Phi_n(\mat{x}_j)} \right]$ is available from \cref{tbl:variances}.

\paragraph{Time and space complexities.}
The time and space complexities of \cref{alg:incremental-algorithm} are  $\mathcal{O} (p D_{
\rm total})$  and   $\mathcal{O}(p)$, respectively.   
Note that from \cref{eq:simple-objective-individual-term},  the objective function can be written as 
$$
f(D_1, \dots, D_p) := \sum_{n=1}^{p} \ a_n^2  \sum_{i \not= j}   \mathbb{V} \left[ \Phi_n(\mat{x}_i)^{\top} \overline{\Phi_n(\mat{x}_j)} \right] = \sum_{n=1}^{p} \frac{a_n^2}{D_n} \sum_{i \not= j} C_{\mat{x}_i, \mat{x}_j}^{(n)}
$$ 
with $a_n$ and $C_{\mat{x}_i, \mat{x}_j}^{(n)}$ not depending on the optimizing variable $D_n$.
Therefore,  one can precompute the term $\sum_{i \neq j} C_{\mat{x}_i, \mat{x}_j}^{(n)}$ for each $n = 1,\dots, p$ before starting the iterations in \cref{alg:incremental-algorithm}, and during the iterations one can use the precomputed values of  $\sum_{i \neq j} C_{\mat{x}_i, \mat{x}_j}^{(n)}$. 
Thus, while the complexity of precomputing $\sum_{i \neq j} C_{\mat{x}_i, \mat{x}_j}^{(n)}$ is $\mathcal{O}(m^2)$, where $m$ is size of the dataset $\mat{x}_1,\dots,\mat{x}_m$ defining the objective function \eqref{eqn:optimization-objective}, the time and space complexities of \cref{alg:incremental-algorithm} do not depend on $m$.

\paragraph{Structured case.}
We assumed here that $\Phi_n$ is one of the unstructured sketches studied in  Section \ref{sec:polynomial-sketches}.
This choice of $\Phi_n$ makes \cref{eq:simple-objective-individual-term} convex and monotonically decreasing with respect to $D_n$, which enables the Incremental Algorithm to solve the optimization problem in \cref{eqn:optimization-objective}.

However, if $\Phi_n$ is a structured sketch (i.e., either real or complex TensorSRHT) in Section \ref{sec:structured-projections}, \cref{eq:simple-objective-individual-term} is not convex with respect to $D_n$, and the Incremental Algorithm is not directly applicable. 
To overcome this problem,  when $\Phi_n$ is a structured sketch, we propose to use convex surrogate functions in  \cref{eqn:var-convexified-2} and \cref{eqn:var-convexified-positive-cov} derived in \cref{sec:incremental-tensorsrht}  to replace $\mathbb{V} \left[ \Phi_n(\mat{x}_i)^{\top} \overline{\Phi_n(\mat{x}_j)} \right]$ in the objective function \eqref{eqn:optimization-objective}, and then apply the Incremental Algorithm. 
We summarize the concrete form of the convex surrogate function in  \cref{tbl:variances}. 
For details, see  \cref{sec:incremental-tensorsrht}.

\begin{algorithm}[t]
    \caption{Incremental Algorithm}
    \label{alg:incremental-algorithm}
    \SetAlgoLined
    \KwResult{Optimal solution $D_1, \dots, D_p \geq  1$ to the optimization problem \eqref{eqn:optimization-objective}. } 
    {\bf Input:} Dot product kernel $k(\mat{x}, \mat{y}) = \sum_{n=0}^{\infty} a_n (\mat{x}^\top \mat{y})^n$ with $a_n \geq 0$, truncation order $p \in \mathbb{N}$,   the total number of random features $D_{\rm total} \in \mathbb{N}$ \;
    Initialize $D_1 = \dots = D_p = 1$ and $t=0$ \;
Let $f(D_1, \dots, D_p) := \sum_{n=1}^{p} \ a_n^2  \sum_{i \not= j}   \mathbb{V} \left[ \Phi_n(\mat{x}_i)^{\top} \overline{\Phi_n(\mat{x}_j)} \right]$.

    \While{$t < D_{\rm total}$}{
        Find $j^* = \argmin_{j \in \{1,\dots,p\}} f(D_1, \dots, D_j+1, \dots, D_p)$ \;
        $D_{j*} = D_{j*} + 1$ \;
        $t = t + 1$ \;
    }
\end{algorithm}

\subsubsection{Solving the Full Problem}

We now address the full problem in \cref{eq:full-optimization-problem} using \cref{alg:incremental-algorithm} developed for the simplified problem in \cref{eqn:optimization-objective}. 
Recall that, by fixing $p \in \{p^*_{\rm min}, \dots, p^*_{\rm max} \}$ and constraining $D_n \geq 1$ for all $n = 1, \dots, p$, the full problem in  \cref{eq:full-optimization-problem} becomes equivalent to the simplified problem in \cref{eqn:optimization-objective}, which can be solved by \cref{alg:incremental-algorithm}.
Thus, we propose to solve the full problem in \cref{eq:full-optimization-problem}  by i) first performing \cref{alg:incremental-algorithm} for each $p \in \{p_{\rm min}, \dots, p_{\rm max}\}$,  ii) then evaluate each solution $(D_n)_{n=1}^p$ by computing the objective function  $g(p, (D_n)_{n = 1}^p )$  in \cref{eq:maclaurin-empirical-objective}, and finally pick up $p$ that gives the smallest objective function value.

\cref{alg:extended-incremental-algorithm} summarizes the whole procedure for solving the full optimization problem in \cref{eq:full-optimization-problem}.
\cref{alg:extended-incremental-algorithm} returns the optimal truncation order $p^* \in \{ p_{\rm min}, \dots, p_{\rm max} \}$ with the corresponding feature cardinalities $D^* = (D_1, \dots, D_{p^*})$. 
Given these values, one can construct a feature map as summarized in \cref{alg:improved-maclaurin}.  
Note that the U-statistics in the empirical objective \eqref{eq:maclaurin-empirical-objective} can be precomputed for all $p_{\mathrm{min}}, \dots, p_{\mathrm{max}}$ \textit{before} running any optimization algorithm. They do not need to be re-evaluated for every execution of \cref{alg:incremental-algorithm}.

\begin{algorithm}[t]
    \caption{Extended Incremental Algorithm}
    \label{alg:extended-incremental-algorithm}
    \SetAlgoLined
    \KwResult{Optimal polynomial degree $p^* \in \{p_{\rm min}, \dots, p_{\rm max}\}$ and feature cardinalities $D^* = ( D_1, \dots, D_{p^*} ) \in \mathbb{N}^{p^*}$ to the full optimization problem \eqref{eq:full-optimization-problem}.}
    {\bf Input:} Dot product kernel $k(\mat{x}, \mat{y}) = \sum_{n=0}^{\infty} a_n (\mat{x}^\top \mat{y})^n$ with $a_n \geq 0$, upper and lower bounds  $p_{\rm min}, p_{\rm max} \in \mathbb{N} $,  the total number of random features $D_{\rm total} \in \mathbb{N}$ \;

    Set $g^* = \infty$, $p^* = p_{\rm min}$ and $D^* = \{ \}$ \;
    
    \ForAll{$p \in \{p_{\rm min}, \dots, p_{\rm max} \}$}{
        Solve \cref{alg:incremental-algorithm} to obtain $D_1, \dots, D_p$ \;
        Compute $g(p, (D_n)_{n=1}^p)$ in \cref{eq:maclaurin-empirical-objective} \;
        If $g(p, (D_n)_{n=1}^p) < g^*$, set $g^* = g(p, (D_n)_{n=1}^p)$, $D^* = (D_n)_{n=1}^p$ and $p^* = p$ \;
    }
\end{algorithm}

\begin{algorithm}[t]
    \caption{Improved Random Maclaurin (RM) Features}
    \label{alg:improved-maclaurin}
    \SetAlgoLined
    \KwResult{Feature map $\Phi(\mat{x}) \in \mathbb{C}^{D_{\rm total} + 1 }$}
    {\bf Input:} 
    Dot product kernel $k(\mat{x}, \mat{y}) = \sum_{n=0}^{\infty} a_n (\dotprod{x}{y})^n$ with $a_n \geq 0$,  polynomial degree $p^* \in \mathbb{N}$ and feature cardinalities $D_1, \dots, D_{p^*}$  from \cref{alg:extended-incremental-algorithm}  \;
    
    Initialize $\Phi(\mat{x}) := [\sqrt{a_0}]$
    
    \ForAll{$n \in \{ 1, \dots, p^* \}$}{
        Let $\Phi_n(\mat{x}) \in \mathbb{C}^{D_n}$ be an unbiased polynomial sketch of degree $n$ with $D_n$ features  (see Sections \ref{sec:polynomial-sketches} and \ref{sec:structured-projections}) \;
        Append $\sqrt{a_n} \ \Phi_n(\mat{x})$ to $\Phi(\mat{x})$ \;
    }
\end{algorithm}

\subsection{Approximating a Gaussian Kernel}
\label{sec:approx-Gauss-kernel}

Here we describe how to adapt \cref{alg:incremental-algorithm} and \cref{alg:extended-incremental-algorithm} for approximating a Gaussian kernel of the form $k(\mat{x}, \mat{y}) = \exp(-\|\mat{x}-\mat{y}\|^2 / (2 l^2))$ with $l > 0$.  
By \cref{eq:Gaussian-kernel-expansion}, this Gaussian kernel can be written as
\begin{align*}
k(\mat{x}, \mat{y}) &= \exp\left(- \frac{\|\mat{x}\|^2} {2 l^2} \right) \exp\left(-\frac{\|\mat{y}\|^2}{2 l^2} \right) 
    \sum_{n=0}^\infty a_n (\dotprod{x}{y} )^n,
\end{align*}
where $a_n := 1/ (n! l^{2n})$ for $n \in \mathbb{N} \cup \{ 0 \}$.
Notice that $\left(- \frac{\|\mat{x}\|^2} {2 l^2} \right) $ and $\left(- \frac{\|\mat{y}\|^2} {2 l^2} \right) $ are scalar values and can be computed for any given input vectors $\mat{x}, \mat{y} \in \mathbb{R}^d$.

Thus, the objective function $g(p, (D_n)_{n=1}^p)$ in \cref{eq:maclaurin-empirical-objective}, which is an empirical approximation of the bias-variance decomposition of the mean squared error in \cref{eqn:bias-variance-decomposition} using an i.i.d.~sample $\mat{x}_1,\dots,\mat{x}_m \stackrel{i.i.d.}{\sim} P$,
can be adapted as
\begin{align}
& g(p, (D_n)_{n=1}^p)    \label{eq:Gauss-objective-full} \\ 
=&   \frac{1}{m(m-1)}  \sum_{n=1}^{p} \delta[D_n > 0]~ a_n^2 \sum_{i \not= j}    \exp\left(- \frac{\|\mat{x}_i\|^2} {l^2} \right) \exp\left(-\frac{\|\mat{x}_j\|^2}{l^2} \right)   \mathbb{V} \left[ \Phi_n(\mat{x}_i)^{\top} \overline{\Phi_n(\mat{x}_j)} \right]  \nonumber  \\
     + &  \frac{1}{m(m-1)} \sum_{i \not= j}    \left( k(\mat{x}_i, \mat{x}_j) -  \sum_{n=0}^{p} \delta[D_n > 0] \ a_n \exp\left(- \frac{\|\mat{x}_i\|^2} {2 l^2} \right) \exp\left(-\frac{\|\mat{x}_j\|^2}{2 l^2} \right)  (\mat{x}_i^\top \mat{x}_j)^n \right)^2. \nonumber
\end{align}
Accordingly, the objective function of the simplified problem in \cref{eqn:optimization-objective} is adapted as
$$
f(D_1,\dots,D_p) := \frac{1}{m(m-1)}    \sum_{n=1}^{p} \ a_n^2   \sum_{i \not= j}   \exp\left(- \frac{\|\mat{x}_i\|^2} {l^2} \right) \exp\left(-\frac{\|\mat{x}_j\|^2}{l^2} \right) \mathbb{V} \left[ \Phi_n(\mat{x}_i)^{\top} \overline{\Phi_n(\mat{x}_j)} \right].
$$

By these modifications, \cref{alg:incremental-algorithm} and \cref{alg:extended-incremental-algorithm} can be used to obtain the  optimal truncation order $p^* \in \{ p_{\rm min}, \dots, p_{\rm max} \}$ and the corresponding feature cardinalities $D_1, \dots, D_{p^*}$. 
Lastly, \cref{alg:improved-maclaurin} can be adapted by multiplying the scalar value $\exp\left(- \frac{\|\mat{x}\|^2} {2 l^2} \right)$  to the feature map $\Phi(\mat{x})$ obtained from \cref{alg:improved-maclaurin}:  the new feature map is defined as
$\Phi'(\mat{x}) := \exp\left(- \frac{\|\mat{x}\|^2} {2 l^2} \right) \Phi(\mat{x})$.

\subsection{Numerical Illustration of the Objective Function}
 
\label{sec:illustration-objective-optim}

\begin{figure}[t]
\centering
\includegraphics[width=1\textwidth]{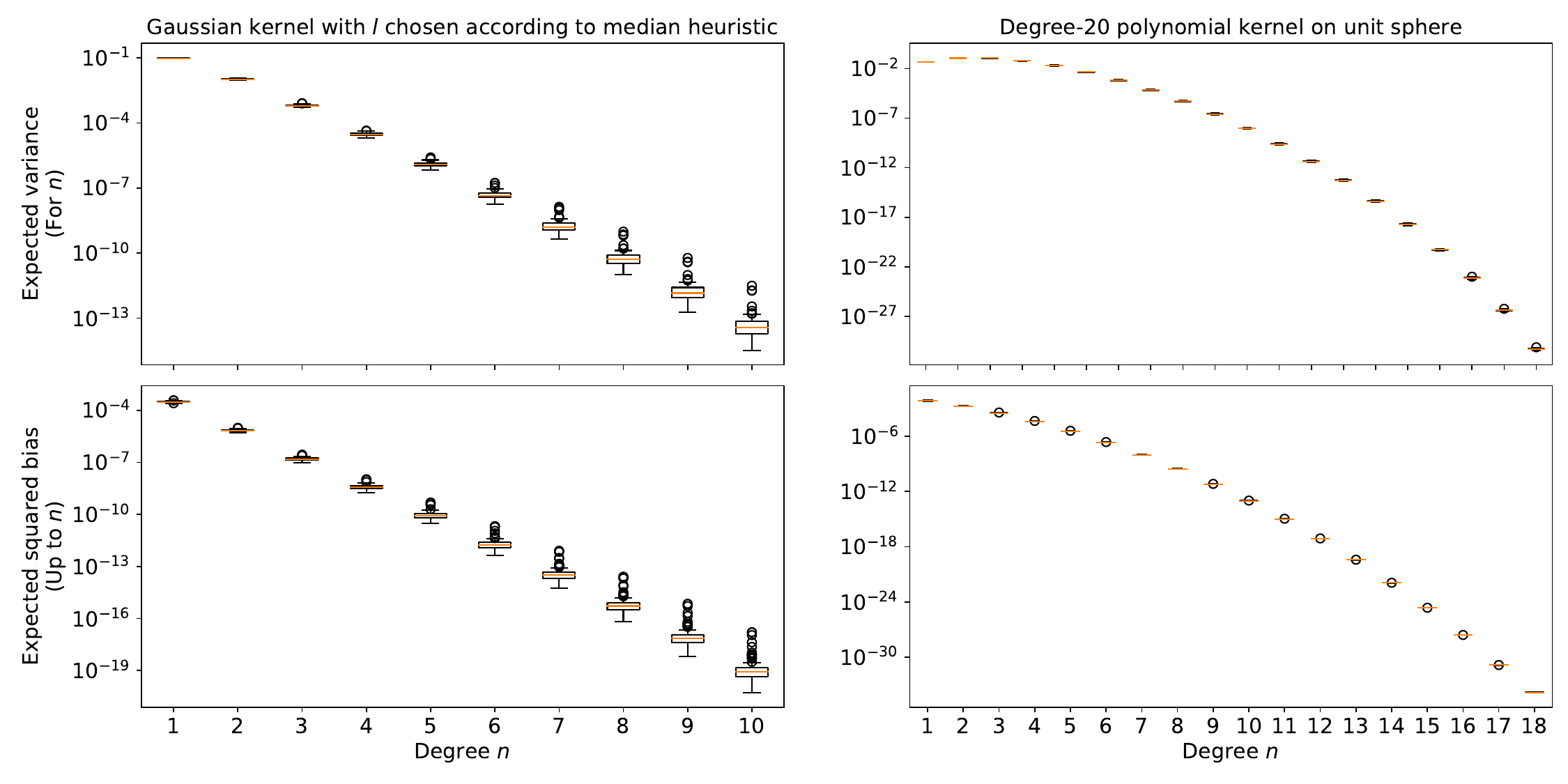}
\caption{Numerical illustration of Section \ref{sec:illustration-objective-optim}. 
The left two figures are box plots for the Gaussian kernel (i), and the right two figures are those for the polynomial kernel (ii). 
The top figures show the variance terms (a), and the bottom figures show the bias terms (b). 
See Section \ref{sec:illustration-objective-optim} for details.
}
\label{fig:bootstrap-boxplot}
\end{figure}

To gain an insight about the behavior of \cref{alg:extended-incremental-algorithm}, we provide a numerical illustration of the bias and variance terms in the objective function $g(p, (D_n)_{n=1}^p)$ in \cref{eq:maclaurin-empirical-objective} (or its version adapted for the Gaussian kernel in \cref{eq:Gauss-objective-full}). 
To this end, we used the Fashion MNIST dataset \citep{xiao2017fashionmnist} and randomly sampled data points $\mat{x}_1,\dots, \mat{x}_m$ with $m = 500$ from the entire dataset of size $60,000$.
As a target kernel to approximate, we consider (i) a polynomial kernel $k(\mat{x}, \mat{y}) = (\dotprod{x}{y} / 8 + 7/8)^{20}$ of degree $p = 20$; and  (ii) the Gaussian kernel  $k(\mat{x}, \mat{y}) = \exp(-\|\mat{x}-\mat{y}\|^2 / (2 l^2))$, where the length scale $l > 0$ is given by the median heuristic \citep{garreau2017large}, i.e., the median of the pairwise Euclidean distances of $\mat{x}_1, \dots, \mat{x}_m$.

For the polynomial kernel (i), we computed (a)  $ \frac{a_n^2 }{m(m-1)}  \sum_{i \not= j}   \mathbb{V} \left[ \Phi_n(\mat{x}_i)^{\top} \overline{\Phi_n(\mat{x}_j)} \right] $ for each  $n = 1, \dots, p~(=20)$, which is the variance component of the objective function in \cref{eq:maclaurin-empirical-objective}; and (b)  $ \frac{1}{m(m-1)} \sum_{i \not= j}    \left( k(\mat{x}_i, \mat{x}_j) -  \sum_{\nu  =0}^{n}  \ a_\mu (\mat{x}_i^\top \mat{x}_j)^\nu \right)^2$ for each $n = 1, \dots,  p~(=20)$, which is the bias component in \cref{eq:maclaurin-empirical-objective}  computed up to $n$-th order. 
For the Gaussian kernel (ii), we computed corresponding quantities from the objective function in  \cref{eq:Gauss-objective-full}: (a) $\frac{ a_n^2}{m(m-1)}    \sum_{i \not= j}    \exp\left(- \frac{\|\mat{x}_i\|^2} {l^2} \right) \exp\left(-\frac{\|\mat{x}_j\|^2}{l^2} \right)   \mathbb{V} \left[ \Phi_n(\mat{x}_i)^{\top} \overline{\Phi_n(\mat{x}_j)} \right] $
for $n = 1, \dots, 10$ and (b)  $\frac{1}{m(m-1)} \sum_{i \not= j}    \left( k(\mat{x}_i, \mat{x}_j) -  \sum_{\nu=0}^{n} \ a_\nu \exp\left(- \frac{\|\mat{x}_i\|^2} {2 l^2} \right) \exp\left(-\frac{\|\mat{x}_j\|^2}{2 l^2} \right)  (\mat{x}_i^\top \mat{x}_j)^\nu \right)^2$ for $n = 1, \dots, 10$.
We used the real Gaussian sketch for the feature map $\Phi_n$, for which \cref{eqn:normal-polynomial-est-variance} gives a closed form expression of the variance  $\mathbb{V} \left[ \Phi_n(\mat{x}_i)^{\top} \overline{\Phi_n(\mat{x}_j)} \right]$; see also \cref{tbl:variances}. 
We set $D_n = 1$ for each $n$ to be evaluated (i.e., $\Phi_n(\mat{x}) \in \mathbb{R}$.)  

To compute the means and standard deviations of the above quantities (a) and (b), we repeated this experiment 100 times by independently subsampling $\mat{x}_1, \dots, \mat{x}_m$  with $m=500$  from the entire dataset each time.
\cref{fig:bootstrap-boxplot} describes the results. 
First, we can see that the standard deviations of the quantities (a) and (b) are relatively small, and thus a subsample $\mat{x}_1,\dots,\mat{x}_m$ of size $m=500$ is sufficient for providing accurate approximations of the respective population quantities of (a) and (b) (where the empirical average with respect to $\mat{x}_1, \dots, \mat{x}_m$ is replaced by the corresponding expectation) in this setting.

Regarding the polynomial kernel (i), the variance terms (a) for polynomial degrees up to $n = 3$ have similar magnitudes, and they decay exponentially fast for polynomial degrees larger than $n = 3$ (notice that the vertical axis of the plot is in log scale).  On the other hand, the bias term (b) decays exponentially fast as the polynomial degree $n$ increases.  These trends suggest that \cref{alg:extended-incremental-algorithm} would assign more features to lower order degrees $n$, in particular to the degree 3 or less. 
One explanation of these trends is that the parametrization of the kernel $k(\mat{x}, \mat{y}) = (\dotprod{x}{y} / 8 + 7/8)^{20}$ gives larger coefficients to lower polynomial degrees in the Maclaurin expansion  (see \cref{eq:poly-ker-expansion}), and that the distribution of pairwise inner products of the data points $\mat{x}_1, \dots \mat{x}_m$ is centered around zero in this experiment.

Regarding the Gaussian kernel (ii), both the variance term (a) and the bias term (b) decay exponentially fast as the polynomial degree $n$ increases.  
This trend suggests that \cref{alg:extended-incremental-algorithm} would assign more features to lower order polynomial degrees $n$. 

To summarize, these observations suggest that, to minimize the mean squared error of the approximate kernel, it is more advantageous to assign more features to lower degree polynomial approximations.
\cref{alg:extended-incremental-algorithm} automatically achieves such feature assignments. 
Additional basic experiments for \cref{alg:extended-incremental-algorithm} are reported in \cref{app:opt-maclaurin}.

\subsection{Gaussian Process Regression Toy Example}

\label{sec:GP-toy-experiment}

\begin{figure}[t]
\centering
\includegraphics[width=0.95\textwidth]{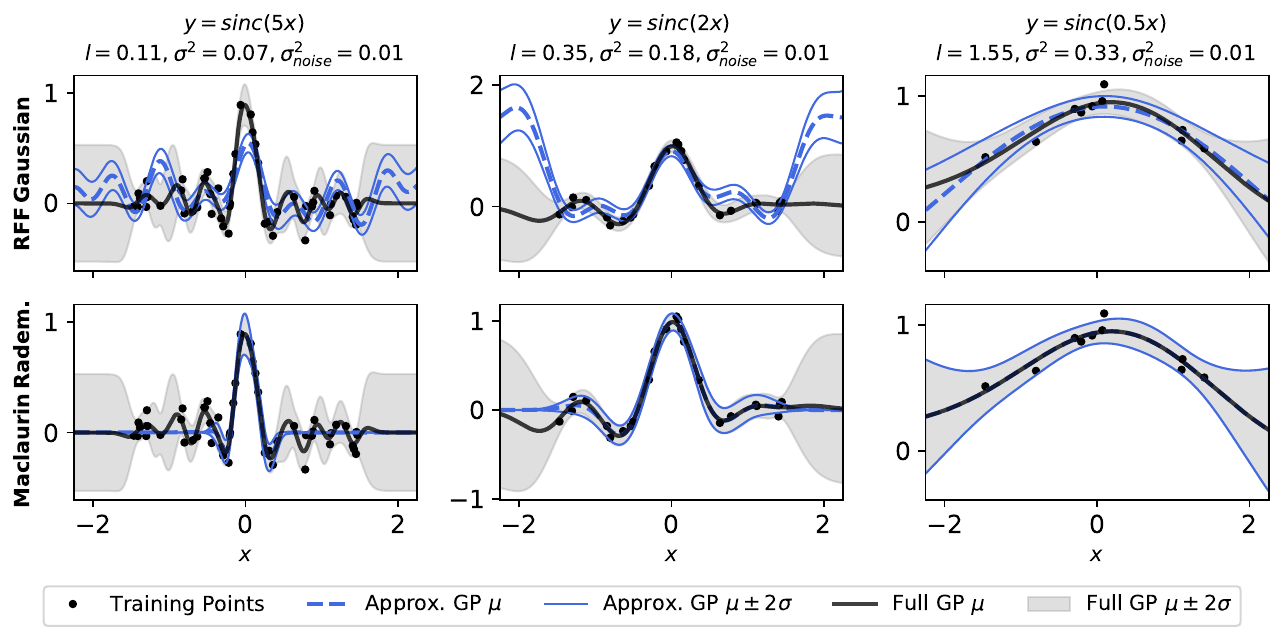}
\caption{One-dimensional GP regression experiment in Section \ref{sec:GP-toy-experiment}.
The top row shows the results of random Fourier features (Gaussian RFF), and the bottom row those of the optimized Maclaurin approach.
The left, middle, and right columns correspond to the ground-truth sinc functions with frequencies of $5$, $2$, and $0.5$, respectively. 
The values of $l$ and $\sigma^2$ are the kernel hyperparameters obtained by maximizing the log likelihood of training data in the full GP (i.e., without approximation). 
Dashed black curves represent approximate GP posterior mean functions; black curves represent the posterior means plus and minus 2 times approximate posterior standard deviations; black curves represent the posterior mean functions of the full GP; and the shaded areas are the full GP posterior means plus and minus 2 times the full GP posterior deviations. 
}
\label{fig:gaussian-1d-regression}
\end{figure}

We performed a toy experiment on one-dimensional Gaussian process (GP) regression, whose results are described in \cref{fig:gaussian-1d-regression}. 
The purpose is to gain a qualitative understanding of the optimized Maclaurin approximation in Section \ref{sec:improving-maclaurin} (\cref{alg:extended-incremental-algorithm}).
For comparison, we also used Random Fourier Features (RFF) of \citet{Rahimi2007} in this experiment.
We use the real Rademacher sketch in the optimized Maclaurin approach. 

We define the ground-truth function as a sinc function,  $f(x) =  \sin(ax) / x$, with $a > 0$, for which we consider three settings: $a \in \{5, 2, 0.5\}$.
We generated training data by adding independent Gaussian noises of variance $\sigma^2_{\mathrm{noise}} = 0.01$ to the ground-truth function $f(x)$. %
With this value of noise variance $\sigma^2_{\rm noise}$, we then fit a GP regressor using the Gaussian kernel $k(x,y) = \sigma^2 \exp (- (x-y)^2 / (2l^2))$ to the training data, where we determined the hyperparameters $l, \sigma^2 > 0$  by maximizing the log marginal likelihood \citep[e.g.,][Chapter 2]{Rasmussen2006}. 
We used the resulting posterior GP as a ground-truth and call it ``full GP'', treating it as a reference for assessing the quality of approximate GPs.
As such, we used the same hyperparameters as the full GP in approximate GPs; this enables evaluating  the effects of the approximation in the resulting GP predictive distributions.

We set the number of random features as $D = 10$.
In this case, the optimized Maclaurin approach in \cref{alg:extended-incremental-algorithm} selects the truncation degree $p^*=9$ %
and simply allocates the feature cardinalities as $D_{1} = \dots = D_{9} = 1$. (Note that one feature is always allocated to the degree $n=0$). 
This behavior is because the variance of the Rademacher sketch in \cref {eqn:rademacher-variance} is zero for all polynomial degrees $n$, as the input dimension is one ($d=1$) in this experiment.\footnote{Thus, the error of the optimized Maclaurin approach stems solely from the finite truncation of the Maclaurin expansion in \cref{eqn:approximate-kernel}.}

We can make the following observations from \cref{fig:gaussian-1d-regression}.
First, with the optimized Maclaurin approach, the approximate GP posterior mean function approximates the full GP posterior mean function around $x = 0$ more accurately than RFF. Moreover, the range of $x$ on which the Maclaurin approach is accurate becomes wider for a lower frequency $a$ of the ground-truth sinc function (for which the length scale $l$ is larger).  This tendency suggests that the Maclaurin approach may be more advantageous than RFF in approximating around $x = 0$ and when the length scale $l$ is relatively large.  Experiments in the next section, in particular those with high dimensional datasets, provide further support for this observation. 
 
 One issue with the Maclaurin approximation is that, as can be seen from \cref{fig:gaussian-1d-regression},  the approximate GP posterior variance tends to collapse for an input location $x$ far from $0$.  This behavior may be explained as follows. 
 Recall that in general, the GP posterior variance at location $\mat{x}$ with an approximate kernel $\hat{k}$ can be written in the form 
 \begin{equation} \label{eq:GP-posterior-var-2585}
 \hat{k}(\mat{x}, \mat{x}) - \hat{\ell}_N(\mat{x}, \mat{x}), 
 \end{equation}
 where $\hat{\ell}_N(\mat{x}, \mat{x}) \geq 0$ is a data-dependent term \citep[see e.g.,][Chapter 2]{Rasmussen2006}.
 Since $\hat{\ell}_N(\mat{x}, \mat{x})$ is non-negative, the GP posterior variance is thus upper-bounded by $\hat{k}(\mat{x}, \mat{x})$.   
Note that the expectation of the Maclaurin-approximate kernel in \cref{eqn:approximate-kernel}  for the Gaussian kernel (see also \cref{eq:Gaussian-kernel-expansion}) is given by $$\mathbb{E} [\hat{k}(\mat{x}, \mat{x})] =  \exp\left(- \left\| \frac{\mat{x} }{l} \right\|^2 \right) \cdot \sum_{n=0}^p \frac{1}{n!} \left\|\frac{ \mat{x} }{ l } \right\|^{2n},$$ 
which decays to $0$ when $\|\mat{x} / l\|$ is large (because of the finite truncation of the Maclaurin expansion).
Therefore,  when $\|\mat{x} / l\|$ is large, the approximate GP posterior variance would decay to $0$ accordingly.

One possible (and easy) way of fixing this issue is 
to add a bias correction term $k(\mat{x}, \mat{x}) - \mathbb{E} [\hat{k}(\mat{x}, \mat{x})] \geq 0$ to the posterior variance in \cref{eq:GP-posterior-var-2585}. 
In this way, we can prevent the underestimation of the posterior variances with the Maclaurin approach where $\| \mat{x}/l \|$ is large, which is where the approximate GP posterior mean function may not be accurate and thus preventing the underestimation is desirable.

\section{Experiments}
\label{sec:experiments}

In this section, we perform systematic experiments to evaluate the various approaches to approximating dot product kernels discussed in this paper.
These approaches include real and complex polynomial sketches in Sections \ref{sec:polynomial-sketches} and \ref{sec:structured-projections}, as well as the optimized Maclaurin approach in Section \ref{sec:approx-dot-prod-kernels}.
We consider approximations of both polynomial kernels and Gaussian kernels. 

We evaluate the performance of each approximation approach in terms of both i) the accuracy in kernel approximation and ii) the performance in downstream tasks. 
The downstream tasks we consider are Gaussian process regression and classification. 
For completeness, we explain how to use complex-valued random features in  Gaussian process inference and discuss the resulting computational costs in \cref{sec:appendix-gp-identities}.

In Section \ref{sec:experiment-setup}, we first explain the setup of the experiments.  
In Section \ref{sec:exp-poly-kernel-frobenius}, we describe experiments on polynomial kernel approximation, comparing various approximation approaches.  
In Section \ref{sec:gp-convergence-experiment}, we report the results of the wall-clock time comparison of real and complex random features, focusing on the downstream task performance of GP classification.  
In Section \ref{sec:exp-opt-Macl}, we present detailed evaluations of the optimized Maclaurin approach for polynomial and Gaussian kernel approximations in GP classification and regression.
Additional experiments are reported in \cref{sec:additional-experiments}.

\subsection{Experimental Setup}

\label{sec:experiment-setup}

 We explain here the common setup for the experiments in this section. 

\begin{table}[t]
\begin{center}
\resizebox{0.9 \textwidth}{!}{
\begin{tabular}{l r r | l r r}
    \toprule
    Classification & Num.~data points $N$ & Dimensionality $d$ & Regression & Num.~data points $N$ & Dimensionality $d$ \\
    \midrule
    Adult & 48,842 & 128 & Boston & 506 & 16 \\
    Cod\_rna & 331,152 & 8 & Concrete & 1,030 & 8 \\
    Covertype & 581,012 & 64 & Energy & 768 & 8 \\
    EEG & 14,980 & 16 & kin8nm & 8,192 & 8 \\
    FashionMNIST & 70,000 & 1,024 & Naval & 11,934 & 16 \\
    Magic & 19,020 & 16 & Powerplant & 9,568 & 4 \\
    MNIST & 70,000 & 1,024 & Protein & 45,730 & 16 \\
    Mocap & 78,095 & 64 & Yacht & 308 & 8 \\
    \bottomrule
\end{tabular}
}
\end{center}
\caption{Datasets used in the experiments. The left and right columns are datasets for classification and regression, respectively. 
}
\label{tbl:datasets}
\end{table}

\subsubsection{Datasets}

\cref{tbl:datasets} shows an overview of the datasets used in the experiments. All the datasets come from the UCI benchmark \citep{Dua:2019} except for Cod\_rna \citep{DBLP:journals/bmcbi/UzilovKM06}, FashionMNIST \citep{xiao2017fashionmnist}, and MNIST \citep{lecun1998}. We pad input vectors with zeros so that the input dimensionality $d$ becomes a power of two to support Hadamard projections in TensorSRHT. %
The train/test split is 90/10 and is recomputed for every random seed for the UCI datasets; otherwise it is predefined.

For each dataset, we use its random subsets of size $m=\min(5000, N_{\mathrm{train}})$ and $m_* = \min(5000, N_{\mathrm{test}})$ to define training and test data in an experiment, respectively, where $N_{\rm train}$ and $N_{\rm test}$ are the sizes of the original training and test datasets. 
Denote by  $X_{\mathrm{sub}} = \{\mat{x}_{1}, \dots, \mat{x}_{m}\}$ and $X_{*, \mathrm{sub}} = \{\mat{x}_{*,1}, \dots, \mat{x}_{*,m_*}\}$ those subsets for training and test, respectively. 
We repeat each experiment 10 times independently using 10 different random seeds, and hence with 10 different subset partitions.

\subsubsection{Target Kernels to Approximate} 

We consider approximation of (i) polynomial kernels and (ii) Gaussian kernels. 

\paragraph{(i) Polynomial kernel approximation.}
We consider  a polynomial kernel of the form
\begin{equation}
    \label{eqn:poly-target-kernel}
    k(\mat{x}, \mat{y})
    = \sigma^2 \left(\left(1-\frac{2}{a^2}\right) + \frac{2}{a^2} \mat{x}^\top \mat{y} \right)^p
    = \sigma^2 \left(1 - \frac{\|\mat{x} - \mat{y}\|^2}{a^2} \right)^p
\end{equation}
with $p \in \mathbb{N}$, $a \geq 2$, $\sigma^2 > 0$,  and $\norm{x}=\norm{y}=1$. 
We choose this form of polynomial kernels because we use {\em Spherical Random Features (SRF)} of \citet{Pennington2015} as one of our baselines, and because SRF approximates the polynomial kernel in \cref{eqn:poly-target-kernel} defined on the unit sphere of $\mathbb{R}^d$.
We follow a similar experimental setup to the one of \citet{Pennington2015}, by setting $a=2$ and  $p \in \{3, 7, 10\}$ in \cref{eqn:poly-target-kernel}. 
Compared to \citet{Pennington2015} we drop the case $p=20$ focusing on more realistic cases of smaller $p$.
To make SRF applicable, we unit-normalize the input vectors in each dataset so that they lie on the unit sphere in $\mathbb{R}^d$.
In an experiment where we zero-centralize the input vectors, we unit-normalize after applying the zero-centering.
We set $\sigma^2$ as the variance of the labels of training subset $X_{\rm sub}$.

\paragraph{(ii) Gaussian kernel approximation.}
We consider the approximation of the Gaussian kernel $k(\mat{x}, \mat{y}) = \sigma^2 \exp(-\|\mat{x} - \mat{y}\|^2 / (2 l^2))$, where we choose the length scale $l > 0$ by the median heuristic \citep{garreau2017large}, i.e., as the median of pairwise Euclidean distances of input vectors in the training subset $X_{\mathrm{sub}}$. We set $\sigma^2 > 0$ as the variance of the labels of~$X_{\rm sub}$.

\subsubsection{Error Metrics}
\label{sec:error-metrics}

We define several error metrics for studying the quality of each approximation approach.

\paragraph{Relative Frobenius norm error.}
To define this error metric, we need to define some notation.
Let $\Phi:\mathbb{R}^d \to \mathbb{C}^D$ be the (either real or complex) feature map of a given approximation method. 
For test input vectors $\mat{X}_{*, \mathrm{sub}} = \{\mat{x}_{*,1}, \dots, \mat{x}_{*,m_*}\}$, let $\hat{\mat{K}} \in \mathbb{C}^{m_* \times m_*}$ be the approximate kernel matrix such that $\hat{\mat{K}}_{i,j} = \Phi(\mat{x}_i)^\top \overline{\Phi(\mat{x}_j)}$.
 Similarly, let  $\mat{K} \in \mathbb{R}^{m_* \times m_*}$ be the exact kernel matrix such that $\mat{K}_{i,j} = k(\mat{x}_{*,i}, \mat{x}_{*,j})$ with $k$ being the target kernel.

We then define the {\em relative Frobenius norm error} of $\hat{\mat{K}}$ against $\mat{K}$ as:
\begin{align}
    \label{eqn:rel-frob-error}
    \|\mat{K} - \hat{\mat{K}} \|_F / \| \mat{K} \|_F :=
    \sqrt{\sum_{i=1}^m \sum_{j=1}^m | \mat{K}_{i,j} - \mat{\hat{K}}_{i,j}|^2} \; \Big/ \;
    \sqrt{\sum_{i=1}^m \sum_{j=1}^m \mat{K}_{i,j}^2} .
\end{align}
This error quantifies the quality of the feature map $\Phi$ in terms of the resulting approximation accuracy of the kernel matrix. 
As the target kernel matrix $\mat{K}$ is real-valued, we discard the imaginary part of  $\hat{\mat{K}}$ if it is complex-valued, unless otherwise specified. \\[5pt]

We define other error metrics in terms of two downstream tasks: Gaussian process (GP) regression and classification (see \cref{sec:appendix-gp-identities} for details of these GP tasks).

\paragraph{Kullback-Leibler (KL) divergence.}
We measure the {\em KL divergence} between two posterior predictive distributions at test input points: one is that of an approximate GP and the other is that of the exact GP without approximation; see \cref{eqn:kl-div} in \cref{sec:appendix-gp-identities} for details.
For GP classification, we measure the KL divergence between the corresponding latent GPs before transformation. %
Since there are as many GPs as the number of classes, we report the KL divergence averaged over those classes.

\paragraph{Prediction performance.}
For GP classification, we use the {\em test error rate} (i.e., the percentage of misclassified examples)  for measuring the prediction performance.
For GP regression, we report the {\em normalized mean squared error (norm.~MSE)} between the posterior predictive outputs and true outputs, normalized by the variance of the test outputs. 
Here, we use the full training data of size $N_{\rm train}$ for computing the approximate GP posterior and the full test data of size $N_{\rm test}$ for evaluating the prediction performance.\footnote{We did not use the full training and test datasets for evaluating the KL divergence, since it requires computing the exact GP posterior on the full training data of size $N_{\rm train}$, which costs $\bigO(N_{\rm train}^3)$ and is not feasible for datasets with large $N_{\rm train}$. }

\paragraph{Mean negative log likelihood (MNLL).}
We compute the {\em mean negative log likelihood (MNLL)} of the test data for the approximate GP predictive distribution. 
MNLL can capture the quality of prediction uncertainties of the approximate GP model \citep[e.g.][p. 23]{Rasmussen2006}.
We use the full training and test data for computing the MNLL, as for the prediction performance.

\subsubsection{Other Settings}
\label{sec:other-settigns}
\paragraph{Optimized Maclaurin approach.} 
For the optimized Maclaurin approach in  \cref{alg:extended-incremental-algorithm}, we set $p_{\mathrm{min}}=2$ and $p_{\mathrm{max}}=10$. 
We use the training subset $X_{\mathrm{sub}} = \{\mat{x}_{1}, \dots, \mat{x}_{m}\}$  to precompute the U-statistics in \cref{eq:maclaurin-var-estimate} and \cref{eq:maclaurin-bias-estimate}.

\paragraph{Regularization parameters.}

We select the regularization parameter in  GP classification and regression by a training-validation procedure. 
That is, we use the 90 \% of training data for training and the remaining 10 \% for validation, and select the regularization parameter that maximizes the MNLL on the validation set. 
For GP classification, we choose the regularization parameter from the range $\alpha \in \{10^{-5}, \dots, 10^{-0}\}$. 
For GP regression, we choose the noise variance from the range $\sigma^2_{\mathrm{noise}} \in \{2^{-15}, \dots, 2^{15}\}$. 
See \cref{sec:appendix-gp-identities} for the definition of these parameters.

Importantly, we perform this selection procedure using a baseline approach,\footnote{More specifically, we use the Spherical Random Features (SRF) \citep{Pennington2015} when the target kernel is a polynomial kernel, and Random Fourier Features \citep{Rahimi2007} when the target kernel is Gaussian, for selecting the regularization parameter.} and after selecting the regularization parameter, we set the {\em same} regularization parameter for all the approaches (including our optimized Maclaurin approach) for computing error metrics.  
In this way, we make sure that the selected regularization parameter is not in favour of our approaches (and in this sense we give an advantage to the baseline).

\subsection{Polynomial Kernel Approximation} \label{sec:exp-poly-kernel-frobenius}

We first study the approximation of the polynomial kernels in \cref{eqn:poly-target-kernel}, comparing different polynomial sketches in terms of the relative Frobenius norm error in \cref{eqn:rel-frob-error} on  FashionMNIST.
\cref{fig:poly-frob-comparison-revised} describes the results.
We consider the following polynomial sketches in this experiment:

 \paragraph{(i) Gaussian and Rademacher sketches (Section \ref{sec:polynomial-sketches}).} 
 We use the real Gaussian and Rademacher sketches, i.e., the unstructured polynomial sketches in  \cref{eqn:polynomial-estimator} with Gaussian and Rademacher weights (``Gaussian'' and ``Rademacher'', respectively, in \cref{fig:poly-frob-comparison-revised}), along with their complex counterparts (``Gaussian Comp.'' and ``Rademacher Comp.'' in \cref{fig:poly-frob-comparison-revised}).

\paragraph{(ii) TensorSRHT (Section  \ref{sec:structured-projections}).}
 We consider the real TensorSRHT in  \cref{eqn:real-struct-polynomial-sketch} with Rademacher weights (``TensorSRHT'' in \cref{fig:poly-frob-comparison-revised}), and the complex TensorSRHT in \cref{eqn:complex-struct-polynomial-sketch} with complex Rademacher weights (``TensorSRHT Comp.'' in the figure); see also \cref{alg:tensor-srht-algorithm}.

\paragraph{(iii) Random Maclaurin  (Section \ref{sec:approx-dot-prod-kernels}).}
We use the Random Maclaurin approach explained in Section \ref{sec:random-maclaurin}.
To improve its performance, we truncate the support of the importance sampling measure $\mu(n)=2^{-(n+1)}$ in \cref{eq:sampling-dist-rand-Mac} to degrees $n \in \{1, \dots, p\}$.\footnote{Without this restriction of the support, the randomized Maclaurin approach may sample polynomial degrees $n$ such that  $n > p$ from $\mu(n)$, for which the associated coefficient in the Maclaurin expansion in \cref{eq:poly-ker-expansion}  is zero.  
Therefore, the resulting feature maps may contain zeros,  which are redundant and make the kernel approximation inefficient.}  
Note that the term $n = 0$ in \cref{eq:dot-prod-kernel-expand} associated with coefficient $a_0$ does not need to be approximated, as we append $\sqrt{a_0}$ to the feature map. %
We consider the Random Maclaurin approach using the real Rademacher sketch (``Rnd.~Macl.~Radem.'' in \cref{fig:poly-frob-comparison-revised}).

\paragraph{(iv) Optimized Maclaurin (Section \ref{sec:approx-dot-prod-kernels}).}
We consider the optimized Maclaurin approach in Section \ref{sec:improving-maclaurin} using the Rademacher approach (``Opt.~Macl.~Radem.'' in \cref{fig:poly-frob-comparison-revised}). 
We also include the real and complex versions involving TensorSRHT (``Opt.~Macl.~TensorSRHT'' and ``Opt.~Macl.~TensorSRHT~Comp.'', respectively, in \cref{fig:poly-frob-comparison-revised}).

\paragraph{(v) TensorSketch}
For completeness, we also include in this experiment {\em TensorSketch} of \citet{Pham2013}, a state-of-the-art polynomial sketch (``TensorSketch'' in %
\cref{fig:poly-frob-comparison-revised}).

\paragraph{Setting.}
We perform the experiments using 
FashionMNIST (``Non-centered data'' in \cref{fig:poly-frob-comparison-revised}) and its centered version for which we subtract the mean of the input vectors from each input vector (``Centered data'' in \cref{fig:poly-frob-comparison-revised}).
For each approach, the number of random features is $D \in \{ d, 3d, 5d\}$, where $d = 1,024$ for FashionMNIST.
\\

\begin{figure}[t]
    \centering
    \includegraphics[width=1\textwidth]{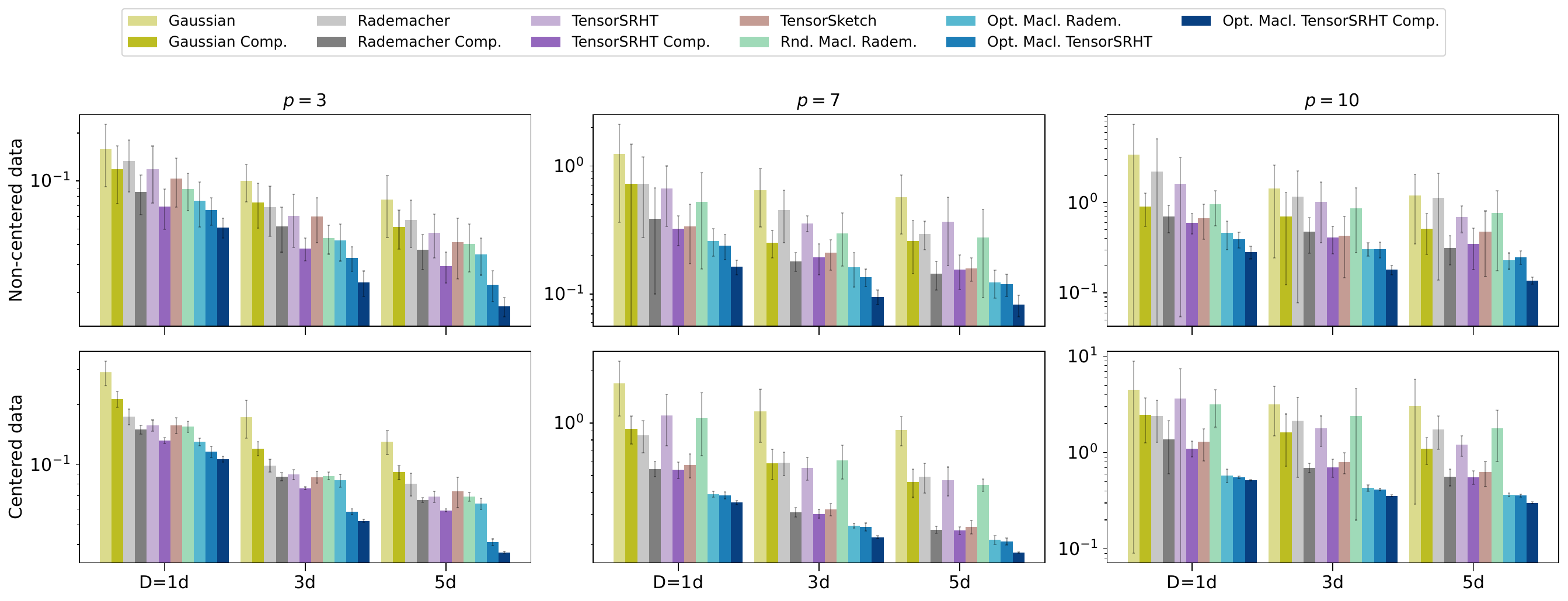}
    \caption{
    Results of the experiments in Section \ref{sec:exp-poly-kernel-frobenius} using FashionMNIST.
    Each plot shows the relative Frobenius norm errors in \cref{eqn:rel-frob-error} of different  sketches for approximating the polynomial kernel in \cref{eqn:poly-target-kernel} with $p \in \{3, 7, 10\}$ and $D \in \{1d, 3d,5d\}$.
    The top and bottom rows show results without and with zero-centring the data, respectively.
}
    \label{fig:poly-frob-comparison-revised}
\end{figure}

From the results in 
\cref{fig:poly-frob-comparison-revised}, we can make the following observations.

\paragraph{Effectiveness of the optimization approach.}
 The optimized Maclaurin approach with the Rademacher sketch (``Opt.~Macl.~Radem.'') achieves smaller errors than the  corresponding random Maclaurin approach (``Rnd.~Macl.~Rad.'') for all cases, and with a large margin for $p=7$  and $p=10$.  
 This improvement demonstrates the effectiveness of the proposed optimization approach that allocates more features to polynomial degrees with larger variance reduction. %

\paragraph{Variance reduction by complex features.}
Complex TensorSRHT achieves significantly smaller errors than the real TensorSRHT (``TensorSRHT''), in particular for small polynomial degrees $p$.
These improvements show the effectiveness of complex features in variance reduction, corroborating the preliminary results shown in  Figures  \ref{fig:comparison-comp-real-over-p} and \ref{fig:var-real-complex-TensorSRHT}.
 The optimized Maclaurin approach using complex features (``Opt.~Macl.~TensorSRHT~Comp.'') also achieves smaller errors than the optimized Maclaurin approach using real features (``Opt.~Macl.~TensorSRHT'') and is quite significant across all methods.

 \paragraph{Effectiveness of complex features on non-negative data.}
The improvements by complex features are more significant for the non-centered data than those for the centered-data. 
The non-centered data here consist of {\em non-negative} input vectors,  as FashionMNIST consists of such vectors.
This observation agrees with the discussion in Section \ref{sec:var-comp-poly} suggesting that complex features yield an approximate kernel whose variance is smaller than that of real features, if the input vectors are non-negative.

\paragraph{TensorSRHT v.s.~TensorSketch.}
While the real TensorSRHT produces larger errors than TensorSketch for all the cases except $p=3$,  the complex TensorSRHT outperforms TensorSketch for all the cases.
This comparison shows that the use of complex features can make TensorSRHT competitive to the state-of-the-art (and one can further improve its performance by using it in the optimized Maclaurin approach).

\subsection{Wall-Clock Time Comparison of Real and Complex Random Features  in GP Classification   }

\label{sec:gp-convergence-experiment}

We consider GP classification using the polynomial kernel in \cref{eqn:poly-target-kernel}, and compare the approximation quality of real and complex random features, in terms of both the number of features and wall-clock time.
As explained in \cref{sec:GP-efficient-implement}, the cost of computing an approximate GP posterior using $D$ complex random features is higher than that using $D$ real features.\footnote{Specifically, if one uses $D$ complex features, then the inversion of the matrix in \cref{eqn:gp-bottleneck}  requires 4 times as many floating point operations as the case of using $D$ real features. Note that, if one instead uses $2D$ real features, then the inversion of the matrix in \cref{eqn:gp-bottleneck} requires $8$ times as many operations as the case of using $D$ real features. 
Thus, doubling the number of real features is 2 times more expensive than using complex features. See  \cref{sec:GP-efficient-implement} for details.}
Therefore, to evaluate the relevance of complex features in practice, we investigate here the approximation quality of complex random features and that of real random features in GP classification, when both are given the {\em same} computational budget (in wall-clock time).

\paragraph{Setting.}
We use the Rademacher sketch and TensorSRHT, and their respective complex versions. 
For each polynomial sketch, we compute the KL divergence \eqref{eqn:kl-div}  between the approximate and exact GP posteriors (see  \cref{sec:appendix-gp-identities} for details), and record wall-clock time (in seconds) spent on constructing random features and on computing the approximate GP posterior.\footnote{We recorded the time measurements on an NVIDIA P100 GPU and PyTorch version 1.10 with native complex linear algebra support.} 
We use FashionMNIST for this experiment. 

 \paragraph{Results.}
\cref{fig:poly-kl-div-comparison} describes the results. 
The approximate GPs using complex random features achieve equal or lower KL-divergences than those using real features of the same computation time, for all the cases.  %
In particular, the improvements of complex features are larger for higher polynomial degrees $p$ and for the non-centered (and thus non-negative) data. 
These observations agree with the corresponding observation in Section \ref{sec:exp-poly-kernel-frobenius} and the discussion in Section \ref{sec:var-comp-poly} on when complex features yield lower variances than real features.

\begin{figure}[t]
    \centering
    \includegraphics[width=1\textwidth]{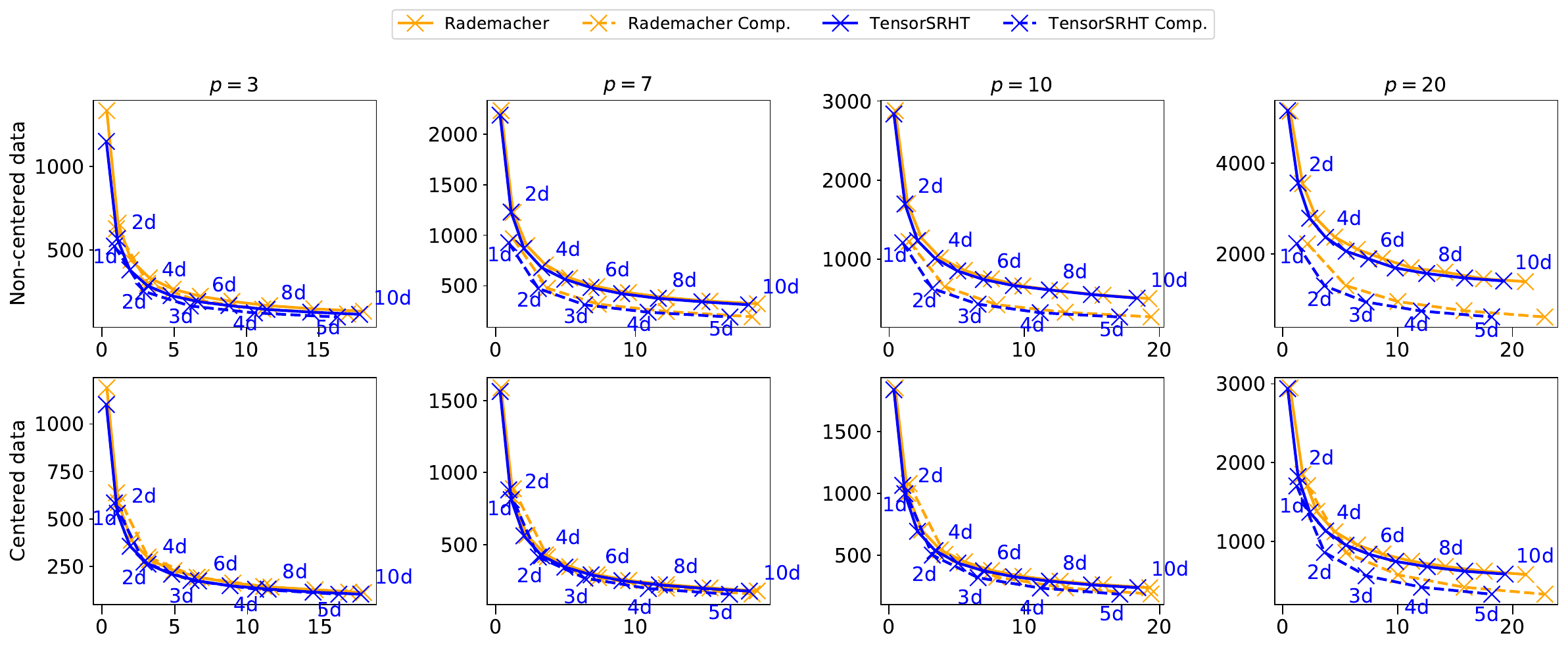}
    \caption{Results of the experiments in Section \ref{sec:gp-convergence-experiment} on wall-clock time comparison of real and complex random features in GP classification on  FashionMNIST.
 In each plot, the vertical axis shows the KL divergence \eqref{eqn:kl-div} between the approximate and the exact GP posteriors for each polynomial sketch, and the horizontal axis is wall-clock time (in seconds) spent on constructing random features and on computing the approximate GP posterior.  
Each column corresponds to a different degree $p \in \{3, 7, 10, 20\}$ of the polynomial kernel in  \cref{eqn:poly-target-kernel}.
The top row shows results on the non-centered (thus non-negative) data, and the bottom row to those on the zero-centered data. 
 The number of random features is $D \in \{1d, \dots, 10d\}$ for real features, and $D \in \{1d, \dots, 5d\}$ for complex features, as annotated next to the respective measurements in each plot.}
    \label{fig:poly-kl-div-comparison}
\end{figure}

\subsection{Systematic Evaluation of the Optimized Maclaurin Approach}

\label{sec:exp-opt-Macl}

Lastly, we systematically evaluate the performance of the optimized Maclaurin approach in  Section \ref{sec:approx-dot-prod-kernels}. 
We run experiments on approximate GP classification and regression on a variety of datasets, using a high-degree polynomial kernel and the Gaussian kernel.

\paragraph{Optimized Maclaurin approach.}
We consider the optimized Maclaurin approach in Section \ref{sec:improving-maclaurin}. 
Similarly to the previous experiments, we compare the optimized version for the Rademacher sketch (``Opt.~Macl.~Radem.'') and the real and complex versions of the optimized Maclaurin method using TensorSRHT  (``Opt.~Macl.~TensorSRHT'' and ``Opt.~Macl.~TensorSRHT~Comp.'').

\paragraph{Baselines.}
We use here approximation approaches based on Random Fourier Features (RFF)   \citep{Rahimi2007} and their extensions such as  {\em Spherical Random Features} (SRF) \citep{Pennington2015} and  {\em Structured Orthogonal Random Features} (SORF) \citep{Yu2016}  as baselines. 
The latter two approaches constitute the state-of-the-art.

These approaches generate a set of frequency samples $\omega_1, \dots, \omega_{D/2} \in \mathbb{R}^d$ (suppose $D$ is even for simplicity) from a certain spectral density, and construct a feature map\footnote{There is another popular version of the feature map in \cref{eq:RFF} defined as $\Phi_\mathcal{R}( \mat{x} ) = \sqrt{\frac{2}{D}} \left[  \cos( \mat{w}_1^\top \mat{x} + b_1 ), \dots, \cos(\mat{w}_D^\top \mat{x} + b_D )  \right]^\top \in \mathbb{R}^D$ with $b_1,\dots,b_D$ uniformly sampled on $[0, 2\pi]$. Following \citet{Sutherland2015} who suggested the superiority of  \cref{eq:RFF}, we use \cref{eq:RFF} here in all the methods using RFF, including SRF and SORF.} of dimension $D$ as, for any $\mat{x} \in \mathbb{R}^d$,  
\begin{equation} \label{eq:RFF}
\Phi_\mathcal{R}( \mat{x} ) =  \sqrt{\frac{2}{D}}  \left[ \cos( \mat{w}_1^\top \mat{x}), \dots, \cos( \mat{w}_{D/2}^\top \mat{x}), \sin( \mat{w}_1^\top \mat{x}), \dots, \sin( \mat{w}_{D/2}^\top \mat{x})  \right]^\top \in \mathbb{R}^{D}.
\end{equation}
Each approach has its own way of generating the frequency samples $\mat{\omega}_1, \dots, \mat{\omega}_{D/2}$: the original RFF generates them in an i.i.d.~manner from the spectral density of a kernel,  SORF uses structured orthogonal matrices (thus we may call it ``RFF Orth.''), and SRF uses a certain optimized spectral density. 

For a thorough comparison, we also consider a complex version of these RFF-based approaches. 
By generating frequency samples $\mat{\omega}_1, \dots, \mat{\omega}_D \in \mathbb{R}^d$ in the specific way of each approach, one can define  a corresponding complex feature map as, for any $\mat{x}\in \mathbb{R}^d$, 
\begin{equation} \label{eq:complex-RFF}
\Phi_\mathcal{C}(\mat{x}) := \sqrt{\frac{1}{D}} \left[  \exp(\iu \mat{\omega}_1^\top \mat{x} ), \dots,  \exp(\iu \mat{\omega}_D^\top \mat{x} )  \right]^\top  \in \mathbb{C}^D. 
\end{equation}
One can see\footnote{
Define an approximate kernel with \cref{eq:complex-RFF} as
  $  \hat{k} (\mat{x}, \mat{y}) := \Phi_\mathcal{C}(\mat{x})^{\top} \overline{\Phi_{\mathcal{C}}(\mat{y})} = \frac{1}{D} \sum_{i=1}^D \exp(\iu \mat{\omega}_i^{\top} (\mat{x} - \mat{y}))
    = \frac{1}{D} \sum_{i=1}^D \exp(\iu \mat{\omega}_i^{\top} \mat{x}) \overline{\exp(\iu \mat{\omega}_i^{\top} \mat{y})}$.
By taking its real part, we have
  $  \mathcal{R} \{ \hat{k} (\mat{x}, \mat{y}) \}
   = \frac{1}{D} \sum_{i=1}^D \cos(\mat{w}_i^{\top} (\mat{x} - \mat{y})) 
    = \frac{1}{D} \sum_{i=1}^D \left( \cos(\mat{w}_i^{\top} \mat{x}) \cos(\mat{w}_i^{\top} \mat{y})
    + \sin(\mat{w}_i^{\top} \mat{x}) \sin(\mat{w}_i^{\top} \mat{y}) \right) =: \Phi_{\mathcal{R}}(\mat{x})^\top \Phi_{\mathcal{R}}(\mat{y})$,
  where $\Phi_{\mathcal{R}}(\mat{x}) :=  \sqrt{\frac{1}{D}}  \left[ \cos( \mat{w}_1^\top \mat{x}), \dots, \cos( \mat{w}_{D}^\top \mat{x}), \sin( \mat{w}_1^\top \mat{x}), \dots, \sin( \mat{w}_{D}^\top \mat{x})  \right]^\top \in \mathbb{R}^{2D}$ is the $2D$-dim.~version of  \cref{eq:RFF}.
}
that \cref{eq:complex-RFF} is a complex version of \cref{eq:RFF}  by defining an approximate kernel with $\Phi_\mathcal{C}(\mat{x})$ and taking its real part, which recovers \cref{eq:RFF} of dimension $2D$.

\subsubsection{Approximate GP Inference with Polynomial Kernels} 

\label{sec:GP-inference-poly}

We first consider approximate GP classification and regression with polynomial kernels.

\paragraph{Setting.}
We set the polynomial degree to $p=3$ and $p=7$ to test various approaches on low and moderate degrees. 
We apply zero-centering to each dataset (i.e., we subtract the mean of input vectors from each input vector), as it improves the MNLL values on most datasets (see \cref{sec:additional-experiments} for supplementary experiments). %
We evaluate all the four error metrics in Section \ref{sec:error-metrics}, including the relative Frobenius norm error in \cref{eqn:rel-frob-error}.
For each approach, the number of random features is $D \in \{ d, 3d, 5d \}$ with $d$ being the dimensionality of input vectors.

\paragraph{Baselines.}
As a baseline, we use SRF \citep{Pennington2015}, a state-of-the-art approach to approximating polynomial kernels defined on the {\em unit sphere} in $\mathbb{R}^d$. 
 \citet{Pennington2015} show that  SRF works particularly well for approximating high degree polynomial kernels, and significantly outperforms the Random Maclaurin approach \citep{Kar2012} and TensorSketch \citep{Pham2013} for such kernels.

We also consider two other extensions of SRF for baselines.  
SRF  generates the frequency samples $\mat{\omega}_1, \dots, \mat{\omega}_{D/2}$ in \cref{eq:RFF}  from an optimized spectral density, by first drawing samples from the unit sphere in $\mathbb{R}^d$.   
Therefore, by replacing these samples on the unit sphere by structured orthogonal projections of SORF \citep{Yu2016}, one can construct a structured version of SRF.  We use this structured SRF as another baseline (``SRF Orth.''). 
Moreover, we consider a complex extension of the structured SRF in the form of  \cref{eq:complex-RFF}  ( ``SRF~Orth.~Comp.'').
While these extensions are themselves novel, we include them in the experiments, as they improve over the vanilla SRF and make the experiments more competitive. 
Finally, we also include the method in \citet{Ahle2020} and its complex version.

\cref{fig:poly-classification-revised-codrna-magic} and \cref{fig:poly-classification-revised-mocap-fmnist} show the results of approximate GP classification on four datasets from \cref{tbl:datasets}. 
We present the results on the other four datasets as well as the results of GP regression in \cref{sec:additional-experiments} to save space.
We can make the following observations from these results.

\paragraph{Relative Frobenius norm error.}
For most cases, the optimized Maclaurin approaches with TensorSRHT achieve lower relative Frobenius norm errors than the SRF approaches and the method by \citet{Ahle2020}. 
Across all methods, there are cases where the improvements offered by the complex approach are quite large.

\paragraph{KL divergence.}
While the optimized Maclaurin approaches achieve lower KL divergences than the SRF approaches for most cases, the margins are smaller than those for the relative Frobenius norm errors.
One possible reason is that the Maclaurin approaches in general (either random or optimized) can be inaccurate in approximating the GP posterior variances at test inputs far from $\mat{x} = \mat{0}$, as discussed in Section \ref{sec:GP-toy-experiment}.
While we suggested a way of fixing this issue in Section \ref{sec:GP-toy-experiment}, we do not implement it to conduct a direct comparison with the SRF approaches.

\paragraph{Classification errors and mean negative log likelihood (MNLL).}
The optimized Maclaurin approaches with TensorSRHT achieve equal or lower classification errors and MNLL than the SRF approaches. 
These results suggest that the optimized Maclaurin approaches are promising not only in kernel approximation accuracy but also in downstream task performance. 
Recall that we selected the regularization parameter in GP classification by maximizing the MNLL of SRF (on the validation set), and used the same regularization parameter in the other approaches (See Section \ref{sec:other-settigns}).
Therefore, the results of \cref{fig:poly-classification-revised-codrna-magic} and \cref{fig:poly-classification-revised-mocap-fmnist} are in favor of the SRF approaches, and the optimized Maclaurin approaches may perform even better if we choose the regularization parameter for them separately.

\begin{figure}[t]
    \centering
    \includegraphics[width=1\textwidth]{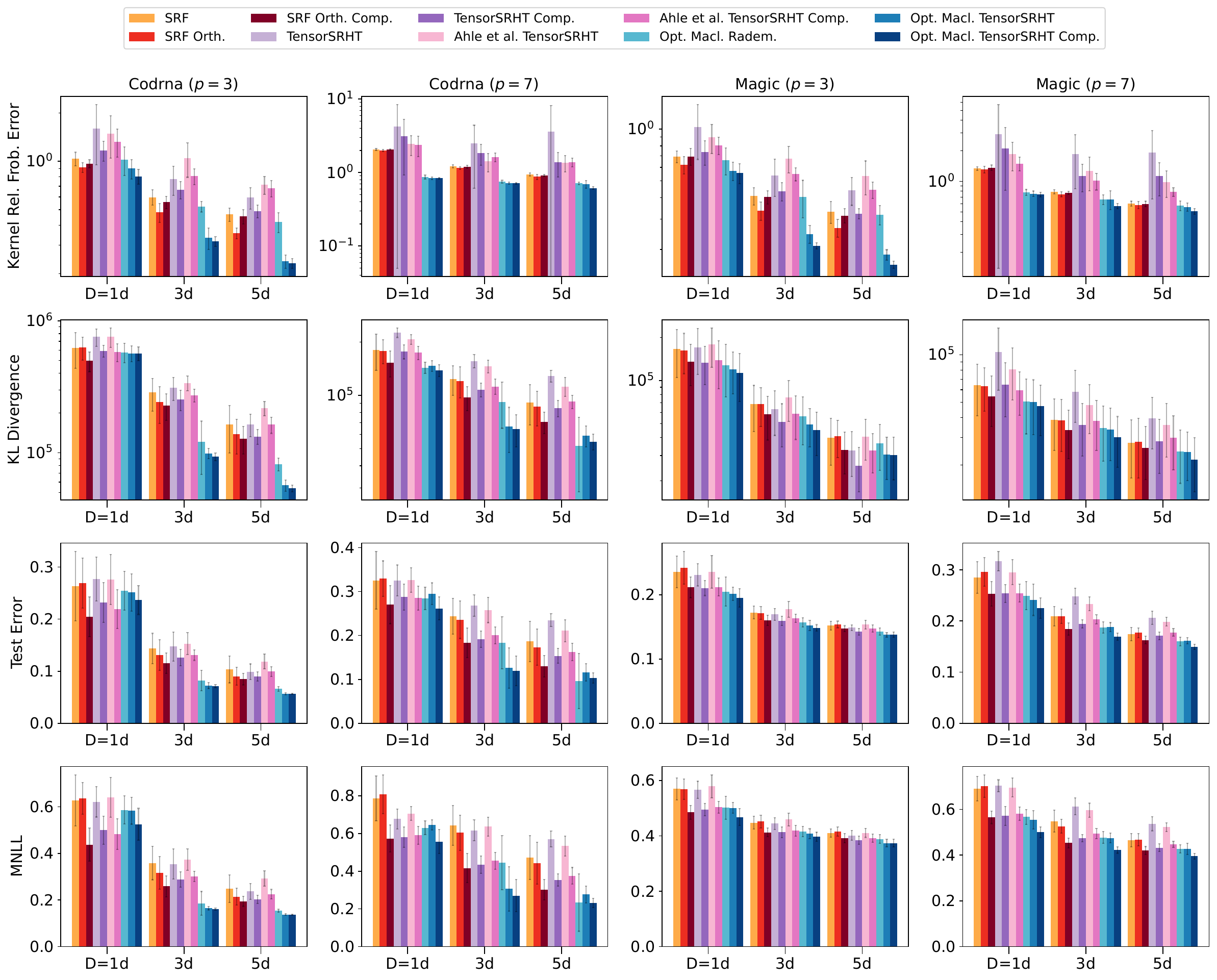}
    \caption{
    Codrna and Magic results of the experiments in Section \ref{sec:GP-inference-poly} on approximate GP classification with polynomial kernels of degree $p=3$ and $p=7$.
    Lower values are better for all the metrics. 
    For each dataset, we show the number of random features $D \in \{1d, 3d, 5d\}$ used in each method on the horizontal axis, with $d$ being the input dimensionality of the dataset.
    }
    \label{fig:poly-classification-revised-codrna-magic}
\end{figure}

\begin{figure}[t]
    \centering
    \includegraphics[width=1\textwidth]{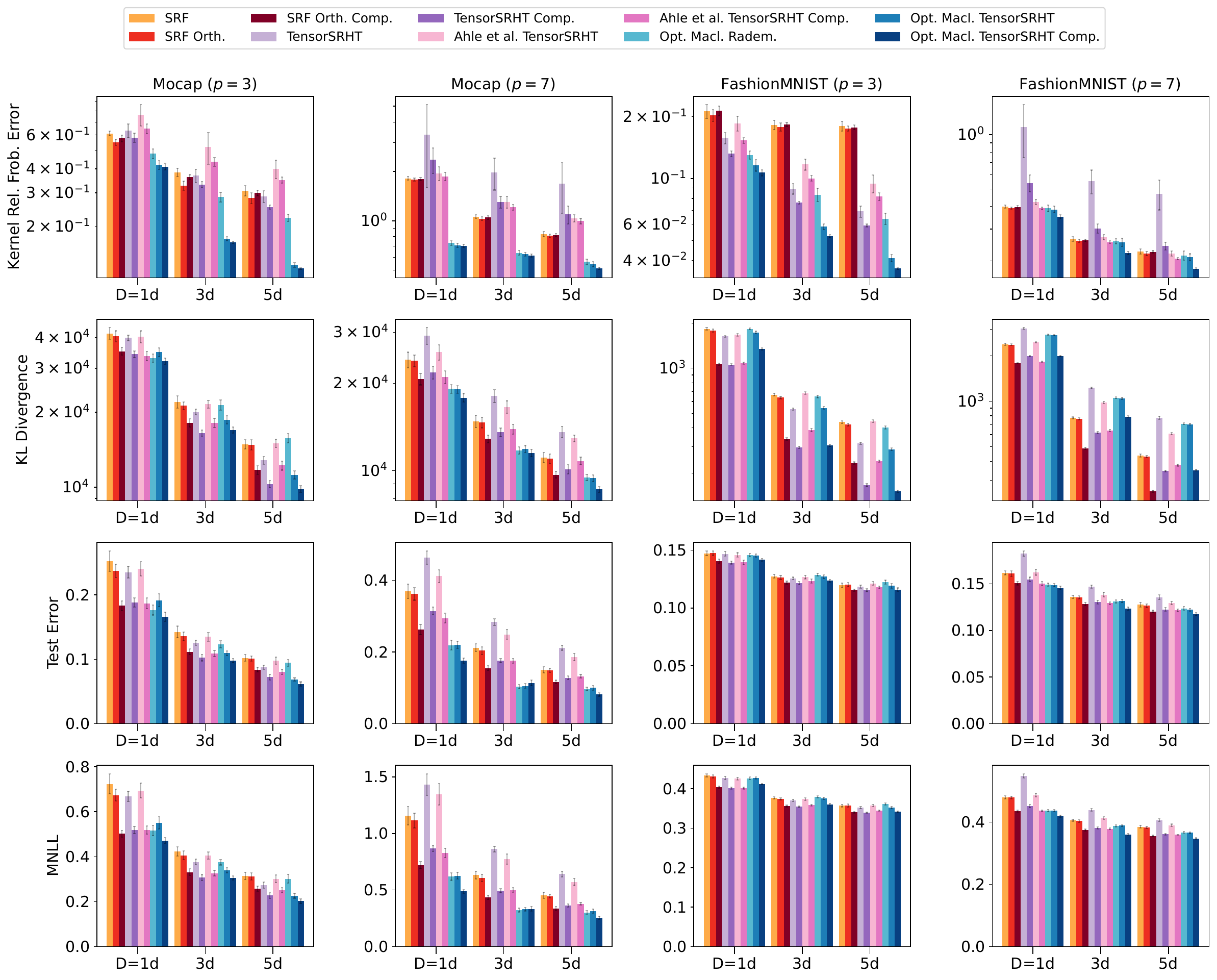}
    \caption{
Mocap and FashionMNIST results of the experiments in Section \ref{sec:GP-inference-poly} on approximate GP classification with polynomial kernels of degree $p=3$ and $p=7$.
    Lower values are better for all the metrics. 
    For each dataset, we show the number of random features $D \in \{1d, 3d, 5d\}$ used in each method on the horizontal axis, with $d$ being the input dimensionality of the dataset.
}
    \label{fig:poly-classification-revised-mocap-fmnist}
\end{figure}

\subsubsection{Approximate GP Inference with a Gaussian kernel}
\label{sec:gaussian-approximation}

\begin{figure}[t]
    \centering
    \includegraphics[width=1\textwidth]{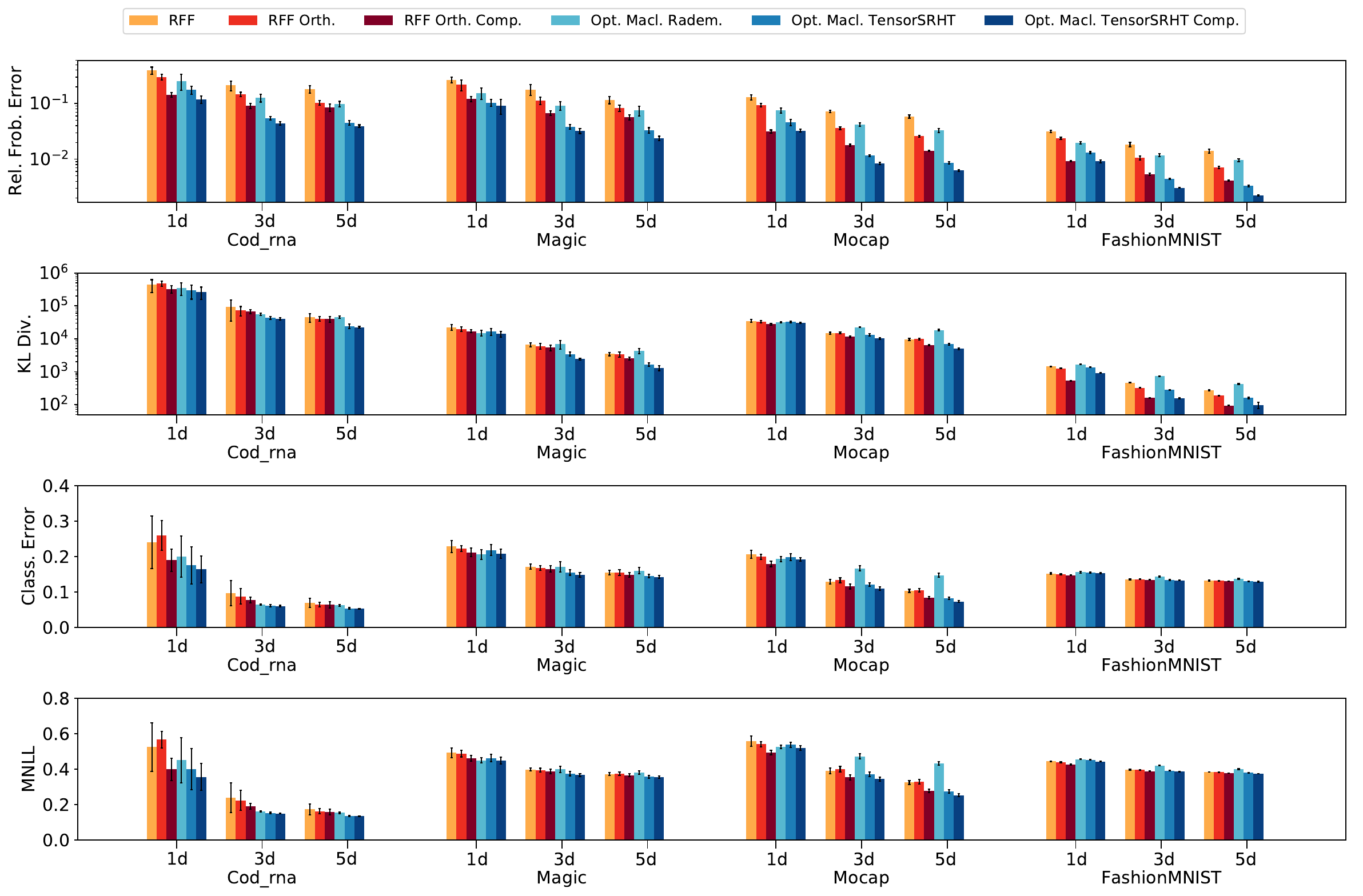}
    \caption{Results of the experiments in Section \ref{sec:gaussian-approximation} on approximate GP classification with a Gaussian kernel.
        Lower values are better for all the metrics. 
    For each dataset, we show the number of random features $D \in \{1d, 3d, 5d\}$ used in each method on the horizontal axis, with $d$ being the input dimensionality of the dataset. We put the legend labels and the bars in the same order.
    }
    \label{fig:rbf-classification}
\end{figure}

We next consider GP classification using a Gaussian kernel. 
As in Section \ref{sec:GP-inference-poly}, we apply zero-centring to the input vectors of each dataset.

\paragraph{Baselines.} 
We use RFF, SORF (``RFF Orth.'') and a complex extension of SORF (``RFF Orth.\ Comp.'') as baselines (see the beginning of Section \ref{sec:exp-opt-Macl} for details).  
SORF  is a state-of-the-art approach to approximating a Gaussian kernel \citep[e.g.][]{Choromanski2018}.
As in Section \ref{sec:GP-inference-poly}, we consider its complex extension to make the experiments more competitive.

\paragraph{Results.}
\cref{fig:rbf-classification} summarizes the results on four datasets  from \cref{tbl:datasets}. 
We show the results on the rest of the datasets as well as the results of GP regression in \cref{sec:additional-experiments}. 
We can make similar observations for \cref{fig:rbf-classification}  as for the polynomial kernel experiments in Section \ref{sec:GP-inference-poly} (and thus we omit explaining them).
The results suggest the effectiveness of the optimized Maclaurin approach with TensorSRHT in approximating the Gaussian kernel.

\subsubsection{Influence of the Data Distribution on the Kernel Approximation}
\label{sec:influence-data-dist}
\begin{figure}[t]
    \centering
    \includegraphics[width=0.95\textwidth]{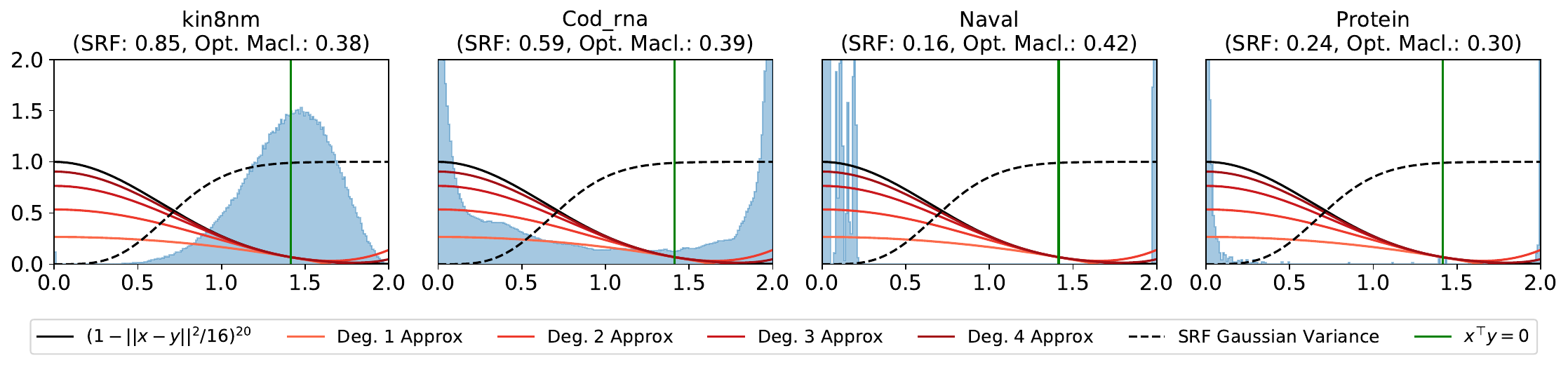}
    \caption{Histograms of pairwise Euclidean distances $\{ \| \mat{x}_{*,i} - \mat{x}_{*,j} \| \}_{i \not= j} $  for test subsets of four datasets (Section \ref{sec:influence-data-dist}).  
 On the top of each figure, we show the relative Frobenius norm errors \eqref{eqn:rel-frob-error} of the optimized Maclaurin approach with real TensorSRHT and of  SRF with structured orthogonal projections.
 The black curve represents the polynomial kernel in  \cref{eqn:poly-target-kernel} with $p=20$ as a function of  $\tau := \|\mat{x} - \mat{y}\|$ (the horizontal axis); the orange curves describe its degree $n \in \{1, 2, 3, 4\}$ approximations (i.e., the truncation of the  Maclaurin expansion \eqref{eq:poly-ker-expansion} of the polynomial kernel up to the $n$-th degree terms.). The dashed curve represents the variance of the SRF approximation as a function of $\tau = \|\mat{x} - \mat{y}\|$.
 The green vertical line shows the value of $\tau = \|\mat{x} - \mat{y}\| = \sqrt{2}$ for which the input vectors $\mat{x}, \mat{y}$ are orthogonal, $\mat{x}^\top \mat{y} = 0$.
 }
    \label{fig:data-distributions-srf}
\end{figure}
 
Lastly, we investigate a characterization of datasets for which the optimized Maclaurin approach performs well.
We focus on polynomial kernel approximation, and make a comparison with SRF as in Section \ref{sec:GP-inference-poly}.

\cref{fig:data-distributions-srf} describes a histogram of pairwise distances $\{ \| \mat{x}_{*, i} - \mat{x}_{*, j} \| \}_{i \not= j}$ of the input vectors in a test subset $X_{*, {\rm sub}} = \{ \mat{x}_{*, 1}, \dots, \mat{x}_{*, m_*} \}$, obtained after zero-centering and unit-normalization, of each of four representative datasets (kin8nm, Cod\_rna, Naval, and Protein). 
For these datasets, the optimized Macluarin approach and SRF show stark contrasts in their performances; see Section \cref{sec:GP-inference-poly} and  \cref{sec:additional-experiments}.   
Note that the polynomial kernel in \cref{eqn:poly-target-kernel} is a shift-invariant kernel on the unit sphere of $\mathbb{R}^d$, and thus its value depends only on the distance $\tau := \|\mat{x} - \mat{y}\|$ between the input vectors $\mat{x}, \mat{y}$ as long as $\norm{x}=\norm{y}=1$. 
This motivates us to study here the distribution of pairwise distances and its effects on approximating the polynomial kernel in  \cref{eqn:poly-target-kernel}.

In \cref{fig:data-distributions-srf}, the optimized Maclaurin approach yields lower relative Frobenius norm errors \eqref{eqn:rel-frob-error} than SRF for the left two plots, while the optimized Maclaurin approach is less accurate than SRF for the right two plots. 
For the datasets of the right two plots (Naval and Protein),  the pairwise distances $\{ \| \mat{x}_{*, i} - \mat{x}_{*, j} \| \}_{i \not= j}$ concentrate around $\tau = 0$ (and there is a smaller mass around $\tau =2$).
In comparison, for the datasets of the left two plots (kin8nm and Cod\_rna), the pairwise distances  are relatively more evenly distributed across the possible range  $\tau \in [0,2]$.

The above observation suggests that the optimized Maclaurin approach is more suitable for datasets in which the pairwise distances $\{ \| \mat{x}_{*, i} - \mat{x}_{*, j} \| \}_{i \not= j}$ are not concentrating around $0$, i.e., datasets in which there is a diversity in the input vectors $\{\mat{x}_{*,1}, \dots, \mat{x}_{*,m_*} \}$. 
In fact, for approximating the polynomial kernel (black curve in \cref{fig:data-distributions-srf}), the finite-degree Maclaurin approximations (orange curves) tend to be less accurate for input vectors $\mat{x}, \mat{y}$ close to each other, $\tau = \|\mat{x} - \mat{y} \| \approx 0$, and become relatively more accurate as input vectors $\mat{x}, \mat{y}$  approach orthogonality, i.e. $\mat{x}^\top \mat{y} = 0$ (or $\tau = \| \mat{x} - \mat{y} \| = \sqrt{2}$; the vertical green line); see also the Maclaurin expansion \eqref{eq:poly-ker-expansion} of the polynomial kernel.
On the other hand, the variance of SRF is the lowest around $\tau = \| \mat{x} - \mat{y} \| = 0$ and increases as $\tau$ tends to $2$. 
Therefore, the SRF performs well if the pairwise distances $\{ \| \mat{x}_{*, i} - \mat{x}_{*, j} \| \}_{i \not= j}$ concentrate around $0$, and may become inaccurate if they do not.

\section{Conclusion}

We made several contributions for understanding and improving random feature approximations for dot product kernels. 
First, we studied polynomial sketches, i.e., random features for polynomial kernels, such as the Rademacher sketch and TensorSRHT, and discussed their generalizations using complex-valued features. 
We derived closed form expressions for the variances of these polynomial sketches, which are useful in both theory and practice.

On the theoretical side, these variance formulas provide novel insights into these polynomial sketches, such as conditions for a structured sketch to have a lower variance than the corresponding unstructured sketch, and conditions for a complex sketch to have a lower variance than the corresponding real sketch.
Our systematic experiments support these findings.
On the practical side, these variance formulas can be evaluated in practice, and therefore enable us to estimate the mean squared errors of the approximate kernel for given input points.

Based on the derived variance formulas, we developed a novel optimization algorithm for data-driven random feature approximations of dot product kernels, which is also applicable to the Gaussian kernel.
This approach uses a finite Maclaurin approximation of the kernel, which approximates the kernel as a finite sum of polynomial kernels of different degrees. 
Given a total number of random features, our optimization algorithm determines how many random features should be used for each polynomial degree in the Maclaurin approximation.
We defined the objective function of this optimization algorithm as an estimate of the averaged mean squared error regarding the data distribution, and used the variance formulas for this purpose.
We empirically demonstrated that this optimized Maclaurin approach achieves state-of-the-art performance on a variety of datasets, both in terms of the kernel approximation accuracy and downstream task performance.

As described in the introduction, dot product kernels have been actively used in many domains of applications, such as genomic data analysis, recommender systems, computer vision, and natural language processing.
In these applications, interactions among input variables have significant effects on the output variables of interest, and thus dot product kernels offer an appropriate modeling tool.
In particular, dot product kernels are being used in an inner-loop of larger neural network models, such as the dot product attention mechanism used in Transformer architectures \citep{Vaswani2017,Choromanski21a}. 

One major challenge of using dot product kernels is the computational efficiency, and random feature approximations offer a promising solution. 
Our contributions improve the efficiency of random feature approximations, and we hope that these contributions make dot product kernels even more useful in the above application domains.

\acks{The Authors wish to thank the Associate Editor and the anonymous Reviewers for their insightful reviews and comments.
This work has been in part supported by the French government through the 3IA Cote d’Azur Investment in the Future Project, which is managed by the National Research Agency (ANR) and has the reference number ANR-19-P3IA-0002.
MF also gratefully acknowledges support from the AXA Research Fund and ANR (grant ANR-18-CE46-0002).
}

\appendix

\newpage

\section{Proofs for Section \ref{sec:polynomial-sketches}}

\subsection{Proof of Theorem \ref{thm:expression-var-complex-aprox-kernel}}
\label{sec:appendix-complex-variance}

We first show 
\begin{align} \label{eq:variance-expression-complex}
\mathbb{V}[ \hat{k}_\mathcal{C}(\mat{x}, \mat{y}) ] &=  \bigg( \sum_{k=1}^d \mathbb{E}[|z_k|^4] x_k^2 y_k^2 +  \| \mat{x} \|^2 \| \mat{y} \|^2 - 2 \sum_{k=1}^d x_k^2 y_k^2   +    (\dotprod{x}{y})^2   \nonumber  \\ 
& \quad \quad + \sum_{i=1}^d \sum_{\substack{j=1 \\ j \neq i}}^d \mathbb{E} [z_i^2] \mathbb{E}[ \overline{z_j}^2] x_i x_j y_i y_j \bigg)^p - (\mat{x}^\top \mat{y})^{2p}  .
\end{align}
where $z_i^2 := z_i z_i$ and $\overline{z_i}^2 := \overline{z_i} \overline{z_i}$ are in general different from $|z_i|^2 = z_i \overline{z_i}$.
We have
\begin{align} 
     \mathbb{V} [ \hat{k}_\mathcal{C}(\mat{x}, \mat{y})  ] &  = \mathbb{E}[ | \hat{k}_\mathcal{C}(\mat{x}, \mat{y})|^2 ] -  | \mathbb{E}[  \hat{k}_\mathcal{C}(\mat{x}, \mat{y} )  ]|^2   = \mathbb{E}[ | \prod_{i=1}^p \mat{z}_i^\top \mat{x} \overline{\mat{z}_i^\top \mat{y}}  |^2 ] - ( \mat{x}^\top \mat{y} )^{2p} \nonumber \\
    & = \prod_{i=1}^p \mathbb{E}[ | \mat{z}_i^\top \mat{x} \overline{\mat{z}_i^\top \mat{y}}  |^2 ] - ( \mat{x}^\top \mat{y} )^{2p} =  ( \mathbb{E}[ | \mat{z}^\top \mat{x} \overline{\mat{z}}^\top \mat{y}  |^2 ] )^p - ( \mat{x}^\top \mat{y} )^{2p}.  \label{eq:proof-var-complex-expression}
\end{align}

Henceforth we focus on $ \mathbb{E}[ | \mat{z}^\top \mat{x} \overline{\mat{z}}^\top \mat{y} |^2 ]$ in the last expression \eqref{eq:proof-var-complex-expression}.
Write $\mat{z}= (z_1, \dots, z_d)^\top$, $\mat{x} = (x_1, \dots, x_d)^\top$, and $\mat{y} = (y_1, \dots, y_d)^\top$. 
Since $\mathbb{E}[\mat{z} \overline{\mat{z}}^\top] = \mat{I}_d$, we have $\mathbb{E}[z_i \overline{z_j}] = 1$ if $i = j$ and  $\mathbb{E}[z_i \overline{z_j}] = 0$ if $i \not= j$.
Recall also that $z_1, \dots, z_d \in \mathbb{C}$ are i.i.d, and $\mathbb{E}[z_i] = 0$ for $i = 1,\dots,d$.
Then
\begin{align}
    &  \mathbb{E}[ | \mat{z}^\top \mat{x} \overline{\mat{z}}^\top \mat{y}  |^2 ] = \mathbb{E}\left[  (\sum_{i=1}^d z_i x_i)  (\sum_{j=1}^d \overline{z_j} y_j) ( \sum_{k=1}^d\overline{z_k} x_k ) ( \sum_{l = 1}^d z_l y_l ) \right] \nonumber \\
     & = \sum_{i=1}^d \sum_{j=1}^d \sum_{k=1}^d \sum_{l=1}^d \mathbb{E} [z_i\overline{z_j} \overline{z_k} z_l] x_i y_j x_k y_l . \label{eq:complex-weights-expansion}
\end{align}
The expected value $\mathbb{E} \big[ z_i\overline{z_j}\overline{z_k} z_l \big]$ is different from $0$, only if:
\begin{enumerate}[label={(\alph*)}]
\item $i=j=k=l$, for which there are $d$ terms and $\mathbb{E} \big[ z_i\overline{z_j}\overline{z_k} z_l \big]  x_i y_j x_k y_l = \mathbb{E}[| z_i |^4]  x_i^2 y_i^2$.

\item $i=j \neq k=l$, for which there are $d(d-1)$ terms and $\mathbb{E} \big[ z_i\overline{z_j}\overline{z_k} z_l \big] x_i y_j x_k y_l = \mathbb{E}[ |z_i|^2] \mathbb{E}[|z_k|^2 ]  x_i x_k y_i y_k$.

\item $i=k \neq j=l$, for which there are $d(d-1)$ terms and $\mathbb{E} \big[ z_i\overline{z_j}\overline{z_k} z_l \big] x_i y_j x_k y_l = \mathbb{E}[ |z_i|^2] \mathbb{E}[ |z_j|^2 ] x_i^2 y_j^2$.

\item $i=l \neq j=k $, for which there are $d(d-1)$ terms and $\mathbb{E} \big[ z_i\overline{z_j}\overline{z_k} z_l \big] x_i y_j x_k y_l  = \mathbb{E}[ z_i^2  ] \mathbb{E}[\overline{z_j}^2  ] x_i x_j y_i y_j$.
\end{enumerate}
Therefore,
\begin{align*}
    \eqref{eq:complex-weights-expansion}
    &= \underbrace{\sum_{i=1}^d \mathbb{E}[|z_i|^4] x_i^2 y_i^2}_{\text{case (a)}} + \underbrace{\sum_{i=1}^d \sum_{\substack{j=1 \\ j \neq i}}^d \mathbb{E}[|z_i|^2] \mathbb{E}[|z_j|^2] x_i^2 y_j^2}_{\text{case (c)}} + \underbrace{ \sum_{i=1}^d \sum_{\substack{j=1 \\ j \neq i}}^d \mathbb{E}[|z_i|^2] \mathbb{E}[|z_j|^2] x_i x_j y_i y_j}_{\text{case (b)}} \\
    & 
    + \underbrace{
     \sum_{i=1}^d \sum_{\substack{j=1 \\ j \neq i}}^d \mathbb{E} [z_i^2] \mathbb{E}[ \overline{z_j}^2] x_i x_j y_i y_j  }_{\text{case (d)}} \\
    & = \sum_{i=1}^d \mathbb{E}[|z_i|^4] x_i^2 y_i^2 + \sum_{i=1}^d \sum_{\substack{j=1 \\ j \neq i}}^d x_i^2 y_j^2 +  \sum_{i=1}^d \sum_{\substack{j=1 \\ j \neq i}}^d x_i x_j y_i y_j + \sum_{i=1}^d \sum_{\substack{j=1 \\ j \neq i}}^d \mathbb{E} [z_i^2] \mathbb{E}[ \overline{z_j}^2] x_i x_j y_i y_j \\
    &= \sum_{i=1}^d \mathbb{E}[|z_i|^4] x_i^2 y_i^2 + \left[ \| \mat{x} \|^2 \| \mat{y} \|^2 - \sum_{i=1}^d x_i^2 y_i^2 \right] +  \left[ (\dotprod{x}{y})^2 - \sum_{i=1}^d x_i^2 y_i^2 \right]   \\
    & \quad + \sum_{i=1}^d \sum_{\substack{j=1 \\ j \neq i}}^d \mathbb{E} [z_i^2] \mathbb{E}[ \overline{z_j}^2] x_i x_j y_i y_j 
\end{align*}
The proof of \cref{eq:variance-expression-complex} completes by using this expression of $ \mathbb{E}[ ( \mat{z}^\top \mat{x} \overline{\mat{z}^\top \mat{y}}  )^2 ]$ in  \cref{eq:proof-var-complex-expression}.
\cref{eq:complex-poly-var-coro-1} follows from \cref{eq:variance-expression-complex} and $\mathbb{E}[z_k^2] =\mathbb{E}[\overline{z_k}^2] = 2q-1$, which uses \cref{eq:coro-complex-var-condition}.

\subsection{Proof of Theorem \ref{trm:real-bernstein-bound}}
\label{sec:appendix-rademacher-bound}

We make use of Bernstein's inequality \citep[e.g.,][Theorem 2.8.4]{vershynin2018high}:
For independent random variables $X_1, \dots, X_D \in \mathbb{R}$  such that $\mathbb{E}[X_i] = 0$ and $|X_i| \leq R$ almost surely for a constant $R > 0$, we have for any $t > 0$:
\begin{align}
    \label{eqn:bernstein}
    \mathrm{Pr}\left[\left| \sum_{i=1}^D X_i \right| \geq t \right]
    \leq 2 \exp \left( \frac{- t^2 / 2}{\sum_{i=1}^D \mathbb{V} \left[ X_i \right] + R t / 3} \right)
\end{align}
We define $X_i := \Phi_{\mathcal{C}} (\mat{x})_i \overline{\Phi_{\mathcal{C}} (\mat{y})_i} - (\dotprod{x}{y})^p / D \in \mathbb{R}$, where  $\Phi_\mathcal{C} (\mat{x}) \in \mathbb{C}^D$ is defined in \cref{eqn:comp-rad-polynomial-sketch}: $\Phi_{\mathcal{C}}(\mat{x})
    =  \frac{1}{\sqrt{D}} \left[ (\prod_{i=1}^p \dotprodi{z}{x}{i,1}), \dots, (\prod_{i=1}^p \dotprodi{z}{x}{i,D}) \right]^{\top}$.
Then we have $\mathbb{E}[X_i] = 0$.
Moreover,  
\begin{align*}
    |X_i| &\leq |  \Phi_{\mathcal{C}} (\mat{x})_i \overline{\Phi_{\mathcal{C}} (\mat{y})_i} | + |(\dotprod{x}{y})^p / D| = \frac{1}{D} \left( \prod_{j=1}^p | \dotprodi{z}{x}{j}| | \overline{\dotprodi{z}{y}{j}} | + |(\dotprod{x}{y})^p| \right)   \\
    &\leq \frac{1}{D} ( \norm{x}_1^p \norm{y}_1^p + \norm{x}_2^p \norm{y}_2^p)  \leq \frac{2}{D} \norm{x}_1^p \norm{y}_1^p =: R
\end{align*}
where the first inequality is the triangle inequality. The second inequality uses H\"older's inequality (and that the absolute value of each element of $\mat{z}_j$ is $1$) as well as the upper bound $\dotprod{x}{y} \leq \|\mat{x}\|_2 \|\mat{y}\|_2$.
Furthermore, by assumption we have 
$$
\mathbb{V} [X_i] = \frac{\sigma^2 \norm{x}_2^{2p} \norm{y}_2^{2p} }{D^2}  \leq \frac{\sigma^2 \norm{x}_1^{2p} \norm{y}_1^{2p} }{D^2}
$$
for some $\sigma^2 \geq 0$.
Therefore, using \cref{eqn:bernstein} and setting $t := \norm{x}_1^p \norm{y}_1^p \epsilon$, we have
\begin{equation*}
    \mathrm{Pr}\left[\left| \sum_{i=1}^D X_i \right| \geq \epsilon \norm{x}_1^p \norm{y}_1^p \right]
    \leq 2 \exp \left( \frac{- D \epsilon^{2} / 2}{\frac{2}{3} \epsilon + \sigma^2} \right)
\end{equation*}
Setting $D \geq 2(\frac{2}{3\epsilon}  + \frac{\sigma^2}{\epsilon^2} ) \log ( \frac{2}{\delta})$ and taking the complementary probability gives the desired result.

\section{Proofs for Section \ref{sec:structured-projections}}

\subsection{Key Lemma}
First, we state a key lemma that is needed for deriving the variance of real and complex TensorSRHT.
This result is essentially given in \citet[Proof of Proposition 8.2]{Choromanski2017}.
However, their proof contains a typo missing the negative sign, and they use a different definition of the Hadamard matrix from ours. 
Therefore, for completeness, we state the result formally and provide a proof. 

\begin{lemma} \label{lemma:key-lemma-hadamard}
Let $d = 2^m$ for some $m \in \mathbb{N}$ and  $\mat{H}_d = (\mat{h}_1, \dots, \mat{h}_d) \in \{1, - 1\}^{d \times d}$ be the unnormalized Hadamard matrix defined in \cref{eq:hadamard}, where $\mat{h}_\ell = ( h_{\ell, 1}, \dots, h_{\ell, d} )^\top  \in \{ 1, -1 \}^d$ for $\ell \in \{1,\dots,d\}$.
Let $\pi: \{1,\dots, d\} \to \{1,\dots,d\}$ be a uniformly random permutation. 
Then for any $\ell, \ell' \in \{ 1,\dots, d \}$ with $\ell \not= \ell'$ and $t, u \in \{ 1, \dots, d\}$ with $t \not= u$, we have
$$
\mathbb{E} [h_{\pi(\ell), t}  h_{\pi(\ell'), t}  h_{\pi(\ell) ,u} h_{\pi(\ell'), u}] = -\frac{1}{d-1},
$$
where the expectation is with respect to the random permutation $\pi$.
\end{lemma}

\begin{proof}
We first derive a few key identities needed for our proof. 
For simplicity of notation, define
$$
\alpha_\ell := h_{\ell, t} h_{\ell, u}, \quad \ell \in \{ 1, \dots, d\}.
$$
Since any two distinct rows (and any two distinct columns) of $\mat{H}_d$ are orthogonal, we have 
$$
\sum_{\ell = 1}^d \alpha_\ell = \sum_{\ell=1}^d h_{\ell, t} h_{\ell, u} = 0.
$$
Since $\alpha_\ell \in \{-1, 1\}$, this identity implies that exactly $d/2$ elements in $\{ \alpha_1, \dots, \alpha_d \}$ are $1$, and the rest  are $-1$. 
Note that for each $\ell \in \{1,\dots, d \}$ the randomly permuted index $\pi(\ell)$ takes values in $\{1, \dots, d\}$ with equal probabilities. 
Therefore, the probability of $\alpha_{\pi(\ell)}$ being $1$ and that of $\alpha_{\pi(\ell)}$ being $-1$ are equal:
\begin{align*}
    {\rm Pr}( \alpha_{\pi(\ell)} = 1)
    = {\rm Pr}( \alpha_{\pi(\ell)} = -1)
    = 0.5.
\end{align*}
Note that $\pi^b(\ell) \not= \pi^b(\ell')$ since $\ell \not= \ell'$ and $\pi$ is a (random) permutation.  
Therefore, we have the following conditional probabilities:
$$
{\rm Pr}(\alpha_{\pi(\ell')} = a \mid \alpha_{\pi(\ell)} = b ) = 
\begin{cases}
\frac{d/2-1}{d-1} \quad \text{ if } a = b = 1 \text{ or } a = b = -1  \\
\frac{d/2}{d-1} \quad \text{ if } a = 1, b = -1 \text{ or } a = -1, b = -1
\end{cases}
$$

Using the above identities, we now prove the assertion:
\begin{align*}
    & \mathbb{E} [h_{\pi^b(\ell), t} h_{\pi^b(\ell'), t} h_{\pi^b(\ell), u} h_{\pi^b(\ell'), u}] = \mathbb{E} [ \alpha_{\pi(\ell)} \alpha_{\pi(\ell')} ] \\
    & = {\rm Pr}( \alpha_{\pi(\ell)} = 1 ) \mathbb{E}[ \alpha_{\pi(\ell)} \alpha_{\pi(\ell')} \mid \alpha_{\pi(\ell)} = 1 ] + {\rm Pr}( \alpha_{\pi(\ell)} = -1 ) \mathbb{E}[ \alpha_{\pi(\ell)} \alpha_{\pi(\ell')} \mid \alpha_{\pi(\ell)} = -1 ]  \\
    & = \frac{1}{2} \mathbb{E}[  \alpha_{\pi(\ell') } \mid \alpha_{\pi(\ell)} = 1 ] - \frac{1}{2} \mathbb{E}[  \alpha_{\pi(\ell') } \mid \alpha_{\pi(\ell)} = -1 ] \\
    &  = \frac{1}{2} \left( \frac{d/2-1}{d-1} - \frac{d/2}{d-1} \right)
    - \frac{1}{2} \left(\frac{d/2}{d-1} - \frac{d/2-1}{d-1} \right)  = -\frac{1}{d-1}.
\end{align*}
 
\end{proof}

\subsection{Proof of Theorem \ref{thm:variance-complexTensorSRHT}}

\label{sec:proof-var-comp-TensorSRHT}

We first clarify the notation we use.  
Recall that our  feature map $\Phi(\mat{x}) \in \mathbb{C}^D$ is given by 
\begin{align*}
    \Phi (\mat{x}) = \frac{1}{\sqrt{D}} \left[ (\prod_{i=1}^p \mat{s}_{i, 1}^{\top} \mat{x}), \dots, (\prod_{i=1}^p \mat{s}_{i, D}^{\top} \mat{x}) \right]^{\top} \in \mathbb{C}^D.
\end{align*}
The random vectors $\mat{s}_{i, \ell} \in \mathbb{C}^d$ are independently generated blockwise, and there are $B :=  \lceil D/d \rceil$ blocks in total (and note that $D = (B-1)d + {\rm mod}(D,d)$): 
For each $i = 1, \dots, p$,
\begin{align*}
  &  \underbrace{(\mat{s}_{i,1}, \dots,\mat{s}_{i,d})}_{\text{Block }1},\  \underbrace{ (\mat{s}_{i,d+1}, \dots,\mat{s}_{i,2d}) }_{\text{Block }2},\ \dots,
 \\  
  & \quad \quad \quad \quad  \underbrace{(\mat{s}_{i,(B-2)d + 1}, \dots, \mat{s}_{i, (B-1)d} )}_{\text{Block }B-1},\ \underbrace{(\mat{s}_{i,(B-1)d + 1}, \dots, \mat{s}_{i, (B-1)d + {\rm mod}(D,d)} )}_{\text{Block }B} \\
  &=:  \underbrace{(\mat{s}^1_{i,1}, \dots,\mat{s}^1_{i,d})}_{\text{Block }1},\  \underbrace{ (\mat{s}^2_{i,1}, \dots,\mat{s}^2_{i,d}) }_{\text{Block }2},\ \dots,\
  \underbrace{(\mat{s}^{B-1}_{i,1}, \dots, \mat{s}^{B-1}_{i,d})}_{\text{Block }B-1}, \
  \underbrace{(\mat{s}^B_{i,1}, \dots, \mat{s}^B_{i,{\rm mod}(D,d)})}_{\text{Block }B},
\end{align*}
where we introduced in the second line a new notation:
$$
\mat{s}^b_{i, \ell} := \mat{s}_{i, (b-1)d + \ell} \quad (b = 1,\dots,B, \ \ell = 1,\dots, d.).
$$
Here $b$ serves as the indicator of the $b$-th block. Thus, using this notation,
$$ 
\mat{s}^b_{i, \ell} = \mat{z}_i^b \circ \mat{h}_{\pi^b(\ell)} \in \mathbb{C}^d \quad (\ell = 1,\dots,d), 
$$ 
where $\mat{z}_{i}^b = (z_{i,1}^b, \dots, z_{i,d}^b)^\top \in \mathbb{C}^d$ is a random vector whose elements $z_{i,1}^b, \dots, z_{i,d}^b$ are i.i.d., and $\pi^b: \{1,\dots, d\} \to \{1,\dots,d\}$ is a random permutation of the indices.
Note that $\mat{z}_{i}^b$ and $\pi^b$ are generated independently for each $b \in \{1, \dots, B\}$.
Therefore,  the random vectors $\mat{s}^b_{i, \ell}$ and $\mat{s}^{b'}_{i, \ell'}$ are statistically independent if they are from different blocks, i.e., if $b \not= b'$.

For each $b=1,\dots,B$, define $\mat{z}^b = (z^b_1, \dots, z^b_d)^\top \in \mathbb{C}^d$ as a random vector independently and identically distributed as $ \mat{z}^b_1, \dots, \mat{z}^b_p$. 
Define
\begin{align} \label{eq:def-s-vector-proof}
    \mat{s}^b_\ell := \mat{z}^b \circ \mat{h}^b_{\pi(\ell)} 
    &= (z^b_1 h_{\pi^b(\ell),1}, \dots,  z^b_d h_{\pi^b(\ell),d})^\top =: (s^b_{\ell,1}, \dots, s^b_{\ell, d})^\top \in \mathbb{C}^d.  
\end{align}
Then $\mat{s}^b_\ell$ is independently and identically distributed as $\mat{s}_{1, \ell}^b, \dots, \mat{s}_{p, \ell}^b$. 
Moreover, given the permutation $\pi^b$ fixed, $\mat{s}^b_\ell$ is identically distributed as $\mat{z}^b$. 
This is because 1) $z^b_1, \dots, z^b_d$ are i.i.d., 2) each $z_t^b$ is symmetrically distributed ($t=1,\dots,d$), and 3) $h_{\pi^b(\ell),1} ,\dots, h_{\pi^b(\ell),d} \in \{1, -1 \}$.

Now let us start proving the assertion.   We first have
\begin{align*}
    \mathbb{V}[ \hat{k}( \mat{x}, \mat{y} )] =     \mathbb{E} [ | \hat{k}(\mat{x}, \mat{y}) |^2] -  |\mathbb{E}[\hat{k}( \mat{x}, \mat{y} ) ]   |^2 =     \mathbb{E} [ | \hat{k}(\mat{x}, \mat{y}) |^2] - ( \mat{x}^\top \mat{y} )^{2p},
\end{align*}
where the second identity follows from the approximate kernel being unbiased for both real and complex TensorSRHT.
Thus, from now on we study the term $\mathbb{E} [| \hat{k}(\mat{x}, \mat{y}) |^2]$.

For simplicity of notation, define $I_b := \{1,\dots, d\}$ for $b = 1,\dots, B-1$ and $I_b := \{ 1, \dots, {\rm mod}(D,d) \}$ for $b = B$.
Since the approximate kernel can be written as
$$\hat{k} (\mat{x}, \mat{y}) := \Phi (\mat{x})^\top \overline{\Phi (\mat{y})} = \frac{1}{D} \sum_{\ell = 1}^D \prod_{i=1}^p \mscpi{s}{x}{i, \ell}  \overline{\mscpi{s}{y}{i, \ell}} = \frac{1}{D} \sum_{b=1}^B \sum_{\ell \in I_b} \prod_{i=1}^p \left(\mat{s}_{i,\ell}^{b \top} \mat{x}  \right)  \overline{ \left( \mat{s}_{i,\ell}^{b \top} \mat{y} \right) }
$$ 
its second moment can be written as
\begin{align}
    \mathbb{E} [ | \hat{k}(\mat{x}, \mat{y}) |^2] 
    &= \frac{1}{D^2}
    \sum_{b, b'=1}^B
    \sum_{\ell \in I_b} \sum_{\ell \in I_{b'}} 
    \mathbb{E} \left[ \prod_{i=1}^p 
    \left(\mat{s}^{b \top}_{i,\ell} \mat{x} \right)  \overline{ \left(\mat{s}^{b \top}_{i,\ell} \mat{y} \right) }
   \overline{ \left(\mat{s}^{b' \top}_{i,\ell'}\mat{x}\right) }  \left(\mat{s}^{b' \top}_{i,\ell'}\mat{y}\right)  \right] \nonumber \\
    &= \frac{1}{D^2}
    \sum_{b, b'=1}^B
    \sum_{\ell \in I_b} \sum_{\ell \in I_{b'}}  \prod_{i=1}^p 
    \mathbb{E} \left[ 
    \left(\mat{s}^{b \top}_{i,\ell} \mat{x} \right) \overline{ \left(\mat{s}^{b \top}_{i,\ell} \mat{y} \right) }
   \overline{ \left(\mat{s}^{b' \top}_{i,\ell'}\mat{x}\right)  }\left(\mat{s}^{b' \top}_{i,\ell'}\mat{y}\right)  \right] \nonumber \\
    &= \frac{1}{D^2}
    \sum_{b, b'=1}^B
    \sum_{\ell \in I_b} \sum_{\ell \in I_{b'}} 
    \left( \mathbb{E} \left[ 
    \left(\mat{s}^{b \top}_{\ell} \mat{x} \right) \overline{ \left(\mat{s}^{b \top}_{\ell} \mat{y} \right) }
   \overline{ \left(\mat{s}^{b' \top}_{\ell'}\mat{x}\right) }\left(\mat{s}^{b' \top}_{\ell'}\mat{y}\right)  \right] \right)^p. \label{eq:second-moment-real-SRHT} 
\end{align}

Now we study individual terms in \eqref{eq:second-moment-real-SRHT}, categorizing the indices $b, b' \in \{1,\dots,B\}$ and $\ell, \ell' \in \{1, \dots, d\}$ of indices into the following 3 cases:
\begin{enumerate}
    \item \underline{$b = b'$ and $\ell=\ell'$ ($D$ terms):} 
    As mentioned earlier, conditioned on the permutation $\pi^b$, $\mat{s}^b_\ell$ is identically distributed as $\mat{z}^b$ (see the paragraph following \cref{eq:def-s-vector-proof}).
        Thus,
    \begin{align*}
       &  \mathbb{E} \left[ 
        \left(\mat{s}_\ell^{b \top} \mat{x}\right)^2 \left(\mat{s}_\ell^{b \top} \mat{y}\right)^2  \right]  = \mathbb{E}_{\pi^b} \left[ \mathbb{E} \left[ 
        \left(\mat{s}_\ell^{b \top} \mat{x}\right)^2 \left(\mat{s}_\ell^{b \top} \mat{y}\right)^2 \mid \pi^b  \right] \right] \\
        & =  \mathbb{E}_{\pi^b} \left[ \mathbb{E} \left[ 
        \left(\mat{z}^{b \top} \mat{x}\right)^2 \left(\mat{z}^{b \top} \mat{y}\right)^2   \right] \right]  =    \mathbb{E} \left[ 
        \left(\mat{z}^{b \top} \mat{x}\right)^2 \left(\mat{z}^{b \top} \mat{y}\right)^2   \right] =  \mathbb{E} \left[ 
        \left(\mat{z}^{\top} \mat{x}\right)^2 \left(\mat{z}^{\top} \mat{y}\right)^2   \right],
    \end{align*}
where $\mathbb{E}_{\pi^b}$ denotes the expectation with respect to $\pi^b$ and $\mat{z} \in \mathbb{C}^d$ is a random vector identically distributed as $\mat{z}^1, \dots, \mat{z}^B$.

    \item \underline{$b = b'$ and $\ell \neq \ell'$ ($c(D,d)$ terms, where $c(D,d)$ is defined in \cref{eq:non-zero-pairs}):} This case requires a  detailed analysis, which we will do below. 
    
    \item \underline{$b \not= b'$  (The rest of terms $= D^2 - D - c(D,d)$ terms):}
    Since $\mat{s}^{b}_{\ell}$ and $\mat{s}^{b'}_{\ell'}$ are independent in this case, we have
    \begin{align*}
   &  \mathbb{E} \left[ 
    \left(\mat{s}^{b \top}_{\ell} \mat{x} \right) \overline{ \left(\mat{s}^{b \top}_{\ell} \mat{y} \right) }
   \overline{ \left(\mat{s}^{b' \top}_{\ell'}\mat{x}\right) }\left(\mat{s}^{b' \top}_{\ell'}\mat{y}\right)  \right] 
    = \mathbb{E} \left[   \left(\mat{s}^{b \top}_{\ell} \mat{x} \right)  \overline{ \left(\mat{s}^{b \top}_{\ell} \mat{y} \right) } \right] \mathbb{E} \left[ \overline{ \left(\mat{s}^{b' \top}_{\ell'}\mat{x}\right) } \left(\mat{s}^{b' \top}_{\ell'}\mat{y}\right) \right] \\
    & = \mathbb{E}[ \hat{k}( \mat{x}, \mat{y} ) ] \overline{  \mathbb{E}[ \hat{k}( \mat{x}, \mat{y} ) ] } = (\mat{x}^\top \mat{y})^{2},
    \end{align*}
\end{enumerate}
where the last equality follows from the approximate kernel being unbiased.

We now analyze the case 2:
\begin{align*}
    &\mathbb{E} \left[ 
    \left(\mat{s}^{b \top}_{\ell} \mat{x} \right) \overline{ \left(\mat{s}^{b \top}_{\ell} \mat{y} \right) }
    \overline{ \left(\mat{s}^{b \top}_{\ell'}\mat{x}\right) } \left(\mat{s}^{b \top}_{\ell'}\mat{y}\right)  \right]
    = \sum_{t, u, w, v=1}^d   \mathbb{E} [s_{\ell, t}^b \overline{s_{\ell, u}^b} \overline{s_{\ell', v}^b} s_{\ell', w}^b ] x_t y_u x_v y_w \\
    &= \sum_{t, u, w, v=1}^d 
    \mathbb{E} [z^b_t \overline{z^b_u} \overline{z^b_v} z^b_w]~ \underbrace{ \mathbb{E} [h_{\pi^b(\ell),t} h_{\pi^b(\ell), u}  h_{\pi^b(\ell'), v} h_{\pi^b(\ell'),w} ]}_{=: E}~ x_t y_u x_v y_w
\end{align*}
Note that we have $\mathbb{E} [z^b_t \overline{z^b_u} \overline{z^b_v} z^b_w] = 0$ unless:
\begin{enumerate}[label={(\alph*)}]
\item $t=u=v=w$:  $\mathbb{E} [z^b_t \overline{z^b_u} \overline{z^b_v} z^b_w] = \mathbb{E} [ | z^b_t |^4 ] = 1$ and  $E = \mathbb{E} [h_{\pi^b(\ell),t}^2  h_{\pi^b(\ell'), t}^2  ] = 1.$

\item $t=u \neq v=w$: $\mathbb{E} [z^b_t \overline{z^b_u} \overline{z^b_v} z^b_w] =  \mathbb{E}[ | z^b_t |^2  | z^b_v |^2  ] = 1$ and $E = \mathbb{E} [h_{\pi^b(\ell),t}^2 h_{\pi^b(\ell'), v}^2] = 1$.

\item $t=v \neq u=w $:  $\mathbb{E} [z^b_t \overline{z^b_u} \overline{z^b_v} z^b_w] =  \mathbb{E}[ | z^b_t |^2  | z^b_u |^2  ] = 1$ and $E= \mathbb{E} [h_{\pi^b(\ell),t} h_{\pi^b(\ell), u}  h_{\pi^b(\ell'), t} h_{\pi^b(\ell'),u} ]$.

\item $t=w \neq u=v$: 
$\mathbb{E} [ z^b_t \overline{z^b_u} \overline{z^b_v} z^b_w]  = \mathbb{E}[ (z^b_t)^2 (\overline{z^b_u} )^2 ] = (2q - 1)^2$ and $E= \mathbb{E} [h_{\pi^b(\ell),t} h_{\pi^b(\ell), u}  h_{\pi^b(\ell'), u} h_{\pi^b(\ell'),t} ]$.
\end{enumerate}
Therefore, we have
\begin{align*}
    &\mathbb{E} \left[ \mscpi{s}{x}{\ell} \mscpi{s}{y}{\ell} \mscpi{s}{x}{\ell'} \mscpi{s}{y}{\ell'} \right] \\
    & = \sum_{t=1}^d  x_t^2 y_t^2 + \sum_{t \not= v} x_t y_t x_v y_v + \sum_{t \not= u}  \mathbb{E} [h_{\pi^b(\ell),t} h_{\pi^b(\ell'), t} h_{\pi^b(\ell), u}   h_{\pi^b(\ell'),u} ] \left( x_t^2 y_u^2 + (2q-1)^2 x_t y_t x_u y_u \right) \\
    & = ( \mat{x}^\top \mat{y} )^2 - \frac{1}{d-1} \sum_{t \not= u} \left( x_t^2 y_u^2 + (2q-1)^2 x_t y_t x_u y_u \right) \quad (\because \text{Lemma~\ref{lemma:key-lemma-hadamard}}) \\ 
    & = ( \mat{x}^\top \mat{y} )^2 - \frac{V^{(1)}_q}{d-1}   ,
\end{align*}
where $V_q^{(1)} := \sum_{t \not= u} \left( x_t^2 y_u^2 + (2q-1)^2 x_t y_t x_u y_u \right)$ is \cref{eq:complex-poly-var-coro-2}
 with $p = 1$, which is the variance of the unstructured polynomial sketch \eqref{eqn:comp-rad-polynomial-sketch} with a single feature.

Now, using these identities in \cref{eq:second-moment-real-SRHT}, the variance of the approximate kernel can be expanded as
\begin{align*}
    & \mathbb{V}[ \hat{k}( \mat{x}, \mat{y} )]  =     \mathbb{E} [\hat{k}(\mat{x}, \mat{y})^2] - ( \mat{x}^\top \mat{y} )^{2p} \\
    = &\frac{1}{D} \left(  \mathbb{E} \left[
        \left(\mat{z}^{\top} \mat{x}\right)^2 \left(\mat{z}^{\top} \mat{y}\right)^2   \right] \right)^p  + \frac{c(D,d)}{D^2} \left(  ( \mat{x}^\top \mat{y} )^2 - \frac{V^{(1)}_q}{d-1}  \right)^p \\
    &\, + \frac{D^2 - D - c(D,d)}{D^2}   ( \mat{x}^\top \mat{y} )^{2p}  - ( \mat{x}^\top \mat{y} )^{2p} \\
    = & \frac{1}{D} \left[ \left(  \mathbb{E} \left[
        \left(\mat{z}^{\top} \mat{x}\right)^2 \left(\mat{z}^{\top} \mat{y}\right)^2   \right] \right)^p - ( \mat{x}^\top \mat{y} )^{2p} \right]   - \frac{c(D,d)}{D^2} \left[ ( \mat{x}^\top \mat{y} )^{2p} -  \left(  ( \mat{x}^\top \mat{y} )^2 - \frac{V^{(1)}_q}{d-1}  \right)^p \right] \\
    = &  \frac{1}{D} V_q^{(p)}  - \frac{c(D,d)}{D^2}  \left[ (\mat{x}^\top \mat{y})^{2p}  -  \left( ( \mat{x}^\top \mat{y} )^2 - \frac{V_q^{(1)}}{d-1}   \right)^p \right]
\end{align*}
where $V_q^{(p)} \geq 0$ is \cref{eq:complex-poly-var-coro-2} with the considered value of the polynomial degree $p$, which is the  variance of the unstructured polynomial sketch \eqref{eqn:comp-rad-polynomial-sketch} with a single feature. 
This completes the proof.

\section{Convex Surrogate Functions for TensorSRHT Variances}
\label{sec:incremental-tensorsrht}

To extend the applicability of the Incremental Algorithm in  \cref{alg:incremental-algorithm} to TensorSRHT, we derive here convex surrogate functions for the variances of TensorSRHT, 
To this end, we first analyze the variances of TensorSRHT in  \cref{sec:analysis-variances-tensorSRHT}.
We then derive convex surrogate functions in  \cref{sec:covex-surrogate-TensorSRHT}.

\subsection{Analyzing the Variances of TensorSRHT}
\label{sec:analysis-variances-tensorSRHT}

We first derive another form of the variance of TensorSRHT given in \cref{eqn:complex-tensor-srht-variance} of \cref{thm:variance-complexTensorSRHT}, which we will use in a later analysis.
Let $\Phi_n : \mathbb{R}^d \to \mathbb{C}^D$ be a complex TensorSRHT sketch of degree $n \in \mathbb{N}$ satisfying the assumptions in  \cref{thm:variance-complexTensorSRHT} with $0 \leq q \leq 1$.
For $q=1$ we recover the real TensorSRHT and for $q=1/2$ the complex one.

As shown in \cref{sec:proof-var-comp-TensorSRHT}, the approximate kernel of the complex TensorSRHT can be written as
\begin{align*}
    \hat{k} (\mat{x}, \mat{y}) := \Phi_n (\mat{x})^\top \overline{\Phi_n (\mat{y})} = \frac{1}{D} \sum_{b=1}^B \sum_{\ell \in I_b} \prod_{i=1}^n \left(\mat{s}_{i,\ell}^{b \top} \mat{x}  \right)  \overline{ \left( \mat{s}_{i,\ell}^{b \top} \mat{y} \right) },
\end{align*}
where $B := \lceil D/d \rceil$, $I_b := \{1,\dots, d\}$ for $b = 1,\dots, B-1$ and $I_b := \{ 1, \dots, {\rm mod}(D,d) \}$ for $b = B$, and  $\mat{s}_{i, \ell}^b \in \mathbb{C}^d$ are the structured random weights defined in \cref{eq:def-s-vector-proof}.
We can then write the variance of the approximate kernel as
\begin{align}
    \mathbb{V} [\hat{k} (\mat{x}, \mat{y})]
     = & \frac{1}{D^2} \sum_{b=1}^B \mathbb{V} \left[ \sum_{\ell \in I_b} \prod_{i=1}^n \left(\mat{s}_{i,\ell}^{b \top} \mat{x}  \right)  \overline{ \left( \mat{s}_{i,\ell}^{b \top} \mat{y} \right) } \right] \nonumber \\
     = & \frac{1}{D^2} \sum_{b=1}^B \sum_{\ell \in I_b} \underbrace{\mathbb{V} \left[ \prod_{i=1}^n \left(\mat{s}_{i,\ell}^{b \top} \mat{x}  \right)  \overline{ \left( \mat{s}_{i,\ell}^{b \top} \mat{y} \right) } \right]}_{= \ V_q^{(n)}}  \nonumber \\
   & +  \frac{1}{D^2} \sum_{b=1}^B \sum_{\substack{\ell, \ell' \in I_b, \\ \ell \neq \ell'}} \underbrace{{\rm Cov} \left( \prod_{i=1}^n \left(\mat{s}_{i,\ell}^{b \top} \mat{x}  \right)  \overline{ \left( \mat{s}_{i,\ell}^{b \top} \mat{y} \right) }, \prod_{i=1}^n \left(\mat{s}_{i,\ell'}^{b \top} \mat{x}  \right)  \overline{ \left( \mat{s}_{i,\ell'}^{b \top} \mat{y} \right) } \right)}_{=:  \ {\rm Cov}_q^{(n)}} \nonumber \\ 
   = &  \frac{V^{(n)}_q}{D} 
    + \frac{c(D,d)}{D^2} {\rm Cov}^{(n)}_q = {\rm  \cref{eqn:complex-tensor-srht-variance}}, \label{eqn:var-tensor-srht-rewritten}
\end{align}
where   $c(D,d) = \lfloor D/d \rfloor d(d-1) + {\rm mod}(D,d)({\rm mod}(D, d) - 1)$ and the last line follows from that the values of $V_q^{(n)}$ and ${\rm Cov}_q^{(n)}$ do not depend on the choice of $\ell, \ell'$ and $b$ (which can be shown from the arguments in \cref{sec:proof-var-comp-TensorSRHT}).
Here, $V^{(n)}_q$ is the variance of the unstructured Rademacher sketch with a single feature in \cref{eq:complex-poly-var-coro-2} with $p = n$, and ${\rm Cov}_q^{(n)}$ is the covariance for distinct indices $\ell, \ell'$ inside each block $b$. 
By comparing  \cref{eqn:complex-tensor-srht-variance} and \cref{eqn:var-tensor-srht-rewritten}, the concrete form of ${\rm Cov}_q^{(n)}$ is given by
$$
{\rm Cov}_q^{(n)} = - \left[ (\mat{x}^\top \mat{y})^{2n}  -  \left( ( \mat{x}^\top \mat{y} )^2 - \frac{V_q^{(1)}}{d-1}   \right)^n \right]
$$

\cref{eqn:var-tensor-srht-rewritten} is a useful representation of the variance of TensorSRHT in \cref{eqn:complex-tensor-srht-variance}  for studying its (non-)convexity with respect to $D$.
The following result shows  a range of values of $D$ for which \cref{eqn:complex-tensor-srht-variance}  is convex.

\begin{theorem}
\label{thrm:tensor-srht-convexity}
The variance of the TensorSRHT  sketch in  \cref{eqn:complex-tensor-srht-variance}  is convex and monotonically decreasing with respect to $D \in \{ 1, \dots, d \}$  and with respect to $D \in \{ kd \mid k \in \mathbb{N} \}$. 
\end{theorem}

\begin{proof}
If $D \in \{1, \dots, d\}$, we have $c(D,d) = D(D-1)$ in \cref{eqn:var-tensor-srht-rewritten}. Therefore, \cref{eqn:var-tensor-srht-rewritten} is equal to 
\begin{align}
    \label{eqn:variance-tensor-srht-small-D}
    \frac{1}{D} V_q^{(n)} + \left(1 - \frac{1}{D} \right) {\rm Cov}_q^{(n)}
    = \frac{1}{D} \left( V_q^{(n)} - {\rm Cov}_q^{(n)} \right) + {\rm Cov}_q^{(n)}.
\end{align}
For two random variables $X, Y$ it generally holds that $|{\rm Cov}(X,Y)| \leq \sqrt{\mathbb{V}[X] \mathbb{V}[Y]}$ by the Cauchy-Schwarz inequality.
Hence, we have $|{\rm Cov}_q^{(n)}| \leq V_q^{(n)}$ and thus $V_q^{(n)} - {\rm Cov}_q^{(n)} \geq 0$. Therefore,  \cref{eqn:variance-tensor-srht-small-D} is proportional to $1/D$ with a non-negative coefficient, and thus it is convex and monotonically decreasing for $D \in \{1,\dots, d\}$.

Next, suppose $D = kd$ for some $k \in \mathbb{N}$, in which case we have $c(D, d) = kd(d-1)$ in \cref{eqn:var-tensor-srht-rewritten}. 
Therefore \cref{eqn:var-tensor-srht-rewritten} is equal to
\begin{align} \label{eq:proof-convex-var-TSRHT}
  \frac{1}{kd} \left( V_q^{(n)} + (d-1) {\rm Cov}_q^{(n)} \right) =   \frac{1}{D} \left( V_q^{(n)} + (d-1) {\rm Cov}_q^{(n)} \right).
\end{align}
The term in the parenthesis is non-negative, because  \eqref{eqn:var-tensor-srht-rewritten} is the variance of TensorSRHT and thus non-negative. 
Therefore,  \eqref{eqn:var-tensor-srht-rewritten} is convex and monotonically decreasing with respect to $D \in \{ kd \mid k \in \mathbb{N} \}$.

\end{proof}

As we do next,
\cref{thrm:tensor-srht-convexity} is useful for designing a convex surrogate function for \cref{eqn:complex-tensor-srht-variance}, as it shows the range of $D$ on which  \cref{eqn:complex-tensor-srht-variance}  is already convex and does not need to be modified.

\subsection{Convex Surrogate Functions}
\label{sec:covex-surrogate-TensorSRHT}

Based on \cref{eqn:complex-tensor-srht-variance}, we now propose a convex surrogate function for the variance of TensorSRHT in \cref{eqn:complex-tensor-srht-variance}.
We consider the following two cases separately: i)  $ {\rm Cov}_q^{(n)} \leq 0 $ and ii)   $ {\rm Cov}_q^{(n)} > 0 $.
For each case, we propose a convex surrogate function.

\paragraph{\underline{ i) Case  $ {\rm Cov}_q^{(n)} \leq 0 $.}}
 We define a surrogate function of  \cref{eqn:complex-tensor-srht-variance} by concatenating the two expressions of  \cref{eqn:var-tensor-srht-rewritten}  for $D \in \{1,\dots,d\}$ and $D \in \{kd \mid k \in \mathbb{N} \}$ given in \cref{eqn:variance-tensor-srht-small-D} and \cref{eq:proof-convex-var-TSRHT}, respectively, and extend their ranges to the entire domain $D \in \mathbb{N}$:
\begin{align}
    V^{(n)}_{\rm Surr.}(D)
    := \left\{
    \begin{array}{ll}
        \frac{1}{D} \left( V_q^{(n)} - {\rm Cov}_q^{(n)} \right) + {\rm Cov}_q^{(n)} & \mbox{if }  D \leq d \\
        \frac{1}{D} \left(V_q^{(n)} + (d-1) {\rm Cov}_q^{(n)}\right) & \mbox{if }  D > d .
    \end{array} 
\right.
\label{eqn:var-convexified-2}
\end{align}

\paragraph{\underline{ii) Case $ {\rm Cov}_q^{(n)} >  0$.}}
We use the expression \eqref{eq:proof-convex-var-TSRHT} to define a surrogate function on $D \in \mathbb{N}$:
\begin{align}
    V^{(n)}_{\rm Surr.}(D)
    :=   \frac{1}{D} \left(V_q^{(n)} + (d-1) {\rm Cov}_q^{(n)}\right)
\label{eqn:var-convexified-positive-cov}
\end{align}
The convexity of \cref{eqn:var-convexified-positive-cov} immediately follows from $V_q^{(n)} + (d-1) {\rm Cov}_q^{(n)} \geq 0$, which holds as we show in the proof of \cref{thrm:tensor-srht-convexity}.
Note that ${\rm Cov}_q^{(n)} > 0$ can only occur when $n$ is even, as shown in Corollary \ref{coro:complex-tensor-srht-odd-p} of Section~\ref{sec:structured-projections}.

We defined the surrogate function in \cref{eqn:var-convexified-2}  by interpolating  the variances of TensorSRHT in \cref{eqn:complex-tensor-srht-variance} for $D \in \{ 1, \dots, d \}$ and 
 $D \in \{ kd \mid k \in \mathbb{N} \}$ and extending the domain to $\mathbb{N}$.
In fact, for $D \in \{ 1, \dots, d \}$ and 
 $D \in \{ kd \mid k \in \mathbb{N} \}$,  \cref{eqn:var-convexified-2}  is equal to \cref{eqn:complex-tensor-srht-variance}, as shown in the proof of \cref{thrm:tensor-srht-convexity}.
\cref{fig:convex-approx} illustrates the convex surrogate function in \cref{eqn:var-convexified-2} and the variance of  TensorSRHT in \eqref{eqn:complex-tensor-srht-variance} when  ${\rm Cov}_q^{(n)} \leq 0$ holds.

Note that, as mentioned later in Remark \ref{remark:non-convexity-pos-covariance}, the surrogate function in \cref{eqn:var-convexified-2}  may not be convex over $D \in \mathbb{N}$ if the condition ${\rm Cov}_q^{(n)} \leq 0$ does not hold.
This is why we defined another convex surrogate function as in \cref{eqn:var-convexified-positive-cov} for the case ${\rm Cov}_q^{(n)} > 0$.

\begin{figure}[t]
\centering
\includegraphics[width=0.9\textwidth]{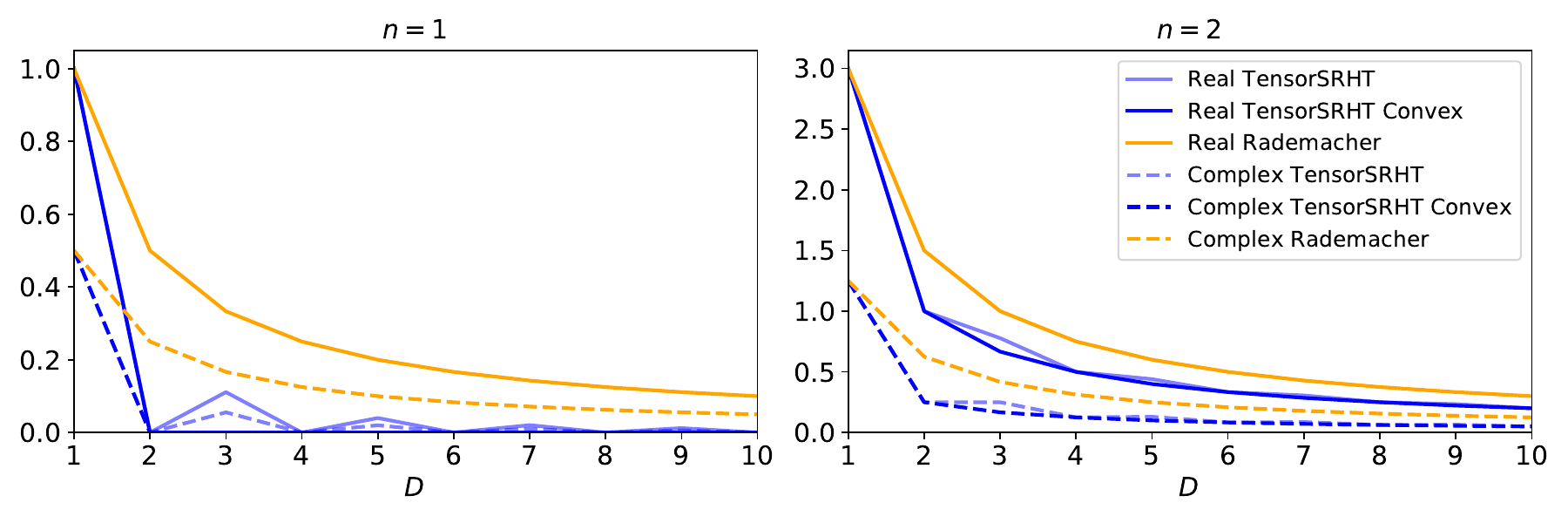}
\caption{Convex surrogate functions in \cref{eqn:var-convexified-2} and the variances of  TensorSRHT in \eqref{eqn:complex-tensor-srht-variance} as a function of the number of random features $D$, with polynomial degrees $n=1,2$ and input vectors $\mat{x} = \mat{y} = [\sqrt{1/2}, \sqrt{1/2}]^{\top}$ ($d = 2$). 
For comparison, we also plot the variances of the real Rademacher sketch in \cref{eqn:rademacher-variance} and the complex Rademacher sketch in \cref{eqn:complex-rademacher-variance}.
}
\label{fig:convex-approx}
\end{figure}

The following theorem shows that the surrogate function in \cref{eqn:var-convexified-2}  is convex in the considered case of i)  $ {\rm Cov}_q^{(n)} \leq 0 $.
\begin{theorem}
\label{thrm:surrogate-convexity}
If $ {\rm Cov}_q^{(n)} \leq 0 $,  \cref{eqn:var-convexified-2}  is convex with respect to $D \in \mathbb{N}$.
\end{theorem}

\begin{proof}
As shown in \cref{thrm:tensor-srht-convexity},  $V^{(n)}_{\rm Surr.}(D) = \frac{1}{D} ( V_q^{(n)} - {\rm Cov}_q^{(n)} ) + {\rm Cov}_q^{(n)}$ is convex over $D \in \{1,\dots,d\}$.
Likewise, 
$V^{(n)}_{\rm Surr.}(D) = \frac{1}{D} (V_q^{(n)} + (d-1) {\rm Cov}_q^{(n)})$ is convex over $D \in  [d, \infty) \cap  \mathbb{N}$, since  $V_q^{(n)} + (d-1) {\rm Cov}_q^{(n)} \geq 0$ holds as we show in the proof of \cref{thrm:tensor-srht-convexity}.

Therefore, the proof completes by showing that the concatenated function $V^{(n)}_{\rm Surr.}(D)$ in \cref{eqn:var-convexified-2} is also convex over $D \in \{d-1, d, d+1\}$, i.e.,
\begin{equation} \label{eq:convexity-surrogate-proof-inequality}
\frac{1}{2} \left( V^{(n)}_{\rm Surr.}(d-1) + V^{(n)}_{\rm Surr.}(d+1)  \right) \geq V^{(n)}_{\rm Surr.}(d) .
\end{equation}

By using the definition in  \cref{eqn:var-convexified-2}, this inequality is equivalent to
\begin{align}
 &   \frac{1}{2} \left(\frac{1}{d-1} \left( V_{q}^{(n)} + (d-2) {\rm Cov}_q^{(n)} \right) + \frac{1}{d+1} \left( V_{q}^{(n)} + (d-1) {\rm Cov}_q^{(n)} \right)\right) \nonumber \\
&    \geq \frac{1}{d} \left( V_{q}^{(n)} + (d-1) {\rm Cov}_q^{(n)} \right). \label{eq:inequality-3898}
\end{align}
Note that we have $V_q^{(n)} + (d-1) {\rm Cov}_q^{(n)} \geq 0$ , as mentioned earlier. %
If  $V_q^{(n)} + (d-1) {\rm Cov}_q^{(n)} = 0$ holds, then we have  $V^{(n)}_{\rm Surr.}(D) = 0$ for $D \geq d$ by the definition in  \cref{eqn:var-convexified-2}, and thus \cref{eq:convexity-surrogate-proof-inequality} holds (which concludes the proofs).
Therefore,  we assume the inequality to be strict, i.e.,  $V_q^{(n)} + (d-1) {\rm Cov}_q^{(n)} > 0$.

Dividing the both sides of \cref{eq:inequality-3898} by $(V_q^{(n)} + (d-1) {\rm Cov}_q^{(n)})$, we  obtain
\begin{align*}
    \frac{1}{2} \left( \frac{1}{d-1} \frac{V_q^{(n)} + (d-2) {\rm Cov}_q^{(n)}}{V_q^{(n)} + (d-1) {\rm Cov}_q^{(n)}} + \frac{1}{d+1}\right) \geq \frac{1}{d},
\end{align*}
which after some rearrangement gives
\begin{align} \label{eq:proof-convexity-3917}
    \frac{V_q^{(n)} + (d-2) {\rm Cov}_q^{(n)}}{V_q^{(n)} + (d-1) {\rm Cov}_q^{(n)}} \geq 1 - \frac{2}{d^2 + d}.
\end{align}
This inequality holds  because we have $(d-2) {\rm Cov}_q^{(n)} \geq (d-1) {\rm Cov}_q^{(n)}$, which follows from our assumption ${\rm Cov}_q^{(n)} \leq 0$. 
Therefore  \cref{eq:inequality-3898} holds.

\end{proof}

\begin{remark} \label{remark:non-convexity-pos-covariance}
\cref{thrm:surrogate-convexity} shows the convexity of the surrogate function in  \cref{eqn:var-convexified-2}, assuming  ${\rm Cov}_q^{(n)} \leq 0$. 
If this condition does not hold, i.e., if  ${\rm Cov}_q^{(n)} > 0$,  then the surrogate function in \cref{eqn:var-convexified-2} may not be convex.  
To see this,  let $d=2, \mat{x} = (a, 0)^\top$ with $a > 0$, $\mat{y} = (0, b)^\top$ with $b > 0$,  and $n$ be even; then we have $V_q^{(n)} = {\rm Cov}_q^{(n)} = a^{2n} b^{2n} > 0$, and the inequality in \cref{eq:proof-convexity-3917} in the proof of \cref{thrm:surrogate-convexity} does not hold, which implies that the surrogate function in  \cref{eqn:var-convexified-2} is not convex.
\end{remark}

As mentioned in Section \ref{sec:structured-projections}, the variance of TensorSRHT in \cref{eqn:complex-tensor-srht-variance} becomes zero if $ n = 1$ and $D \in \{ kd \mid k \in \mathbb{N} \}$, i.e., $\mathbb{V}[\Phi_1(\mat{x})^{\top} \overline{\Phi_1(\mat{y})}] = 0$ holds.
Therefore, because the convex surrogate functions in \cref{eqn:var-convexified-2} and \cref{eqn:var-convexified-positive-cov} are equal to the variance of TensorSRHT in \cref{eqn:complex-tensor-srht-variance} for $D  \in \{ kd \mid k \in \mathbb{N} \}$, these surrogate functions also become zero for  $n=1$ and $D \in \{ kd \mid k \in \mathbb{N} \}$.
Thus, the Incremental Algorithm (\cref{alg:incremental-algorithm}), when used with the surrogate functions in  \cref{eqn:var-convexified-2} and \cref{eqn:var-convexified-positive-cov}, will not assign more than $D=d$ random features to the polynomial degree $n = 1$. 
Note that assigning $D=d$ random features is equivalent to appending the input vectors $\mat{x}$ and $\mat{y}$ to the approximate kernel \eqref{eqn:approximate-kernel}, which is called {\em H0/1 heuristic} in \citet{Kar2012}.
Therefore, the Incremental Algorithm with the surrogate functions in  \cref{eqn:var-convexified-2} and \cref{eqn:var-convexified-positive-cov} automatically achieve the H0/1 heuristic.

Finally, we describe briefly how to use the convex surrogate functions in  \cref{eqn:var-convexified-2} and \cref{eqn:var-convexified-positive-cov}  in the Incremental Algorithm in \cref{alg:incremental-algorithm}. 
To this end, we rewrite \cref{eq:simple-objective-individual-term} using the surrogate functions as follows:
(Here, we make the dependence of $ V_q^{(n)}$ and ${\rm Cov}_q^{(n)}$ on the input vectors $\mat{x}, \mat{y} \in \mathbb{R}^d$ explicit and write them as  $ V_q^{(n)} (\mat{x}, \mat{y})$ and ${\rm Cov}_q^{(n)} (\mat{x}, \mat{y})$, respectively.)  
\begin{align*}
    {\rm \cref{eq:simple-objective-individual-term}}
    = \left\{
    \begin{array}{ll}
        \frac{a_n^2}{D_n} \left( \sum_{i \neq j} V_q^{(n)}(\mat{x}_i, \mat{x}_j) + (d-1) \sum_{i \neq j} {\rm Cov}_q^{(n)}(\mat{x}_i, \mat{x}_j) \right) \\
        \quad\quad \mbox{if } \sum_{i \neq j} {\rm Cov}_q^{(n)}(\mat{x}_i, \mat{x}_j) > 0 \text{ or } D_n > d, \\
        \frac{a_n^2}{D_n} \left( \sum_{i \neq j} V_q^{(n)}(\mat{x}_i, \mat{x}_j) - \sum_{i \neq j} {\rm Cov}_q^{(n)}(\mat{x}_i, \mat{x}_j) \right) + a_n^2 \sum_{i \neq j} {\rm Cov}_q^{(n)}(\mat{x}_i, \mat{x}_j) \\
        \quad\quad \mbox{otherwise.}
    \end{array} 
    \right.
\end{align*}
After precomputing the constants $\sum_{i \neq j} V_q^{(n)}(\mat{x}_i, \mat{x}_j)$ and $\sum_{i \neq j} {\rm Cov}_q^{(n)}(\mat{x}_i, \mat{x}_j)$ for each $n \in \{1,\dots,p\}$, which can be done in $\mathcal{O}(m^2)$ time, one can directly use the above modification of \cref{eq:simple-objective-individual-term} in the objective function in \cref{eqn:optimization-objective}. 
In this way, we adapt the objective function in \eqref{eqn:optimization-objective} to be convex, so that the Incremental Algorithm in \cref{alg:incremental-algorithm} is directly applicable.

\section{Gaussian Processes with Complex Random Features}
\label{sec:appendix-gp-identities}

We describe here how to use complex random features in Gaussian process (GP) regression and classification.
Since real random features are special cases of complex random features, all derivations for the complex case also hold for the real case as well.

For GP classification, we employ the framework of \citet{Milios2018}, which formulates GP classification using GP regression and provides a solution in closed form. 
Therefore, closed form solutions are available for both GP regression and classification, and this enables us to compare different random feature approximations directly.\footnote{If we use a formulation of GP classification that requires an optimization procedure, comparisons of random feature approximations become more involved, as we need to perform convergence verification for the optimization procedure.}

\paragraph{Notation and definitions.}
For a matrix $\mat{A} \in \mathbb{C}^{n \times m}$ with $n, m \in \mathbb{N}$, denote by $\mat{A}^H := \overline{\mat{A}}^\top \in \mathbb{C}^{m \times n}$ be its conjugate transpose.
Note that if $\mat{A} \in \mathbb{R}^{n \times m}$, then $\mat{A}^H = \mat{A}^\top \in \mathbb{R}^{m \times n}$. 
For $n \in \mathbb{N}$, $\mat{I}_n \in \mathbb{R}^{n \times n}$ be the identity matrix.
 
For $\mat{\mu} \in \mathbb{C}^n$ and  positive semi-definite\footnote{A Hermitian matrix $\mat{\Sigma} \in \mathbb{C}^{n \times n}$ is called positive semi-definite, if for all $\mat{v} \in \mathbb{C}^n$, we have $\mat{v}^H \mat{\Sigma} \mat{v} \geq 0$.} $\mat{\Sigma} \in \mathbb{C}^{n \times n}$ with $n \in \mathbb{N}$, we denote by  $\mathcal{CN}(\mat{\mu}, \mat{\Sigma})$  the $n$-dimensional {\em proper} complex Gaussian distribution with mean vector $\mat{\mu}$ and covaraince matrix $\mat{\Sigma}$, whose density function is given by \citep[e.g.,][Theorem 1]{neeser1993proper}
 $$
\mathcal{CN}(\mat{v}; \mat{\mu}, \mat{\Sigma}) := \frac{1}{\pi^n \sqrt{ | \mat{\Sigma} | }} \exp\left( -   (\mat{v}- \mat{\mu})^H \mat{\Sigma}^{-1} (\mat{v}- \mat{\mu}) \right), \quad \mat{v} \in \mathbb{C}^n,
 $$
 where $| \mat{\Sigma} |$ is the determinant of $\mat{\Sigma}$.
If a random vector $\mat{f} \in \mathbb{C}^n$ follows $\mathcal{CN}(\mat{v}; \mat{\mu}, \mat{\Sigma})$, we have $\mathbb{E}[\mat{f}] = \mat{\mu}$, $\mathbb{E}[ (\mat{f} - \mat{\mu}) (\mat{f} - \mat{\mu})^H  ] = \mat{\Sigma}$, and $\mathbb{E}[ (\mat{f} - \mat{\mu}) (\mat{f} - \mat{\mu})^\top  ] = 0$,  where the last property is the definition of $\mat{f}$ being a proper complex random variable \citep[Definition 1]{neeser1993proper}.

\subsection{Complex GP Regression}

We first describe the approach of {\em complex GP regression} \citep{Boloix-Tortosa2015}, a Bayesian nonparametric approach to complex-valued regression.

Suppose that there are training data $(\mat{x}_i, y_i)_{i=1}^N \subset \mathbb{R}^d \times \mathbb{C}$ for a complex-valued regression problem with $N \in \mathbb{N}$, and let 
$\mat{X} := (\mat{x}_1, \dots, \mat{x}_N)^{\top} \in \mathbb{R}^{N \times d}$ and $\mat{y} := ( y_1, \dots, y_N )^\top \in \mathbb{C}^N$.
We assume the following model for the training data:
\begin{equation} \label{eq:complex-regres-model}
    y_i = f(x_i) + \varepsilon_i, \quad (i=1,\dots,N),
\end{equation}
where $f: \mathbb{R}^d \to \mathbb{C}$ is an unknown complex-valued function, and $\varepsilon_i \sim \mathcal{CN}(0, \sigma_i^2)$ is an independent  complex Gaussian noise with variance $\sigma_i^2 > 0$.
Let $\mat{\sigma}^2 := (\sigma_1^2, \dots, \sigma_N^2)^{\top} \in \mathbb{R}^N$.

The task of complex-valued function is to estimate the unknown complex-valued function $f$ in \cref{eq:complex-regres-model} from the training data $(\mat{x}_i, y_i)_{i=1}^N \subset \mathbb{R}^d \times \mathbb{C}$.  
In complex GP regression, one defines a {\em complex GP prior distribution} for the unknown function $f$, and derives a {\em complex GP posterior distribution} of $f$, given the data $(\mat{x}_i, y_i)_{i=1}^N \subset \mathbb{R}^d \times \mathbb{C}$ and the likelihood function given by  \cref{eq:complex-regres-model}. 
For the prior, we focus on a {\em proper} complex GP \citep[Section II-C]{Boloix-Tortosa2015}, which we describe below.

 \paragraph{Proper complex Gaussian processes.}
A complex-valued function $k: \mathbb{R}^d \times \mathbb{R}^d \to \mathbb{C}$ is called {\em positive definite kernel}, if 1) $k(\mat{x}, \mat{x}') = \overline{ k(\mat{x}', \mat{x})  }$ for all $\mat{x}, \mat{x}' \in \mathbb{R}^d$; and ii)  for all $n \in \mathbb{N}$ and all $\mat{x}_1, \dots, \mat{x}_n \in \mathbb{R}^d$, the matrix $\mat{K} \in \mathbb{C}^{n \times n}$ with $\mat{K}_{i,j} = k(x_i, x_j)$ satisfies $\mat{v}^H \mat{K} \mat{v} \geq 0$.

Let $f : \mathbb{R}^d \to \mathbb{C}$ be a zero-mean complex-valued stochastic process, and $k: \mathbb{R}^d \times \mathbb{R}^d \to \mathbb{C}$ be a positive definite kernel.
We call $f$ a (zero-mean) {\em proper complex GP} with covariance kernel $k$, if for all $n \in \mathbb{N}$ and all $\mat{x}_1, \dots, \mat{x}_n \in \mathbb{R}^d$, the random vector $\mat{f} := (f(\mat{x}_1), \dots, f(\mat{x}_n))^\top \in \mathbb{C}^n$  follows the proper complex Gaussian distribution $\mathcal{CN}(\mat{0}, \mat{K})$ with covariance matrix $\mat{K} \in \mathbb{C}^{n \times n}$ with $\mat{K}_{i, j} = k(\mat{x}_i, \mat{x}_j)$.
If $f$ is a zero-mean proper complex GP with covariance kernel $k$, we write $f \sim \mathcal{CGP}(0,k)$.

We now describe the approach of complex GP regression. 
For the unknown  $f$ in \cref{eq:complex-regres-model}, we define a proper complex GP prior with kernel $k$, assuming that 
\begin{align}
    \label{eqn:gp-prior}
    f \sim \mathcal{CGP}(0,k)
\end{align}
Then  the observation model \eqref{eq:complex-regres-model} and the prior \eqref{eqn:gp-prior}  induce a joint distribution of the unknown function $f$ and the training observations $\mat{y} = (y_1, \dots, y_N)^\top$. 
Conditioned on $\mat{y}$, we obtain the {\em posterior distribution} of $f$, which is also a proper complex GP
\citep[Section II-C]{Boloix-Tortosa2015}:
\begin{align}
    \label{eqn:complex-gp-posterior}
   f  \mid \mat{y} & \sim \mathcal{CGP} (\mu_N, k_N ),
\end{align}
where $\mu_N : \mathbb{R}^d \to \mathbb{C}$ is the {\em posterior mean function} and $k_N : \mathbb{R}^d \times \mathbb{R}^d \to \mathbb{C}$  is the {\em posterior covariance function} given by
\begin{align}   \label{eqn:complex-gp-post-mean}
    \mu_N( \mat{x} ) &:=  \mat{k}(\mat{x})^H ( \mat{K} + \diag{\mat{\sigma}^2})^{-1} \mat{y}, \quad \mat{x} \in \mathbb{R}^d \\
    \label{eqn:complex-gp-post-cov}
   k_N(\mat{x}, \mat{x}') &:= k(\mat{x}, \mat{x}')  - \mat{k}(\mat{x})^H ( \mat{K} + \diag{\mat{\sigma}^2})^{-1} \mat{k}(\mat{x}),  \quad \mat{x}, \mat{x}'  \in \mathbb{R}^d, 
\end{align}
where $\mat{k}(\mat{x}) := (k(\mat{x}, \mat{x}_1), \dots, k(\mat{x}, \mat{x}_N))^\top \in \mathbb{C}^N$, $\mat{K} \in \mathbb{C}^{N \times N}$ with $\mat{K}_{i,j} = k(\mat{x}_i, \mat{x}_j)$, and   $\diag{\mat{\sigma}^2} \in \mathbb{R}^{d \times d}$ is the diagonal matrix with diagonal elements $\mat{\sigma}^2 = (\sigma_1^2, \dots, \sigma_N^2)^\top$.

Notice that, if the kernel $k$ is real-valued and so are the observations $\mat{y}$,   \cref{eqn:complex-gp-post-mean} and \cref{eqn:complex-gp-post-cov} reduce to the posterior mean and covariance functions of standard real-valued GP regression \citep[e.g.,][Chapter 2]{Rasmussen2006}.  
In this sense, complex GP regression with a proper GP prior is a natural complex extension of standard GP regression.

\subsection{GP Regression with Complex Features} 

\label{sec:GP-regress-comp-features}
We next describe how to use complex features in GP regression. 
Let $\Phi: \mathbb{R}^d \to \mathbb{C}^D$ be a complex-valued (random) feature map,\footnote{Again, this subsumes the case of real-valued feature maps.}  and let $\hat{k}(\mat{x}, \mat{x}') := \Phi(\mat{x})^\top \overline{\Phi(\mat{x}')}$ be the approximate kernel. Define
\begin{equation} \label{eq:notation-feat-map-kernel}
\Phi(\mat{X}) := (\Phi(\mat{x}_1), \dots, \Phi(\mat{x}_N))^{\top} \in \mathbb{C}^{N \times D},\quad \hat{\mat{K}} := \Phi(\mat{X}) \Phi(\mat{X})^H \in \mathbb{C}^{N \times N},
\end{equation}
where $\mat{x}_1, \dots, \mat{x}_N \in \mathbb{R}^D$ are training inputs.
Note that $\hat{\mat{K}}_{i,j}= \Phi(\mat{x}_i)^\top \overline{\Phi(\mat{x}_j)} = \hat{k}(\mat{x}_i, \mat{x}_j)$, i.e., $\hat{\mat{K}}$ is the kernel matrix with kernel $\hat{k}$. 

The approximate kernel $\hat{k}: \mathbb{R}^d \times \mathbb{R}^d \to \mathbb{C}$ is complex-valued, and thus induces a proper complex GP, $f \sim \mathcal{CGP}(0, \hat{k})$.
Using this GP as a prior for the unknown function $f$ in the observation model  \eqref{eq:complex-regres-model}, and conditioning on the observations $\mat{y} = (y_1, \dots, y_N)^\top$, we obtain the following approximate complex GP posterior:
\begin{align}
    \label{eqn:complex-gp-posterior-approx}
   f  \mid \mat{y} & \sim \mathcal{CGP} (\hat{\mu}_N, \hat{k}_N ),
   \end{align}
where $\hat{\mu}_N : \mathbb{R}^d \to \mathbb{C}$ is an approximate posterior mean function and $\hat{k}_N : \mathbb{R}^d \times \mathbb{R}^d \to \mathbb{C}$  is an approximate  posterior covariance function, defined as 
    \begin{align}   \label{eqn:complex-gp-post-mean-approx}
    \hat{\mu}_N( \mat{x} ) &:=  \hat{\mat{k}}(\mat{x})^H ( \hat{\mat{K}} + \diag{\mat{\sigma}^2})^{-1} \mat{y}, \quad \mat{x} \in \mathbb{R}^d \\
    \label{eqn:complex-gp-post-cov-approx}
   \hat{k}_N(\mat{x}, \mat{x}') &:= \hat{k}(\mat{x}, \mat{x}')  - \hat{\mat{k}}(\mat{x})^H ( \hat{\mat{K}} + \diag{\mat{\sigma}^2})^{-1} \hat{\mat{k}}(\mat{x}),  \quad \mat{x}, \mat{x}'  \in \mathbb{R}^d, 
\end{align}
where $\hat{\mat{k}}(\mat{x}) := (\hat{k}(\mat{x}, \mat{x}_1), \dots, \hat{k}(\mat{x}, \mat{x}_N))^\top \in \mathbb{C}^N$, and $\hat{\mat{K}} \in \mathbb{C}^{N \times N}$ with $\hat{\mat{K}}_{i,j} = \hat{k}(\mat{x}_i, \mat{x}_j)$.

Finally, we define a real-valued approximate GP posterior using the real parts of  \cref{eqn:complex-gp-post-mean-approx} and
\cref{eqn:complex-gp-post-cov-approx}.  
That is, define $\hat{\mu}_{N, \mathbb{R}}: \mathbb{R}^d \to \mathbb{R}$ as the real part of the approximate posterior mean function in \cref{eqn:complex-gp-post-mean-approx} , and $\hat{k}_{N, \mathbb{R}}$ as the real part of the approximate covariance function in \cref{eqn:complex-gp-post-cov-approx}: 
\begin{align}
\hat{\mu}_{N, \mathbb{R}}(\mat{x}) &:= \mathcal{R} \left\{ \hat{\mu}_N(\mat{x}) \right\}, \quad \mat{x} \in \mathbb{R}^d, \label{eqn:real-gp-post-mean-approx} \\
\hat{k}_{N, \mathbb{R}}  (\mat{x}, \mat{x}') &:= \mathcal{R} \left\{ \hat{k}_N (\mat{x}, \mat{x}')  \right\}, \quad \mat{x}, \mat{x}' \in \mathbb{R}^d. \label{eqn:real-gp-post-cov-approx} 
\end{align}
Then, we define a real-valued GP with mean function $\hat{\mu}_{N, \mathbb{R}}$ and covariance function  $\hat{k}_{N, \mathbb{R}}$:
$$
f | \mat{y} \sim \mathcal{GP}( \hat{\mu}_{N, \mathbb{R}}, \hat{k}_{N, \mathbb{R}} ).
$$
We use this approximate GP for prediction tasks in our experiments. 

Note that naive computations of \cref{eqn:complex-gp-post-mean-approx} and
\cref{eqn:complex-gp-post-cov-approx} require $\mathcal{O}(N^3 + N^2 D)$ complexity, and thus do not leverage the computational advantage of random features. 
We will show next how to reformulate \cref{eqn:complex-gp-post-mean-approx} and
\cref{eqn:complex-gp-post-cov-approx} to compute them in $\mathcal{O}(D^3 + N D^2)$, which is  linear in the number of training data points $N$.

\subsection{Computationally Efficient Implementation}

\label{sec:GP-efficient-implement}

We describe how to efficiently compute the approximate posterior mean and covariance functions in 
\cref{eqn:complex-gp-post-mean-approx} and
\cref{eqn:complex-gp-post-cov-approx}, respectively. 
To this end, recall the notation in \cref{eq:notation-feat-map-kernel}.
 Let $\sigma^{-1} := (\sigma_1^{-1}, \dots, \sigma_N^{-1})^\top \in \mathbb{R}^N$ and  $\sigma^{-2} := (\sigma_1^{-2}, \dots, \sigma_N^{-2})^\top \in \mathbb{R}^N$.
 
First we deal with \cref{eqn:complex-gp-post-mean-approx}.
For a matrix  $A \in \mathbb{C}^{N \times D}$,  we have $ (A^H A + \mat{I}_N) A^H = A^H (A A^H + \mat{I}_D)$, and thus $A^H (A A^H + \mat{I}_N)^{-1} = (A^H A + \mat{I}_D)^{-1} A^H$. 
By using this last identity with $A = {\rm diag}(\mat{\sigma}^{-1}) \Phi(X) \in \mathbb{C}^{N \times D}$,   we can rewrite \cref{eqn:complex-gp-post-mean-approx}  as
\begin{align}
    \hat{\mu}_N( \mat{x} ) &=  \hat{\mat{k}}(\mat{x})^H ( \hat{\mat{K}} + \diag{\mat{\sigma}^2})^{-1} \mat{y}, \nonumber  \\
   &= \Phi(\mat{x})^\top \Phi(\mat{X})^H \left(\Phi(\mat{X}) \Phi(\mat{X})^H + \diag{\mat{\sigma}^2} \right)^{-1} \mat{y} \nonumber \\
    &= \Phi(\mat{x})^\top \Phi(\mat{X})^H \diag{\mat{\sigma}^{-1}} \left(\diag{\mat{\sigma}^{-1}} \Phi(\mat{X}) \Phi(\mat{X})^H \diag{\mat{\sigma}^{-1}} + \mat{I}_N \right)^{-1} \diag{\mat{\sigma}^{-1}} \mat{y} \nonumber \\
    &= \Phi(\mat{x})^\top \left (\Phi(\mat{X})^H \diag{\mat{\sigma}^{-2}} \Phi(\mat{X}) + \mat{I}_D \right)^{-1} \Phi(\mat{X})^H \diag{\mat{\sigma}^{-2}} \mat{y} . \label{eq:approx-post-mean-reform}
\end{align}

Next we deal with \cref{eqn:complex-gp-post-cov-approx}. 
For matrices $A, C, U, V$ of appropriate sizes with $A$ invertible, the Woodbury matrix identity states that $A^{-1} - A^{-1} U (C^{-1} + V A^{-1} U)^{-1} V A^{-1} = (A + UCV)^{-1}$.
By using the Woodbury identity with $A = \mat{I}_D$, $C = {\rm diag}(\mat{\sigma}^{-2})$, $U = \Phi(X)^H$ and $V = \Phi(X)$,  we can rewrite \cref{eqn:complex-gp-post-cov-approx} as
\begin{align}
\hat{k}_N(\mat{x}, \mat{x}') &= \hat{k}(\mat{x}, \mat{x}')  - \hat{\mat{k}}(\mat{x})^H ( \hat{\mat{K}} + \diag{\mat{\sigma}^2})^{-1} \hat{\mat{k}}(\mat{x}) \nonumber \\
    &= \Phi(\mat{x})^\top \overline{\Phi(\mat{x})} - \Phi(\mat{x})^\top \Phi(\mat{X})^H \left( \Phi(\mat{X}) \Phi(\mat{X})^H + \diag{\mat{\sigma}^2} \right)^{-1} \Phi(\mat{X}) \overline{\Phi(\mat{x})} \nonumber \\
    &= \Phi(\mat{x})^\top \left(\mat{I}_D - \Phi(\mat{X})^H \left(\diag{\mat{\sigma}^2} + \Phi(\mat{X}) \Phi(\mat{X})^H \right)^{-1} \Phi(\mat{X}) \right) \overline{\Phi(\mat{x})} \nonumber \\
    &= \Phi(\mat{x})^\top \left( \mat{I}_D + \Phi(\mat{X})^H \diag{\mat{\sigma}^{-2}} \Phi(\mat{X}) \right)^{-1} \overline{\Phi(\mat{x})}. \label{eq:approx-post-cov-reform}
\end{align}

We now study the costs of computing \cref{eq:approx-post-mean-reform} and  \cref{eq:approx-post-cov-reform}.
For both \cref{eq:approx-post-mean-reform} and  \cref{eq:approx-post-cov-reform}, the bottleneck is the computation of the inverse of the following matrix.
\begin{align}
    \label{eqn:gp-bottleneck}
    \mat{B} := \Phi(\mat{X})^H \diag{\mat{\sigma}^{-2}} \Phi(\mat{X}) + \mat{I}_D \in \mathbb{C}^{D \times D}.
\end{align}
The time complexity of computing $\mat{B}$  is $\bigO(N D^2)$, and that of the inverse $\mat{B}^{-1}$  is $\bigO(D^3)$, the latter being the complexity of computing  the Cholesky decomposition $\mat{B} = \mat{L} \mat{L}^H$ with $\mat{L} \in \mathbb{C}^{D \times D}$ being a lower triangular matrix. 
Thus, the overall cost of computing $\mat{B}^{-1}$ is $\bigO(N D^2 + D^3)$.

We next conduct a more detailed analysis of  the costs of $\mat{B}$ and its Cholesky decomposition, and compare them with the computational costs for the corresponding matrix inversion using real-valued features (i.e., when $\Phi(X) \in \mathbb{R}^{N \times D}$).
Below we use the shorthand $\tilde{\Phi}(\mat{X}) := \diag{\mat{\sigma}^{-1}} \Phi(\mat{X})$ so that $\mat{B} = \tilde{\Phi}(\mat{X})^H \tilde{\Phi}(\mat{X})  + \mat{I}_D$.
Then the real and imaginary parts of $\mat{B}$ can be written as
\begin{align*}
    \mathcal{R} \{ \mat{B} \}
    &= \mathcal{R} \{ \tilde{\Phi}(\mat{X})^H \tilde{\Phi}(\mat{X}) \} + \mat{I}_D
    = \mathcal{R}\{ \tilde{\Phi}(\mat{X}) \}^{\top} \mathcal{R}\{ \tilde{\Phi}(\mat{X}) \}
    + \mathcal{I}\{ \tilde{\Phi}(\mat{X}) \}^{\top} \mathcal{I}\{ \tilde{\Phi}(\mat{X}) \} + \mat{I}_D \\
    \mathcal{I} \{ \mat{B} \}
    &= \mathcal{I} \{ \tilde{\Phi}(\mat{X})^H \tilde{\Phi}(\mat{X}) \}
    = \mathcal{R}\{ \tilde{\Phi}(\mat{X}) \}^{\top} \mathcal{I}\{ \tilde{\Phi}(\mat{X}) \}
    - \mathcal{I}\{ \tilde{\Phi}(\mat{X}) \}^{\top} \mathcal{R}\{ \tilde{\Phi}(\mat{X}) \}.
\end{align*}
 Since $(\mathcal{R}\{ \tilde{\Phi}(\mat{X}) \}^{\top} \mathcal{I}\{ \tilde{\Phi}(\mat{X}) \})^{\top} = \mathcal{I}\{ \tilde{\Phi}(\mat{X}) \}^{\top} \mathcal{R}\{ \tilde{\Phi}(\mat{X}) \}$,  one can compute $ \mathcal{I} \{ \mat{B} \}$ by only computing $\mathcal{R}\{ \tilde{\Phi}(\mat{X}) \}^{\top} \mathcal{I}\{ \tilde{\Phi}(\mat{X})\} $.
Therefore, the computation of $\mat{B}$ requires the computations of the three real  $D$-by-$D$ matrices (i.e.,  $\mathcal{R}\{ \tilde{\Phi}(\mat{X}) \}^{\top} \mathcal{R}\{ \tilde{\Phi}(\mat{X}) \}$,
    $\mathcal{I}\{ \tilde{\Phi}(\mat{X}) \}^{\top} \mathcal{I}\{ \tilde{\Phi}(\mat{X}) \}$, and   $\mathcal{R}\{ \tilde{\Phi}(\mat{X}) \}^{\top} \mathcal{I}\{ \tilde{\Phi}(\mat{X})\}$).
Thus, the total number of operations for computing $\mat{B}$ is  $3 \cdot (ND^2) + 2 \cdot D^2$, where $3 \cdot (ND^2)$ is operations for the matrix products  and  $2 \cdot D^2$ for the addition and subtraction inside  $\mathcal{R}\{\mat{B}\}$ and $\mathcal{I}\{\mat{B}\}$, respectively.
Hence, assuming $N \gg D$,  the computational cost for $\mat{B}$ is roughly $3$ times more expensive than the corresponding cost when $\Phi$ is real-valued.

Computing the Cholesky decomposition of a $D$ by $D$ matrix requires roughly $\frac{1}{6} D^3$ subtractions and $\frac{1}{6} D^3$ multiplications \citep[e.g.,][p. 175]{trefethen97}. 
Therefore, when $\Phi$ is real-valued (and thus $\mat{B}$ is real-valued), the Cholesky decomposition of $\mat{B}$ requires $\frac{1}{6} D^3 + \frac{1}{6} D^3 = \frac{1}{3} D^3$ FLOPS. 
On the other hand, when $\Phi$ is complex-valued, the Cholesky decomposition of $\mat{B}$ requires $\frac{4}{3} D^3$ FLOPS:  one complex subtraction requires 2 real subtractions, and thus subtractions in total require $ \frac{1}{6} D^3 \times 2  = \frac{1}{3} D^3$ FLOPS; one complex multiplication requires 4 real multiplications and 2 real subtractions, and thus multiplications in total require $\frac{1}{6} D^3 \times 6 = D^3$ FLOPS; thus  $\frac{1}{3}D^3 + D^3 = \frac{4}{3} D^3$ FLOPS in total.
Thus, the cost for computing the Cholesky decomposition of $\mat{B}$ when $\Phi$ is complex-valued is  $4$ times more expensive than the real-valued case.
 
The memory requirement for the complex case is 2 times as large as the real case, since the complex case requires storing both real and imaginary parts.

Note that, if one uses a $2D$-dimensional {\em real} feature map (i.e., $\Phi(\mat{X}) \in \mathbb{R}^{N \times 2D}$), then this requires 4 times as much memory, 4 times as many operations to compute the matrix $\mat{B}$, and $8$ times as many operations for the Cholesky decomposition of $\mat{B}$ as those required for a $D$-dimensional real feature map.
Therefore, using a $2D$-dimensional real feature map is computationally more expensive than using a $D$-dimensional complex feature map, since the latter only requires $2$ times as much memory, $3$ times as many operations for computing $\mat{B}$, and $4$ times as many operations for computing the Cholesky decomposition of $\mat{B}$ as those required for a $D$-dimensional real feature map, as shown above.
Note also that the performance improvement from using a $D$-dimensional complex feature map is typically larger than using a $2D$-dimensional real feature map; see the experiments in Section \ref{sec:experiments}.

\subsection{GP Classification as Closed-form Multi-output Regression}

We now describe the GP classification approach of \citet{Milios2018}, and how to use approximate posteriors for GP regegression in this approach.

We assume that there are $C \in \mathbb{N}$ classes and that output labels are expressed by one-hot encoding.
Thus, for each class $c \in \{1,\dots,C \}$ and each training input $\mat{x}_i \in \mathbb{R}^d$ with $i=1,\dots,N$, there exist an output  $y_{c,i} \in \{0, 1\}$ such that $y_{c,i} = 1$ if $\mat{x}_i$ belongs to class $c$ and $y_{c,i} = 0$ otherwise.

\paragraph{The approach of \citet{Milios2018}.}
Let $\alpha > 0$ be a constant. 
For each class $c \in \{1,\dots,C\}$, \citet{Milios2018}  define transformed versions $\tilde{y}_{c,1}, \dots, \tilde{y}_{c,N} \in \mathbb{R}$ of the training outputs $ y_{c,1}, \dots, y_{c,N}$ as 
$$
    \tilde{y}_{c,i} := \log ( y_{c,i} + \alpha ) - \sigma_{c,i}^2 / 2,   \quad \text{where} \quad \sigma_{c,i}^2 := \log ( (y_{c,i} + \alpha)^{-1} + 1 ), \quad i=1,\dots,N.
$$
\citet{Milios2018} then define an observation model of $\tilde{y}_{c, 1}, \dots, \tilde{y}_{c, N}$ as  
\begin{equation} \label{eq:GP-milos-obs-model}
\tilde{y}_{c,i} = f_c (\mat{x}_i) + \varepsilon_{c, i},  \quad i=1,\dots,N,
\end{equation}
where  $f_c : \mathbb{R}^d \to \mathbb{R}$ is a latent function and $\varepsilon_{c,i} \sim  \mathcal{N}(0, \sigma_{c,i}^2)$ is an independent Gaussian noise with variance $\sigma^2_{c,i}$.
\citet{Milios2018} propose to model $f_c$ for each $c \in \{1,\dots,C\}$ independently as a GP: 
\begin{equation} \label{eq:GP-milos-prior}
f_c \sim \mathcal{GP}(0, k),
\end{equation}
where $k:\mathbb{R}^d \times \mathbb{R}^d \to \mathbb{R}$ is a kernel. 
  \cref{eq:GP-milos-obs-model} and \cref{eq:GP-milos-prior}  define the joint distribution of the latent function $f_c$ and the transformed labels $\tilde{y}_{c, 1}, \dots, \tilde{y}_{c, N}$.
Thus, conditioning on $\tilde{y}_{c, 1}, \dots, \tilde{y}_{c, N}$, one obtains a GP posterior of $f_c$.  
In other words, one can obtain a GP posterior of $f_c$  by performing GP regression for each class $c \in \{1,\dots,C\}$ using $(\mat{x}_i, \tilde{y}_{c, i})_{i=1}^N$ as training data.

 The constant $\alpha$ is a hyperparameter, which  \citet{Milios2018} propose to choose  by cross validation, using the Mean Negative Log Likelihood (MNLL) \citep[e.g.,][p. 23]{Rasmussen2006} as an evaluation criterion.

\paragraph{Using approximate GP posteriors.}
We now explain how to use approximate posteriors for GP regression in the above approach: For each class $c \in \{1,\dots,C\}$, we perform approximate GP regression using $(\mat{x}_i, \tilde{y}_{c, i})_{i=1}^N$ as training data, to obtain an approximate GP posterior for the latent function $f_c$ in \cref{eq:GP-milos-obs-model}.  
For instance, with our approach on approximate GP regression using complex random features in \cref{sec:GP-regress-comp-features}, we obtain a GP posterior $f_c \sim  \mathcal{GP}(\hat{\mu}_{N, \mathbb{R}, c}, \hat{k}_{N, \mathbb{R}, c})$ for each class $c \in \{1,\dots, C\}$, where $\hat{\mu}_{N, \mathbb{R}, c}: \mathbb{R}^d \to \mathbb{R}$ and $\hat{k}_{N, \mathbb{R}, c}: \mathbb{R}^d \times \mathbb{R}^d \to \mathbb{R}$ are the approximate GP posterior mean and covariance functions in \cref{eqn:real-gp-post-mean-approx} and 
\cref{eqn:real-gp-post-cov-approx}, respectively, with $\mat{y} := (\tilde{y}_{c,1}, \dots, \tilde{y}_{c,N})^\top$ and $\mat{\sigma}^2 := ( \sigma_{c,1}^2, \dots,  \sigma_{c,N}^2)^\top$.

For a given test input $\mat{x} \in \mathbb{R}^d$, one can obtain its posterior predictive probabilities over the $C$ classes in the following way.
For each class $c \in \{1, \dots, C\}$, we first generate a sample $z_c \in \mathbb{R}$  from the posterior distribution of the latent function value $f_c(\mat{x})$.
We then apply the softmax transformation to $z_1, \dots, z_C$ to obtain  probabilities $p_1, \dots, p_C \geq 0$ over the $C$ class labels:  $p_c  := \exp(z_{c}) / \sum_{j=1}^C \exp(z_j) $. 
\citet{Milios2018} show that these probabilities $p_1, \dots, p_C$ are approximately a sample from a Dirichlet distribution, yielding well-calibrated predictions.

\subsection{Kullback-Leibler (KL) Divergence}
In the experiments in Section \ref{sec:experiments}, we use the Kullback-Leibler (KL) divergence between the exact and approximate GP posteriors, to evaluate the quality of each approximation approach. 
Let $\mu_{\rm exact}( \mat{x} )$ and $\sigma_{\rm exact}^2(\mat{x} )$  be the posterior mean and variance at $\mat{x} \in \mathbb{R}^d$ from the exact GP posterior, and  let  $\mu_{\rm appr}( \mat{x} )$ and $\sigma_{\rm appr}^2(\mat{x} )$ be those from an approximate GP posterior. 
Let  $\mat{x}_{*,1}, \dots, \mat{x}_{*,m_*} \in \mathbb{R}^d$ be test input points. 
Define $\mat{\mu}_{\rm exact} := ( \mu_{\rm exact}( \mat{x}_{*, 1} ), \dots, \mu_{\rm exact}( \mat{x}_{*, m_*} )  )^\top$, $\mat{\sigma}_{\rm exact}^2 := ( \sigma_{\rm exact}^2 ( \mat{x}_{*, 1} ), \dots, \sigma_{\rm exact}^2 ( \mat{x}_{*, m_*} )  )^\top$, $\mat{\mu}_{\rm appr} := ( \mu_{\rm appr}( \mat{x}_{*, 1} ), \dots, \mu_{\rm appr}( \mat{x}_{*, m_*} )  )^\top$, and $\mat{\sigma}_{\rm appr}^2 := ( \sigma_{\rm appr}^2 ( \mat{x}_{*, 1} ), \dots, \sigma_{\rm appr}^2 ( \mat{x}_{*, m_*} )  )^\top$. 

We then measure the KL divergence between two diagonal Gaussian distributions, $\mathcal{N}(\mat{\mu}_{\rm appr}, {\rm diag}( \mat{\sigma}_{\rm appr}^2) )$ and $\mathcal{N}(\mat{\mu}_{\rm exact}, {\rm diag}( \mat{\sigma}_{\rm exact}^2) )$:  
\begin{align}
&    KL \left[ \mathcal{N}(\mat{\mu}_{\rm appr}, {\rm diag}( \mat{\sigma}_{\rm appr}^2) ) \ || \ \mathcal{N}(\mat{\mu}_{\rm exact}, {\rm diag}( \mat{\sigma}_{\rm exact}^2) ) \right]  \nonumber \\ 
& = \frac{1}{2} 
    \sum_{i=1}^{m_*} \left( \frac{ \sigma_{\rm exact}^2(\mat{x}_{*,i }) }{ \sigma_{\rm appr}^2(\mat{x}_{*,i }) }
    + \log \frac{ \sigma_{\rm exact}^2(\mat{x}_{*,i }) }{ \sigma_{\rm appr}^2(\mat{x}_{*,i }) }
    - 1
    + \frac{(\mu_{\rm exact}(\mat{x}_{*,i} ) - \mu_{\rm appr}(\mat{x}_{*,i} )  )^2}{ \sigma_{\rm appr}^2(\mat{x}_{*,i }) }
    \right), \label{eqn:kl-div}
\end{align}
We consider these diagonal Gaussian distributions, since the focus of our experiments in Section \ref{sec:experiments} is the prediction performance at test input points $\mat{x}_{*,1}, \dots, \mat{x}_{*,m_*} \in \mathbb{R}^d$.

\section{Additional Experiments}
\label{sec:additional-experiments}

We present here additional experimental results, supplementing those in Section \ref{sec:experiments}. 
\cref{tbl:centered-vs-non-centered} shows the effects of applying zero-centering to input vectors in the polynomial kernel approximation experiments. 
\cref{fig:poly-regression-boston-kin8nm}, \cref{fig:poly-regression-naval-yacht} and \cref{fig:rbf-regression} show the results of additional experiments on GP regression. \cref{fig:poly-classification-eeg-mnist} and \cref{fig:rbf-class2} show the results of additional experiments on GP classification.

\begin{figure}[ht]
    \centering
    \includegraphics[width=1\textwidth]{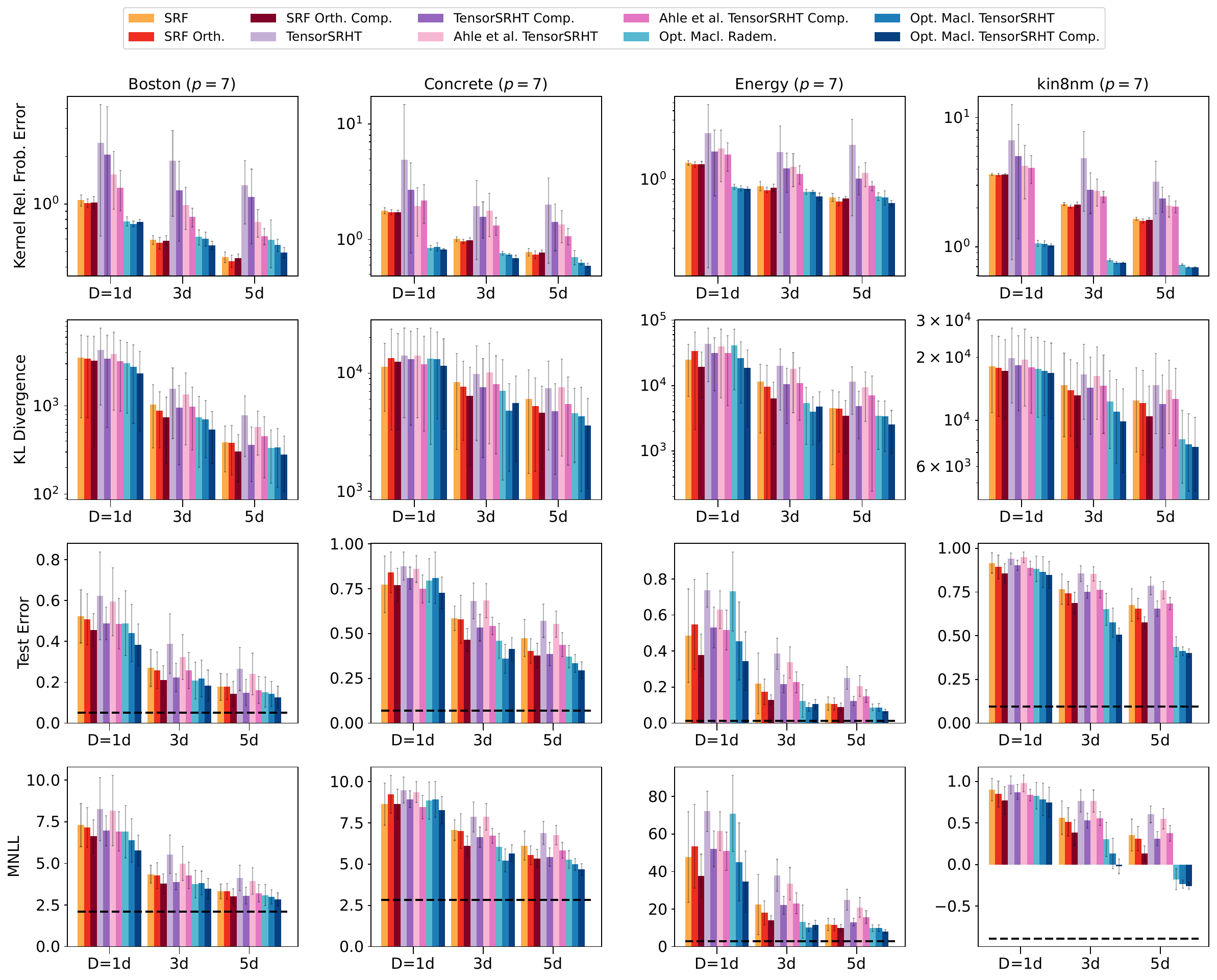}
    \caption{%
    Additional results of the experiments in Section \ref{sec:GP-inference-poly} on approximate GP regression with a $p=7$ polynomial kernel.
    Lower values are better for all the metrics. 
    For each dataset, we show the number of random features $D \in \{1d, 3d, 5d\}$ used in each method on the horizontal axis, with $d$ being the input dimensionality of the dataset. The dashed black line shows test errors and MNLL values for the full target GP.
    In some cases performance is worse than the actual kernel because $D=5d$ is still too small for some datasets. 
    However, there is an indication that the test errors improve as we increase the number of features getting us closer to the true GP.
 } 
    \label{fig:poly-regression-boston-kin8nm}
\end{figure}

\begin{figure}[ht]
    \centering
    \includegraphics[width=1\textwidth]{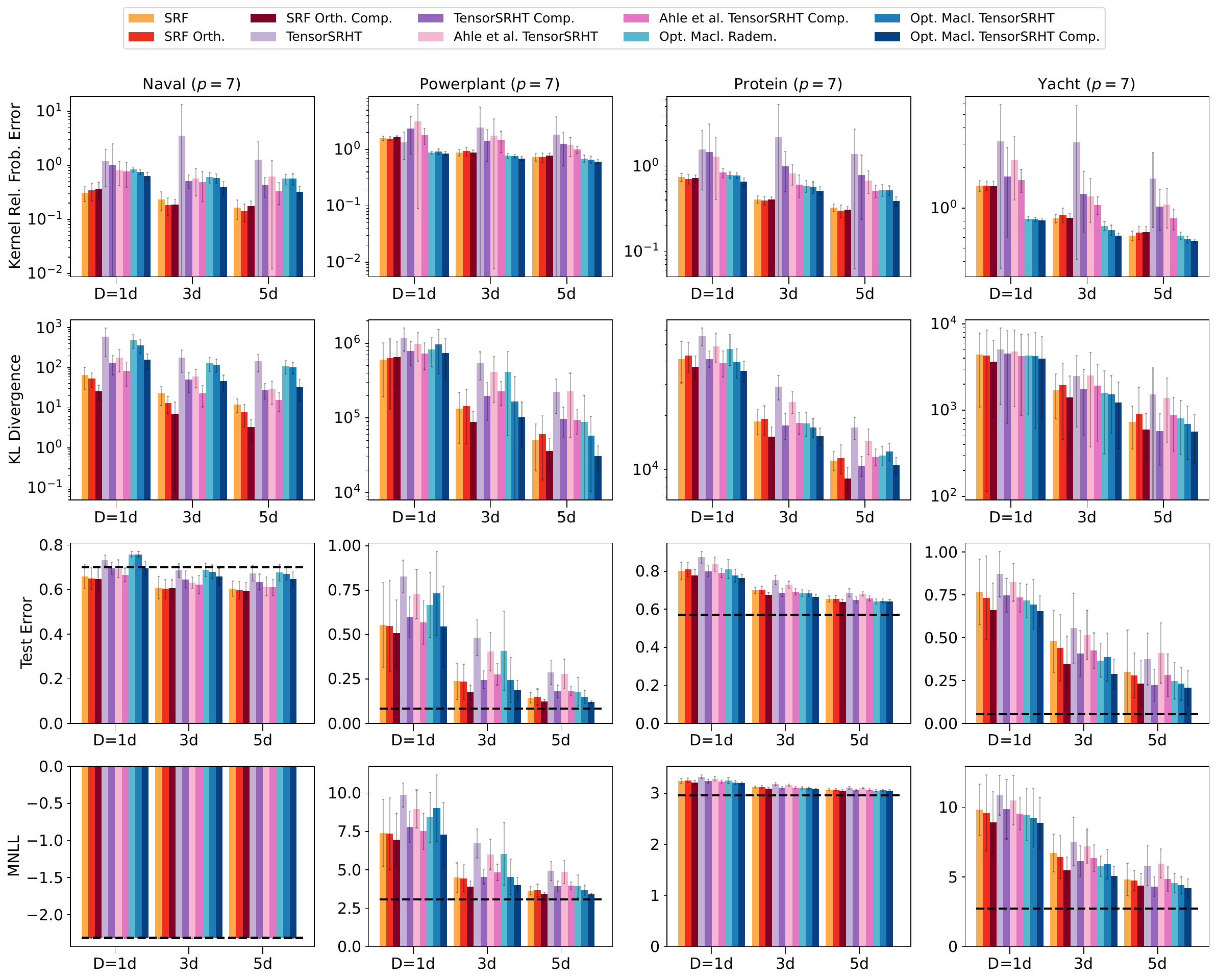}
    \caption{%
    Additional results of the experiments in Section \ref{sec:GP-inference-poly} on approximate GP regression with a $p=7$ polynomial kernel.
    Lower values are better for all the metrics. 
    For each dataset, we show the number of random features $D \in \{1d, 3d, 5d\}$ used in each method on the horizontal axis, with $d$ being the input dimensionality of the dataset. The dashed black line shows test errors and MNLL values for the full target GP.
    Interestingly, in the Naval data set, the sketching approaches provide some form of regularization, yielding better generalization error than the GP with full kernel.  
    }
    \label{fig:poly-regression-naval-yacht}
\end{figure}

\begin{figure}[ht]
    \centering
    \includegraphics[width=1\textwidth]{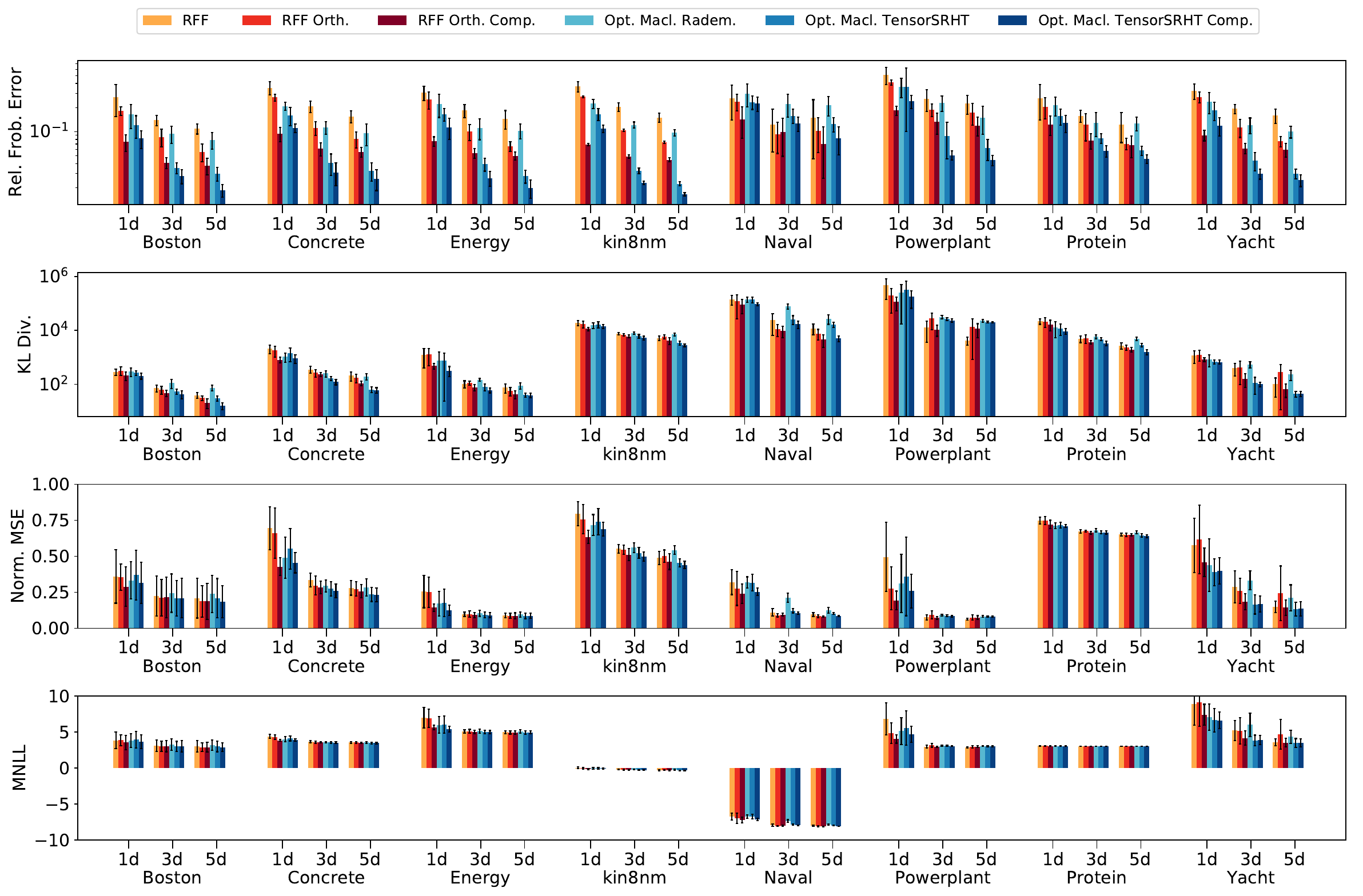}
    \caption{%
    Additional results of the experiments in Section \ref{sec:gaussian-approximation} on approximate GP regression with a Gaussian kernel.
        Lower values are better for all the metrics. 
    For each dataset, we show the number of random features $D \in \{1d, 3d, 5d\}$ used in each method on the horizontal axis, with $d$ being the input dimensionality of the dataset. We put the legend labels and the bars in the same order.
    }
    \label{fig:rbf-regression}
\end{figure}
\begin{table}[ht]
\begin{center}

\resizebox{1 \textwidth}{!}{
\begin{tabular}{l|cc|cc|cc|cc}
\toprule
&
\multicolumn{4}{c|}{MNLL} &
\multicolumn{4}{c}{Rel. Frob. Error} \\

&
\multicolumn{2}{c|}{Non-centred} &
\multicolumn{2}{c|}{Centred} &
\multicolumn{2}{c|}{Non-centred} &
\multicolumn{2}{c}{Centred} \\

Dataset &
SRF Gaus. & Opt. Macl. Rad. &
SRF Gaus. & Opt. Macl. Rad. &
SRF Gaus. & Opt. Macl. Rad. &
SRF Gaus. & Opt. Macl. Rad. \\
\midrule

Boston &
3.410$\pm$0.37 & 3.447$\pm$0.38 &
3.449$\pm$0.62 & \textbf{3.161}$\pm$0.28 &
\textbf{0.044}$\pm$0.02 & 0.212$\pm$0.15 &
0.356$\pm$0.05 & 0.421$\pm$0.06 \\

Concrete &
3.779$\pm$0.07 & 3.811$\pm$0.04 &
3.660$\pm$0.12 & \textbf{3.542}$\pm$0.07 &
\textbf{0.019}$\pm$0.01 & 0.276$\pm$0.17 &
0.610$\pm$0.07 & 0.482$\pm$0.03 \\

Energy &
6.090$\pm$0.12 & 6.090$\pm$0.12 &
5.116$\pm$0.20 & \textbf{5.012}$\pm$0.13 &
\textbf{0.003}$\pm$0.00 & 0.222$\pm$0.14 &
0.507$\pm$0.08 & 0.484$\pm$0.05 \\

kin8nm &
-0.203$\pm$0.07 & -0.310$\pm$0.03 &
-0.203$\pm$0.07 & \textbf{-0.323}$\pm$0.03 &
0.946$\pm$0.04 & 0.525$\pm$0.04 &
0.947$\pm$0.03 & \textbf{0.521}$\pm$0.03 \\

Naval &
-6.069$\pm$0.03 & -6.066$\pm$0.03 &
\textbf{-8.083}$\pm$0.04 & -7.788$\pm$0.10 &
\textbf{0.040}$\pm$0.03 & 0.183$\pm$0.06 &
0.112$\pm$0.04 & 0.384$\pm$0.11 \\

Powerplant &
3.064$\pm$0.03 & \textbf{3.061}$\pm$0.06 &
3.282$\pm$0.14 & 3.400$\pm$0.73 &
\textbf{0.001}$\pm$0.00 & 0.062$\pm$0.04 &
0.609$\pm$0.09 & 0.527$\pm$0.10 \\

Protein &
3.233$\pm$0.01 & 3.233$\pm$0.01 &
3.072$\pm$0.02 & \textbf{3.060}$\pm$0.02 &
\textbf{0.000}$\pm$0.00 & 0.002$\pm$0.00 &
0.277$\pm$0.05 & 0.429$\pm$0.14 \\

Yacht &
4.317$\pm$0.45 & 4.478$\pm$0.45 &
\textbf{3.773}$\pm$0.21 & 3.844$\pm$0.28 &
\textbf{0.028}$\pm$0.01 & 0.276$\pm$0.11 &
0.512$\pm$0.04 & 0.484$\pm$0.03 \\

\midrule

Cod\_rna &
0.307$\pm$0.00 & 0.308$\pm$0.00 &
0.288$\pm$0.06 & \textbf{0.151}$\pm$0.01 &
\textbf{0.022}$\pm$0.01 & 0.087$\pm$0.05 &
0.641$\pm$0.05 & 0.467$\pm$0.05 \\

Covertype &
0.821$\pm$0.01 & - &
0.650$\pm$0.01 & \textbf{0.639}$\pm$0.01 &
\textbf{0.024}$\pm$0.01 & - &
0.361$\pm$0.01 & 0.300$\pm$0.01 \\

Drive &
1.446$\pm$0.02 & 1.453$\pm$0.03 &
0.677$\pm$0.02 & \textbf{0.497}$\pm$0.01 &
\textbf{0.068}$\pm$0.02 & 0.135$\pm$0.05 &
0.348$\pm$0.01 & 0.312$\pm$0.02 \\

FashionMNIST &
\textbf{0.353}$\pm$0.00 & 0.364$\pm$0.00 &
0.364$\pm$0.00 & 0.361$\pm$0.00 &
\textbf{0.029}$\pm$0.00 & 0.062$\pm$0.01 &
0.099$\pm$0.00 & 0.104$\pm$0.01 \\

Magic &
0.453$\pm$0.01 & 0.452$\pm$0.01 &
0.381$\pm$0.02 & \textbf{0.350}$\pm$0.01 &
\textbf{0.068}$\pm$0.01 & 0.147$\pm$0.05 &
0.430$\pm$0.03 & 0.418$\pm$0.04 \\

Miniboo &
0.253$\pm$0.01 & - &
0.239$\pm$0.01 & \textbf{0.213}$\pm$0.01 &
\textbf{0.027}$\pm$0.01 & - &
0.214$\pm$0.01 & 0.229$\pm$0.02 \\

MNIST &
0.076$\pm$0.00 & \textbf{0.074}$\pm$0.00 &
0.290$\pm$0.02 & 0.353$\pm$0.09 &
\textbf{0.073}$\pm$0.00 & 0.082$\pm$0.00 &
0.085$\pm$0.00 & 0.089$\pm$0.01 \\

Mocap &
0.360$\pm$0.01 & 0.334$\pm$0.01 &
0.357$\pm$0.02 & \textbf{0.289}$\pm$0.01 &
\textbf{0.115}$\pm$0.01 & 0.187$\pm$0.04 &
0.414$\pm$0.01 & 0.290$\pm$0.02 \\

\bottomrule
\end{tabular}}
\end{center}
\caption{GP regression (top) and classification (bottom) for centred vs. non-centred data with $D=5d$ features. Non-centred Miniboo and Covertype led to numerical issues for Maclaurin (no scores reported).}
\label{tbl:centered-vs-non-centered}
\end{table}

\begin{figure}[t]
    \centering
    \includegraphics[width=0.93\textwidth]{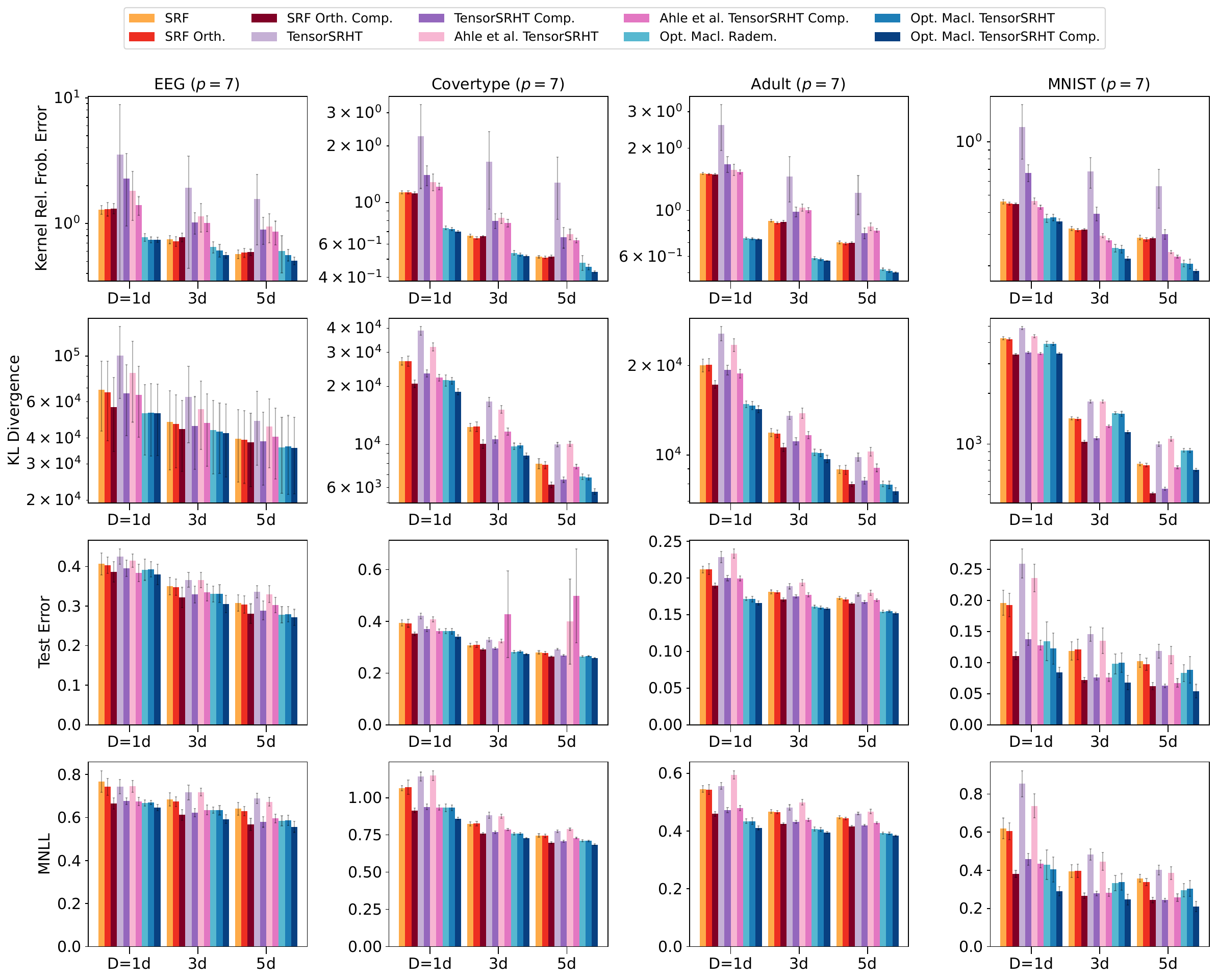}
    \caption{%
    Additional results of the experiments in Section \ref{sec:GP-inference-poly} on approximate GP classification with a $p=7$ polynomial kernel.  
    Lower values are better for all the metrics. 
    For each dataset, we show the number of random features $D \in \{1d, 3d, 5d\}$ used in each method on the horizontal axis, with $d$ being the input dimensionality of the dataset.
    }
    \label{fig:poly-classification-eeg-mnist}
\end{figure}
\begin{figure}[t]
    \centering
    \includegraphics[width=0.93\textwidth]{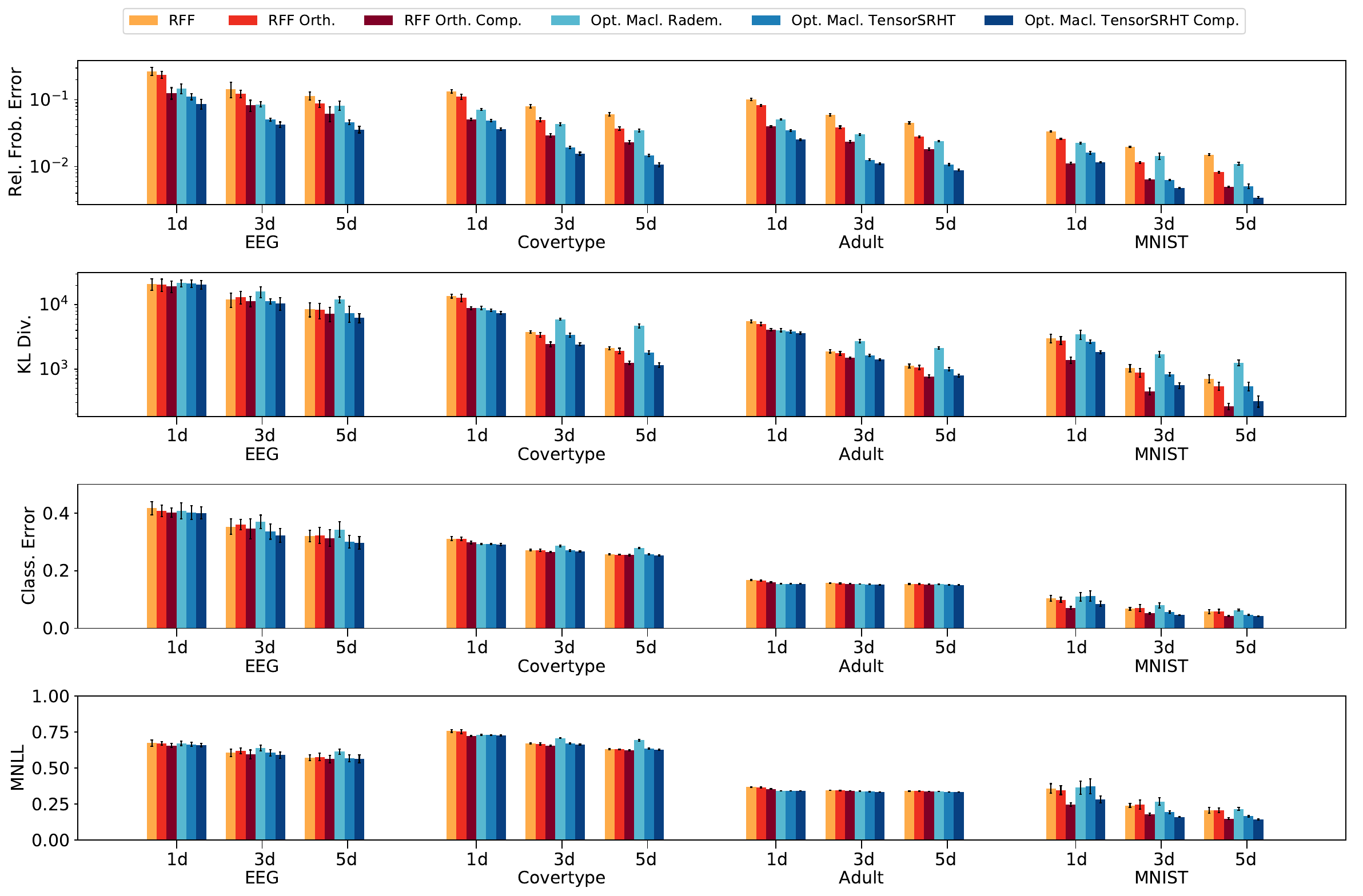}
    \caption{%
    Additional results of the experiments in Section \ref{sec:gaussian-approximation} on approximate GP classification with a Gaussian kernel.
        Lower values are better for all the metrics. 
    For each dataset, we show the number of random features $D \in \{1d, 3d, 5d\}$ used in each method on the horizontal axis, with $d$ being the input dimensionality of the dataset. We put the legend labels and the bars in the same order.
    }
    \label{fig:rbf-class2}
\end{figure}

\section{Additional Results for the Optimized Maclaurin Method} 
\label{app:opt-maclaurin}

We present here additional results for \cref{alg:extended-incremental-algorithm} in \cref{sec:improving-maclaurin}.
In \cref{fig:dp-opt-comparison-fmnist-rademacher} and \cref{fig:dp-opt-comparison-fmnist-srht} we analyze the output of the optimization phase involved in the Maclaurin approximation with Rademacher and TensorSRHT sketches, respectively. We focus on the approximation of the Gaussian kernel and determine $p^*$ using $p_{\rm min}=1$ and $p_{\rm max}=20$. The algorithm is repeated for 20 different random seeds, and the resulting histogram of $p^*$ is shown in the upper figures of \cref{fig:dp-opt-comparison-fmnist-rademacher} and \cref{fig:dp-opt-comparison-fmnist-srht}; the bottom figures show the average and the standard deviation of the number of features assigned to each degree.

We observe that we need relatively few data samples (only 2000 out of 60000 training samples for FashionMNIST) to achieve a stable feature distribution. 
This supports our choice in the experiments in \cref{sec:experiments}, where we used considerably more (5000) data samples to estimate this.
The value of $p^*$ converges more slowly because sometimes very few (e.g., 1) random features are still allocated to high polynomial degrees. 
This does not harm the overall feature distribution too much since most features are already allocated to degrees 1-3 even for small sample sizes (e.g., 500 samples).

Interestingly, most features are allocated to low rather than high degrees, which is due to the variance distribution (see also \cref{fig:bootstrap-boxplot}). 
This supports our choice of $p_{\rm max}=10$ in the experiments because most random features are allocated to small degrees. 
We set $p_{\rm min}=2$ in the experiments to exclude purely linear approximations of the non-linear kernel.
The low-degree allocation phenomenon was the same across datasets in Section~\ref{sec:experiments} as long as the data is zero-centered. 
If it is not, this may change as shown in \citet{WackerKES21}. Then high degrees may receive more features than lower ones.

In the case of TensorSRHT, degree 1 is already perfectly approximated when $D_1=d$ random features are used because the TensorSRHT variance always turns out to be zero for $p=1$.
The algorithm therefore starts investing random features into higher degrees once $D_1=d$. 
So the next degree 2 is chosen to be the most dominant (for Rademacher it was degree 1).
This explains the performance gain of Maclaurin TensorSRHT over Maclaurin Rademacher; we invest less random features into the first degree and use them to decrease the variances of other degrees.

\begin{figure}[h]
    \centering
    \includegraphics[width=1\textwidth]{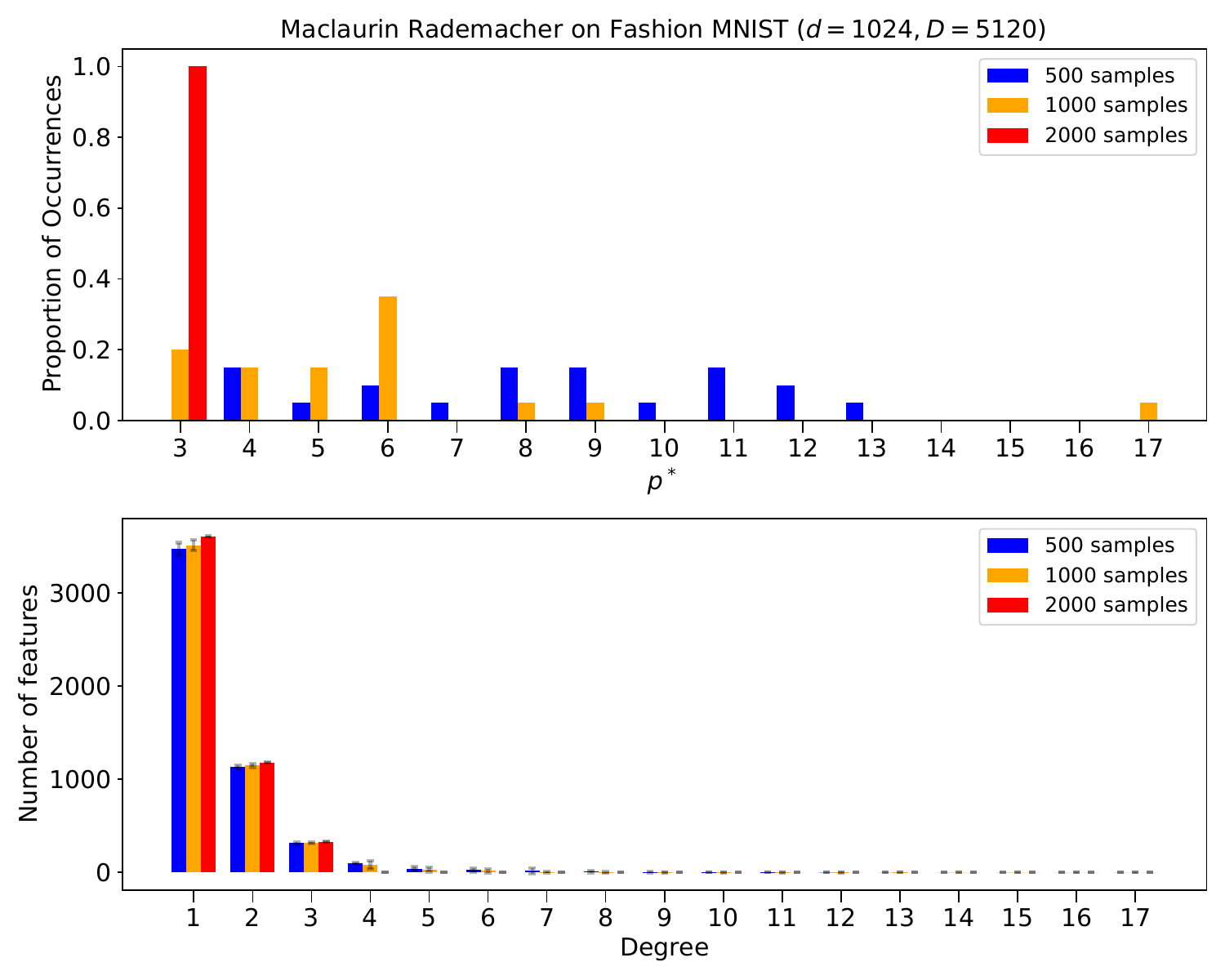}
    \caption{
    Comparing Maclaurin optimization outputs for different sample sizes 500, 1000 and 2000 and real Rademacher sketches over 20 runs of the algorithm. The optimal degree $p^*$ approaches the value of 3 as the sample size increases (20/20 runs resulted in $p^*=3$ for 2000 samples). The feature distribution across degrees allocates most random features to degree 1 followed by degree 2 and so on. This is already stable for small sample sizes (bar heights are mean values and error bars are standard deviations).
    }
    \label{fig:dp-opt-comparison-fmnist-rademacher}
\end{figure}

\begin{figure}[h]
    \centering
    \includegraphics[width=1\textwidth]{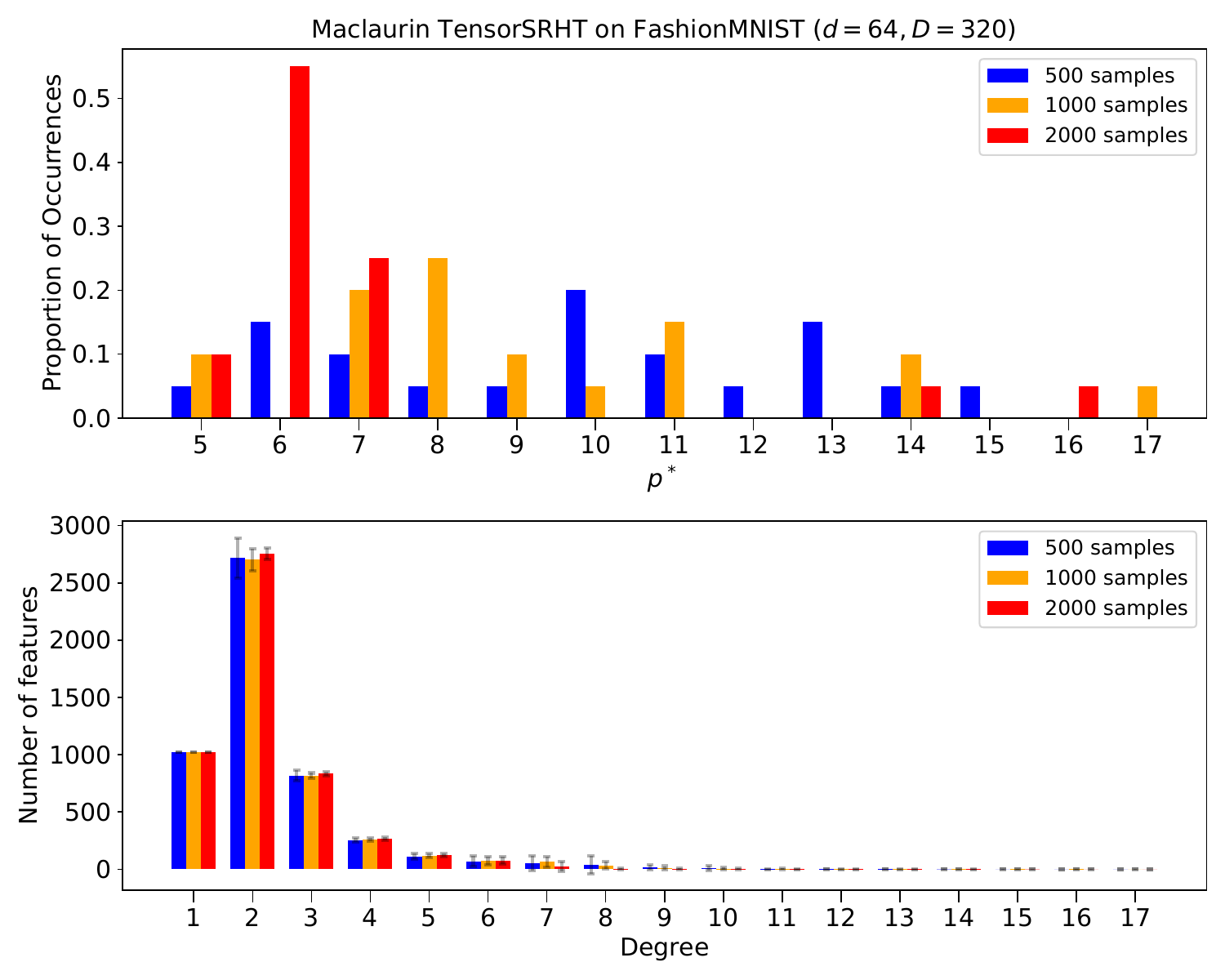}
    \caption{
    Comparing Maclaurin optimization outputs for different sample sizes 500, 1000 and 2000 and real TensorSRHT sketches. The optimal degree $p^*$ approaches the value of 6 as the sample size increases. The feature distribution across degrees allocates most random features to degree 2 this time. This is already stable for small sample sizes.
    }
    \label{fig:dp-opt-comparison-fmnist-srht}
\end{figure}

\paragraph{Time Complexity for Different Values of $p_{\rm min}$ and $p_{\rm max}$.}
Here, we report a theoretical runtime analysis of the Maclaurin method.
In Section \ref{sec:improving-maclaurin}, we have already reported an analysis of the incremental algorithm, giving a time complexity in $\bigO(pD_{\rm total})$. 
In order to find the optimal $p^*$, we need to run the incremental algorithm $p_{\rm max} - p_{\rm min} + 1 \in \bigO(p_{\rm max})$ times.
Additionally, the algorithm requires the \textit{precomputed} variance estimates costing $\bigO(p_{\rm max} m^2)$, where $m$ is the sample size. 
This gives a total time complexity of $\bigO(p_{\rm max}^2 D_{\rm total} + p_{\rm max} m^2)$ for the feature allocation optimization.

Note that a simple Johnsson-Lindenstrauss projection costs $\bigO(ndD)$ for $n$ data points, $d$ input dimensions and $D$ features. 
Note also that $p_{\rm max}^2$ is much smaller than $nd$ though and $p_{\rm max} m^2$ is also small as long as $m \ll n$.
The Maclaurin optimization has therefore a comparable time complexity to a Johnsson-Lindenstrauss projection for reasonable $m$.

\vskip 0.2in
\bibliography{bibliography}

\begin{thebibliography}{52}
\providecommand{\natexlab}[1]{#1}
\providecommand{\url}[1]{\texttt{#1}}
\expandafter\ifx\csname urlstyle\endcsname\relax
  \providecommand{\doi}[1]{doi: #1}\else
  \providecommand{\doi}{doi: \begingroup \urlstyle{rm}\Url}\fi

\bibitem[Abadi et~al.(2016)]{tensorflow2015-whitepaper}
M.~Abadi et~al.
\newblock Tensorflow: Large-scale machine learning on heterogeneous distributed systems.
\newblock \emph{CoRR}, abs/1603.04467, 2016.

\bibitem[Agrawal et~al.(2019)Agrawal, Trippe, Huggins, and Broderick]{Agrawal19a}
R.~Agrawal, B.~Trippe, J.~Huggins, and T.~Broderick.
\newblock The kernel interaction trick: fast bayesian discovery of pairwise interactions in high dimensions.
\newblock In \emph{Proceedings of the 36th International Conference on Machine Learning}, volume~97 of \emph{Proceedings of Machine Learning Research}, pages 141--150. PMLR, 2019.

\bibitem[Ahle et~al.(2020)Ahle, Kapralov, Knudsen, Pagh, Velingker, Woodruff, and Zandieh]{Ahle2020}
T.~D. Ahle, M.~Kapralov, J.~B.~T. Knudsen, R.~Pagh, A.~Velingker, D.~P. Woodruff, and A.~Zandieh.
\newblock Oblivious sketching of high-degree polynomial kernels.
\newblock In \emph{Proceedings of the Thirty-First Annual ACM-SIAM Symposium on Discrete Algorithms}, page 141–160. Society for Industrial and Applied Mathematics, 2020.

\bibitem[Aschard(2016)]{Aschard2016}
H.~Aschard.
\newblock A perspective on interaction effects in genetic association studies.
\newblock \emph{Genetic Epidemiology}, 40(8):\penalty0 678–688, 2016.

\bibitem[Avron et~al.(2014)Avron, Nguyen, and Woodruff]{Avron2014}
H.~Avron, H.~L. Nguyen, and D.~P. Woodruff.
\newblock Subspace embeddings for the polynomial kernel.
\newblock In \emph{Advances in Neural Information Processing Systems 27}, pages 2258--2266. Curran Associates, Inc., 2014.

\bibitem[Blondel et~al.(2016)Blondel, Ishihata, Fujino, and Ueda]{blondel2016polynomial}
M.~Blondel, M.~Ishihata, A.~Fujino, and N.~Ueda.
\newblock Polynomial networks and factorization machines: New insights and efficient training algorithms.
\newblock In \emph{International Conference on Machine Learning}, pages 850--858. PMLR, 2016.

\bibitem[Boloix-Tortosa et~al.(2018)Boloix-Tortosa, Murillo-Fuentes, Payán-Somet, and Pérez-Cruz]{Boloix-Tortosa2015}
R.~Boloix-Tortosa, J.~J. Murillo-Fuentes, F.~J. Payán-Somet, and F.~Pérez-Cruz.
\newblock Complex {G}aussian processes for regression.
\newblock \emph{IEEE Transactions on Neural Networks and Learning Systems}, 29\penalty0 (11):\penalty0 5499--5511, 2018.

\bibitem[Chang et~al.(2010)Chang, Hsieh, Chang, Ringgaard, and Lin]{chang2010training}
Y.-W. Chang, C.-J. Hsieh, K.-W. Chang, M.~Ringgaard, and C.-J. Lin.
\newblock Training and testing low-degree polynomial data mappings via linear {SVM}.
\newblock \emph{Journal of Machine Learning Research}, 11\penalty0 (4), 2010.

\bibitem[Choromanski et~al.(2017)Choromanski, Rowland, and Weller]{Choromanski2017}
K.~Choromanski, M.~Rowland, and A.~Weller.
\newblock The unreasonable effectiveness of structured random orthogonal embeddings.
\newblock In \emph{Advances in Neural Information Processing Systems 31}, pages 218--227. Curran Associates Inc., 2017.

\bibitem[Choromanski et~al.(2018)Choromanski, Rowland, Sarlos, Sindhwani, Turner, and Weller]{Choromanski2018}
K.~Choromanski, M.~Rowland, T.~Sarlos, V.~Sindhwani, R.~Turner, and A.~Weller.
\newblock The geometry of random features.
\newblock In A.~Storkey and F.~Perez-Cruz, editors, \emph{Proceedings of the Twenty-First International Conference on Artificial Intelligence and Statistics}, volume~84, pages 1--9. PMLR, 2018.

\bibitem[Choromanski et~al.(2021)Choromanski, Likhosherstov, Dohan, Song, Gane, Sarlos, Hawkins, Davis, Mohiuddin, Kaiser, Belanger, Colwell, and Weller]{Choromanski21a}
K.~Choromanski, V.~Likhosherstov, D.~Dohan, X.~Song, A.~Gane, T.~Sarlos, P.~Hawkins, J.~Q. Davis, A.~Mohiuddin, L.~Kaiser, D.~B. Belanger, L.~J. Colwell, and A.~Weller.
\newblock Rethinking attention with performers.
\newblock In \emph{Proceedings of the 9th International Conference on Learning Representations}, 2021.

\bibitem[Cotter et~al.(2011)Cotter, Keshet, and Srebro]{Cotter2011}
A.~Cotter, J.~Keshet, and N.~Srebro.
\newblock Explicit approximations of the {G}aussian kernel.
\newblock \emph{CoRR}, abs/1109.4603, 2011.

\bibitem[Dua and Graff(2017)]{Dua:2019}
D.~Dua and C.~Graff.
\newblock {UCI} machine learning repository, 2017.
\newblock URL \url{http://archive.ics.uci.edu/ml}.

\bibitem[Fino and Algazi(1976)]{Fino1976}
B.~J. Fino and V.~R. Algazi.
\newblock Unified matrix treatment of the fast {Walsh-Hadamard} transform.
\newblock \emph{IEEE Transactions on Computers}, 25\penalty0 (11):\penalty0 1142–1146, 1976.

\bibitem[Floudas and Pardalos(2009)]{FloudasChristodoulosA.PardalosPanosM.2009}
C.~A. Floudas and P.~M. Pardalos, editors.
\newblock \emph{Encyclopedia of Optimization, Second Edition}.
\newblock Springer, 2009.

\bibitem[Fukui et~al.(2016)Fukui, Park, Yang, Rohrbach, Darrell, and Rohrbach]{Fukui2016}
A.~Fukui, D.~H. Park, D.~Yang, A.~Rohrbach, T.~Darrell, and M.~Rohrbach.
\newblock Multimodal compact bilinear pooling for visual question answering and visual grounding.
\newblock In \emph{Proceedings of the 2016 Conference on Empirical Methods in Natural Language Processing}, pages 457--468. Association for Computational Linguistics, 2016.

\bibitem[Gao et~al.(2016)Gao, Beijbom, Zhang, and Darrell]{Gao2016}
Y.~Gao, O.~Beijbom, N.~Zhang, and T.~Darrell.
\newblock {Compact bilinear pooling}.
\newblock \emph{Proceedings of the 2016 IEEE Computer Society Conference on Computer Vision and Pattern Recognition}, pages 317--326, 2016.

\bibitem[Garreau et~al.(2017)Garreau, Jitkrittum, and Kanagawa]{garreau2017large}
D.~Garreau, W.~Jitkrittum, and M.~Kanagawa.
\newblock Large sample analysis of the median heuristic.
\newblock \emph{arXiv preprint arXiv:1707.07269}, 2017.

\bibitem[Hamid et~al.(2014)Hamid, Xiao, Gittens, and DeCoste]{Hamid2014}
R.~Hamid, Y.~Xiao, A.~Gittens, and D.~DeCoste.
\newblock Compact random feature maps.
\newblock In \emph{Proceedings of the 31th International Conference on Machine Learning}, volume~32 of \emph{Proceedings of Machine Learning Research}, pages 19--27. PMLR, 2014.

\bibitem[Harris et~al.(2020)]{harris2020array}
C.~R. Harris et~al.
\newblock Array programming with numpy.
\newblock \emph{Nature}, 585:\penalty0 357--362, 2020.

\bibitem[Hensman et~al.(2018)Hensman, Durrande, and Solin]{hensman2017variational}
J.~Hensman, N.~Durrande, and A.~Solin.
\newblock Variational {F}ourier features for {G}aussian processes.
\newblock \emph{Journal of Machine Learning Research}, 18\penalty0 (151):\penalty0 1--52, 2018.

\bibitem[Kar and Karnick(2012)]{Kar2012}
P.~Kar and H.~Karnick.
\newblock Random feature maps for dot product kernels.
\newblock In \emph{Proceedings of the Fifteenth International Conference on Artificial Intelligence and Statistics}, volume~22 of \emph{{JMLR} Proceedings}, pages 583--591. JMLR, 2012.

\bibitem[Lecun et~al.(1998)Lecun, Bottou, Bengio, and Haffner]{lecun1998}
Y.~Lecun, L.~Bottou, Y.~Bengio, and P.~Haffner.
\newblock Gradient-based learning applied to document recognition.
\newblock \emph{Proceedings of the IEEE}, 86\penalty0 (11):\penalty0 2278--2324, 1998.

\bibitem[Lin et~al.(2015)Lin, RoyChowdhury, and Maji]{lin2015bilinear}
T.-Y. Lin, A.~RoyChowdhury, and S.~Maji.
\newblock Bilinear {CNN} models for fine-grained visual recognition.
\newblock In \emph{Proceedings of the IEEE International Conference on Computer Vision}, pages 1449--1457, 2015.

\bibitem[Liu et~al.(2020)Liu, Huang, Chen, and Suykens]{Liu2020a}
F.~Liu, X.~Huang, Y.~Chen, and J.~A.~K. Suykens.
\newblock Random features for kernel approximation: A survey in algorithms, theory, and beyond.
\newblock \emph{CoRR}, abs/2004.11154, 2020.

\bibitem[Milios et~al.(2018)Milios, Camoriano, Michiardi, Rosasco, and Filippone]{Milios2018}
D.~Milios, R.~Camoriano, P.~Michiardi, L.~Rosasco, and M.~Filippone.
\newblock Dirichlet-based {G}aussian processes for large-scale calibrated classification.
\newblock In \emph{Advances in Neural Information Processing Systems 31}, pages 6008--6018. Curran Associates, Inc., 2018.

\bibitem[Neeser and Massey(1993)]{neeser1993proper}
F.~D. Neeser and J.~L. Massey.
\newblock Proper complex random processes with applications to information theory.
\newblock \emph{IEEE transactions on information theory}, 39\penalty0 (4):\penalty0 1293--1302, 1993.

\bibitem[Paszke et~al.(2019)]{paszke2019pytorch}
A.~Paszke et~al.
\newblock {PyTorch: An Imperative Style, High-Performance Deep Learning Library}.
\newblock In \emph{Advances in Neural Information Processing Systems 32}, pages 8026--8037. Curran Associates, Inc., 2019.

\bibitem[Pennington et~al.(2015)Pennington, Yu, and Kumar]{Pennington2015}
J.~Pennington, F.~X.~X. Yu, and S.~Kumar.
\newblock Spherical random features for polynomial kernels.
\newblock In \emph{Advances in Neural Information Processing Systems 28}, pages 1846--1854. Curran Associates, Inc., 2015.

\bibitem[Pham and Pagh(2013)]{Pham2013}
N.~Pham and R.~Pagh.
\newblock Fast and scalable polynomial kernels via explicit feature maps.
\newblock In \emph{Proceedings of the 19th ACM SIGKDD International Conference on Knowledge Discovery and Data Mining}, pages 239--247. Association for Computing Machinery, 2013.

\bibitem[Rahimi and Recht(2007)]{Rahimi2007}
A.~Rahimi and B.~Recht.
\newblock Random features for large-scale kernel machines.
\newblock In \emph{Advances in Neural Information Processing Systems 20}, pages 1177--1184. Curran Associates Inc., 2007.

\bibitem[Rasmussen and Williams(2006)]{Rasmussen2006}
C.~Rasmussen and C.~Williams.
\newblock \emph{{Gaussian Processes for Machine Learning}}.
\newblock MIT Press, 2006.

\bibitem[Rendle(2010)]{Rendle2010}
S.~Rendle.
\newblock Factorization machines.
\newblock In \emph{Proceedings of the 2010 IEEE International Conference on Data Mining}, pages 995--1000, 2010.

\bibitem[Scholkopf and Smola(2002)]{Schoelkopf2001}
B.~Scholkopf and A.~J. Smola.
\newblock \emph{Learning with Kernels: Support Vector Machines, Regularization, Optimization, and Beyond}.
\newblock MIT Press, 2002.

\bibitem[Smola et~al.(2000)Smola, \'{O}v\'{a}ri, and Williamson]{Smola2001}
A.~Smola, Z.~\'{O}v\'{a}ri, and R.~C. Williamson.
\newblock Regularization with dot-product kernels.
\newblock In \emph{Advances in Neural Information Processing Systems 13}, pages 308--314. Curran Associates, Inc., 2000.

\bibitem[Song et~al.(2021)Song, Woodruff, Yu, and Zhang]{Song21c}
Z.~Song, D.~Woodruff, Z.~Yu, and L.~Zhang.
\newblock Fast sketching of polynomial kernels of polynomial degree.
\newblock In \emph{Proceedings of the 38th International Conference on Machine Learning}, volume 139 of \emph{Proceedings of Machine Learning Research}, pages 9812--9823. PMLR, 2021.

\bibitem[Sutherland and Schneider(2015)]{Sutherland2015}
D.~J. Sutherland and J.~Schneider.
\newblock On the error of random fourier features.
\newblock In \emph{Proceedings of the Thirty-First Conference on Uncertainty in Artificial Intelligence}, pages 862--871. AUAI Press, 2015.

\bibitem[Titsias(2009)]{titsias2009variational}
M.~Titsias.
\newblock Variational learning of inducing variables in sparse {G}aussian processes.
\newblock In \emph{Proceedings of the Twelfth International Conference on Artificial Intelligence and Statistics}, volume~5 of \emph{{JMLR} Proceedings}, pages 567--574. JMLR, 2009.

\bibitem[Trefethen and Bau(1997)]{trefethen97}
L.~N. Trefethen and D.~Bau.
\newblock \emph{Numerical Linear Algebra}.
\newblock SIAM, 1997.

\bibitem[Tropp(2011)]{Tropp2011}
J.~A. Tropp.
\newblock {Improved analysis of the subsampled randomized Hadamard transform}.
\newblock \emph{Advances in Adaptive Data Analysis}, 3\penalty0 (1-2):\penalty0 115--126, 2011.

\bibitem[Uzilov et~al.(2006)Uzilov, Keegan, and Mathews]{DBLP:journals/bmcbi/UzilovKM06}
A.~V. Uzilov, J.~M. Keegan, and D.~H. Mathews.
\newblock Detection of non-coding {RNAs} on the basis of predicted secondary structure formation free energy change.
\newblock \emph{{BMC} Bioinformatics}, 7:\penalty0 173, 2006.

\bibitem[Vaswani et~al.(2017)Vaswani, Shazeer, Parmar, Uszkoreit, Jones, Gomez, Kaiser, and Polosukhin]{Vaswani2017}
A.~Vaswani, N.~Shazeer, N.~Parmar, J.~Uszkoreit, L.~Jones, A.~N. Gomez, L.~Kaiser, and I.~Polosukhin.
\newblock Attention is all you need.
\newblock In \emph{Advances in Neural Information Processing Systems 30}, pages 5998--6008. Curran Associates, Inc., 2017.

\bibitem[Vershynin(2018)]{vershynin2018high}
R.~Vershynin.
\newblock \emph{High-Dimensional Probability: An Introduction with Applications in Data Science}.
\newblock Cambridge University Press, 2018.

\bibitem[Wacker and Filippone(2021)]{WackerKES21}
J.~Wacker and M.~Filippone.
\newblock Local random feature approximations of the gaussian kernel.
\newblock In \emph{Proceedings of the 26th International Conference on Knowledge-Based and Intelligent Information {\&} Engineering Systems}, Procedia Computer Science. Elsevier, 2021.

\bibitem[Wahba(1990)]{wahba1990spline}
G.~Wahba.
\newblock \emph{Spline Models for Observational Data}.
\newblock SIAM, 1990.

\bibitem[Weissbrod et~al.(2016)Weissbrod, Geiger, and Rosset]{weissbrod2016multikernel}
O.~Weissbrod, D.~Geiger, and S.~Rosset.
\newblock Multikernel linear mixed models for complex phenotype prediction.
\newblock \emph{Genome Research}, 26\penalty0 (7):\penalty0 969--979, 2016.

\bibitem[Williams and Seeger(2000)]{williams2001using}
C.~K. Williams and M.~Seeger.
\newblock Using the {N}ystr\"om method to speed up kernel machines.
\newblock In \emph{Advances in Neural Information Processing Systems 13}, pages 682--688. Curran Associates, Inc., 2000.

\bibitem[Woodruff(2014)]{Woodruff2014}
D.~P. Woodruff.
\newblock {Sketching as a tool for numerical linear algebra}.
\newblock \emph{Foundations and Trends in Theoretical Computer Science}, 10\penalty0 (1-2):\penalty0 1--157, 2014.

\bibitem[Xiao et~al.(2017)Xiao, Rasul, and Vollgraf]{xiao2017fashionmnist}
H.~Xiao, K.~Rasul, and R.~Vollgraf.
\newblock Fashion-{MNIST}: a novel image dataset for benchmarking machine learning algorithms.
\newblock \emph{CoRR}, abs/1708.07747, 2017.

\bibitem[Yamada and Matsumoto(2003)]{yamada2003statistical}
H.~Yamada and Y.~Matsumoto.
\newblock Statistical dependency analysis with support vector machines.
\newblock In \emph{Proceedings of the Eighth International Conference on Parsing Technologies}, pages 195--206, 2003.

\bibitem[Yang et~al.(2015)Yang, Moczulski, Denil, Freitas, Smola, Song, and Wang]{Yang2015}
Z.~Yang, M.~Moczulski, M.~Denil, N.~D. Freitas, A.~Smola, L.~Song, and Z.~Wang.
\newblock {Deep fried convnets}.
\newblock \emph{Proceedings of the 2015 IEEE International Conference on Computer Vision}, pages 1476--1483, 2015.

\bibitem[Yu et~al.(2016)Yu, Suresh, Choromanski, Holtmann-Rice, and Kumar]{Yu2016}
F.~X. Yu, A.~T. Suresh, K.~Choromanski, D.~Holtmann-Rice, and S.~Kumar.
\newblock Orthogonal random features.
\newblock In \emph{Advances in Neural Information Processing Systems 30}, page 1983–1991. Curran Associates Inc., 2016.

\end{thebibliography}

\end{document}